\documentclass{article}
\usepackage[T1]{fontenc} 
\DeclareUnicodeCharacter{308}{\:u}  
\usepackage{amsfonts, amsmath, amssymb}
\let\originallhook\lhook
\usepackage{MnSymbol}
\let\lhook\originallhook
\usepackage{authblk}
\usepackage[nomessages]{fp} 
\usepackage{graphicx}
\usepackage{dcolumn}
\usepackage{bm}
\usepackage[ruled, lined]{algorithm2e}
\usepackage{hyperref}
\usepackage{tikz-cd}        
\usepackage{amsthm}%
\usepackage{enumitem}%
\usepackage{mathtools}
\usepackage{changepage}
\usepackage{listings}
\newcommand{\BC}{{\mathbb C}}%
\newcommand{\BK}{{\mathbb K}}%
\newcommand{\BN}{{\mathbb N}}%
\newcommand{\BP}{{\mathbb P}}%
\newcommand{\BQ}{{\mathbb Q}}%
\newcommand{\BR}{{\mathbb R}}%
\newcommand{\BS}{{\mathbb S}}%
\newcommand{\calC}{{\mathcal C}}%
\newcommand{\calD}{{\mathcal D}}%
\newcommand{\calE}{{\mathcal E}}%
\newcommand{\calF}{{\mathcal F}}%
\newcommand{\calG}{{\mathcal G}}%
\newcommand{\calH}{{\mathcal H}}%
\newcommand{\calI}{{\mathcal I}}%
\newcommand{\calP}{{\mathcal P}}%
\newcommand{\calR}{{\mathcal R}}%
\newcommand{\calX}{{\mathcal X}}%
\newcommand{\calY}{{\mathcal Y}}%
\newcommand{\eps}{\epsilon}%
\newcommand{\arrow}{\rightarrow}%
\DeclareMathOperator{\im}{Im}
\DeclareMathOperator{\Tr}{Tr}%
\newcommand{\smallspace}{\vskip 2mm\noindent}

\newcommand{\colori}{\ensuremath{\gamma}}  
\newcommand{\SymGr}[1]{\Sigma_{#1}}         
\newcommand{\irred}[1]{\calH^{(#1)}}       

\newcommand{\vv}[1]{\boldsymbol{#1}}

\newtheorem{definition}{Definition}
\newtheorem{proposition}{Proposition}
\newtheorem{theorem}{Theorem}

\newtheorem{lemma}[proposition]{Lemma}

\newtheorem{corollary}[proposition]{Corollary}

\newcommand{\refSecVariants}{B}
\newcommand{\refSecCompleteTheorems}{D}
\newcommand{\refSecAlgorithm}{E}
\newcommand{\refSubsecEquivariantMaps}{G2}
\newcommand{\refSubsecDualTensorHom}{G5}
\newcommand{\refSubsecPPSDsNormalForm}{H1}
\newcommand{\refSubsecPPSDsOnMultisets}{H2}
\newcommand{\refSecCovariantComplete}{I4}
\newcommand{\refSecMatrixMoments}{M}
\newcommand{\refSecProofTheoremMatMult}{N}
\newcommand{\refSecExperimentDetails}{P}
\newcommand{\refTheoremTopologicalCompleteness}{1}
\newcommand{\refTheoremFiniteFeatures}{2}
\newcommand{\refTheoremAlgCompleteness}{3}
\newcommand{\refTheoremInvariantTheory}{4}

\title{Complete and Efficient Covariants for 3D Point Configurations with Application to Learning  Molecular Quantum Properties}
\author[1]{\small Hartmut Maennel}
\author[2]{\small Oliver T. Unke}
\author[2,3,4,5,6]{\small Klaus-Robert M\"uller}
\affil[1]{\footnotesize Google DeepMind Z\"urich, Switzerland}
\affil[2]{\footnotesize Google DeepMind Berlin, Germany}
\affil[3]{\footnotesize TU Berlin, Machine Learning Group, Berlin, Germany}
\affil[4]{\footnotesize Berlin Institute for the Foundation of Learning and Data, Germany}
\affil[5]{\footnotesize Max Planck Institute for Informatics Saarbr\"ucken, Germany}
\affil[6]{\footnotesize Korea University, Department of Artificial Intelligence, Seoul, Korea}
\date{}

\begin{document}

\maketitle
\begin{abstract}
When modeling physical properties of molecules with machine learning, it is desirable to 
incorporate $SO(3)$-covariance. While such models based on low body order features are not complete, we formulate and prove general completeness properties for higher order methods, and show that $6k-5$ of these features are enough for up to $k$ atoms. We also find that the Clebsch--Gordan operations commonly used in these methods can be replaced by matrix multiplications without sacrificing completeness, lowering the scaling from $O(l^6)$ to $O(l^3)$ in the degree of the features. We apply this to quantum chemistry, but the proposed methods are generally applicable for problems involving 3D point configurations.
\end{abstract}
\section*{Introduction}
Atomistic simulations have proven indispensable for advancing chemistry and materials science, providing insights into the behavior of matter at the atomic level.  In the past, these simulations have been computationally demanding, but the advent of Density Functional Theory (DFT) \cite{kohn1965self} significantly enhanced the accessibility of atomistic simulations, and recent breakthroughs in machine learning (ML) have further accelerated progress \cite{rupp2012fast,snyder2012finding,brockherde2017bypassing,bogojeski2020quantum,hermann2020deep,pfau2020ab}. ML methods trained on \textit{ab initio} data now enable the fast and accurate prediction of quantum properties orders of magnitude faster than traditional calculations \cite{noe2020machine,von2020exploring,unke2021machine,keith2021combining,glielmo2021unsupervised,deringer2021gaussian}.
A cornerstone of these methods, whether utilizing kernel-based approaches \cite{unke2017toolkit,chmiela2017machine,chmiela2019sgdml,glielmo2020gaussian} or deep learning \cite{montavon2013machine,schutt2017quantum,schutt2017schnet,schutt2018schnet,unke2019physnet}, lies in the effective representation of molecules \cite{Behler2011,hansen2015machine,faber2018alchemical,christensen2020fchl} or materials \cite{schutt2014represent,butler2018machine,sauceda2022bigdml} through carefully chosen features or descriptors. Early examples include the Coulomb Matrix representation \cite{rupp2012fast} and SOAP \cite{Bartok2013OnRC}, while recent advancements extend this principle beyond rotationally invariant representations with the design of equivariant model architectures \cite{thomas2018tensor,anderson2019cormorant,fuchs2020se,satorras2021n,unke2021se,schutt2021equivariant,unke2021spookynet,batzner20223,frank2022so3krates,musaelian2023learning}.

However, \cite{pozdnyakov2020incompleteness} pointed out that commonly available descriptors are not able to uniquely identify some molecular structures~\cite{pozdnyakov2020incompleteness,Pozdnyakov2022,nigam2023completeness}. This can lead to ambiguities (two distinct structures may be mapped to the same descriptor) that hamper the performance of ML models. Effectively, a lack of uniqueness is similar to introducing a high level of noise into the learning process and may hinder generalization. 
A second important shortcoming of some modern ML architectures was discussed by \cite{fu2022forces} and only becomes visible when running molecular dynamics (MD) simulations \cite{fu2022forces}. It was observed that ML models with excellent prediction accuracy for energies and forces can nevertheless show unphysical instabilities (e.g.~spurious bond dissociation) when simulating longer MD trajectories --- limiting their usefulness in practice.  Equivariant architectures, however, as broad anecdotal evidence and some theoretical analyses have shown \cite{fu2022forces,frank2023from},  were found to enable stable MD simulations over long timescales  \cite{fu2022forces,chmiela2019sgdml,schutt2017schnet,schutt2018schnet,schutt2021equivariant,unke2021spookynet,frank2022so3krates,frank2023from}. \\
\begin{figure*}
    \centering
    \includegraphics[width=\textwidth]{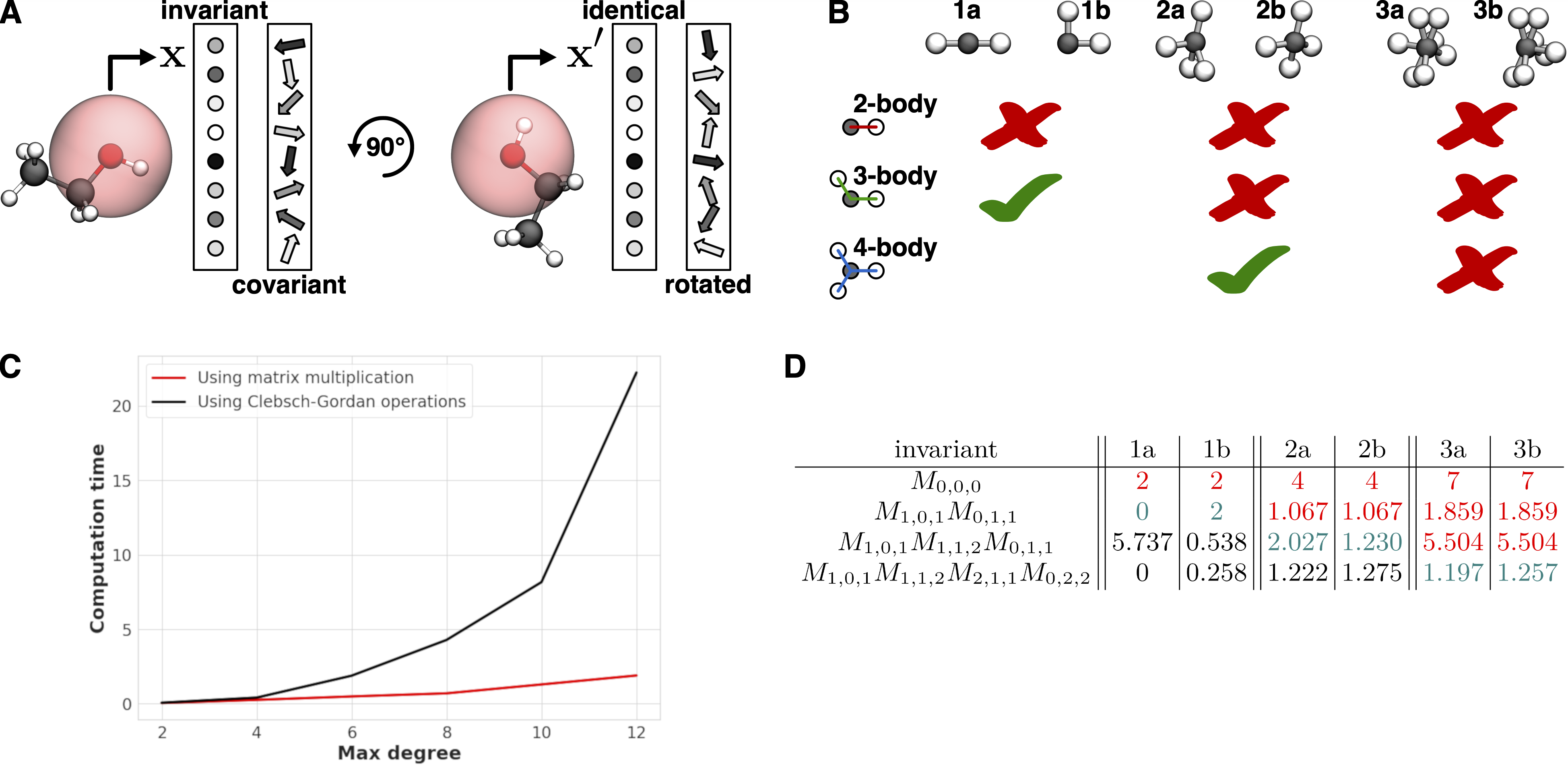}
    \caption{\textbf{A}: Given an atom (red) in a local chemical environment (translucent red sphere), the aim is to find a descriptor, i.e.~a fixed-size feature vector $\mathbf{x}$, such that all different environments are also mapped to different descriptors (uniqueness). The features can be either invariant or covariant w.r.t.\ rotations, meaning that when the environment is rotated, the new features $\mathbf{x'}$ are either identical (invariant) or ``rotate in the same way'' (covariant). 
    \textbf{B}: Examples for local chemical environments (taken from Ref.~\cite{Pozdnyakov2020}), which cannot be distinguished by features constructed from $m$-body information. From the perspective of the central black atom, the environments 1a and 1b appear identical when considering only 2-body information (e.g.\ distances), but are readily distinguished by 3-body information (e.g.\ angles). For environments 2a and 2b, 4-body information (e.g.\ dihedral angles) is necessary to distinguish them, whereas environments 3a and 3b require even higher-order information.
    \textbf{C}: Computational cost for evaluating learned invariants.
    While the straightforward GPU/TPU-friendly implementation of invariants using Clebsch--Gordan operations scales with $O(l^6)$, the proposed implementation replacing them by matrix multiplications scales with $O(l^3)$, 
    enabling the use of higher degrees.
    \textbf{D}: Examples of invariants that distinguish the pairs of B above. According to Theorem~\ref{theorem:MatMultComplete} there must be a distinguishing invariant which is given by matrix multiplication as in \eqref{eq:matrix_invariants}, here we used
    $l=2$. Products of $k$ matrices give $(k+1)$--body invariants, so for structure pairs 1, 2, and 3 (see panel B) we need 2, 3, and 4 matrix factors, respectively. While in general these invariants are $SO(3)$--invariants, the sum of the matrix indices chosen here is even, so the invariants are actually $O(3)$--invariants and these numbers show that the pairs are also in distinct $O(3)$--equivalent classes.
    }
    \label{fig:overview}
\end{figure*}
Both aspects lead to the interesting theoretical question of {\em how to construct a provably unique invariant, or more generally, a ``complete'' (to be defined below) equivariant and computationally efficient representation of descriptors for atomistic simulations}.
We will study this challenge  both by theoretical means and by performing empirical atomistic simulations.
\smallspace
Let us assume that  the origin of our coordinate system was fixed meaningfully and we are looking for unique 
descriptors of point sets that are equivariant under rotations in $SO(3)$.
\\
To get invariant features, we can use a rotationally invariant function of $n$ points (e.g.~distance from the origin for $n=1$, or angles between two points for $n=2$), and then sum over all $n$--tuples of points in the configuration. 
Such descriptors are called ``$(n+1)$--body functions''. It was recently shown that
descriptors based on 2- and 3-body information (distances and angles) are unable to distinguish some non-equivalent environments \cite{Pozdnyakov2020}. Even 4-body information (dihedrals) is not sufficient in all cases
(see Fig.~\ref{fig:overview}B) and it is necessary to include 
higher $m$-body information for some structures. Other methods that construct descriptors implicitly, e.g.\ by message-passing \cite{gilmer2017neural}, suffer from similar problems \cite{Pozdnyakov2022}.
\section*{Results and discussion}
Let us start defining an appropriate mathematical language. In applications to
chemistry, the points in the point set can belong to different atom types/elements which have to be treated
differently. We assume there is a fixed finite set $\calC$ of ``colors'' (the atom types/elements), and each point in the
point set $S$ is assigned a color in $\calC$, i.e.\ $S = \bigcup_{\gamma\in\calC} S_\gamma$.\\
We propose to take as potential features all \emph{polynomial point set descriptors} (PPSDs), i.e.~all scalar expressions that can be written down for colored point sets, using the coordinates of points, constants from $\BR$, addition, multiplication, and summations over all points of a given color, such that these expressions can be evaluated for any point set independent of the number of points (See Appendix \refSubsecPPSDsNormalForm\ for  formal definitions).
\smallspace
In practice, a variant of PPSDs is more useful, using polynomials only for the angular part (i.e.~as a function on the sphere $\BS^2$) and some other function space for the radial part. With the assumptions that
these radial functions are analytic and allow approximation of continuous functions in the radius, we can 
(with some extra effort) prove almost the same theorems, see Appendix \refSecVariants\ for the definitions, and later sections for details and proofs.
\smallspace
{\bf General theorems:}
We now describe informally a series of mathematical theorems about PPSDs that we prove in this work,
see respective Appendices 
for the precise formulations and proofs.
\smallspace
We first observe (see Appendix \refSubsecPPSDsNormalForm) that the computation of any scalar PPSD
can be arranged into two steps:
\begin{enumerate}
  \item Evaluate expressions involving only one summation sign: $\sum_{\vv r\in S_\gamma} P(\vv r)$ for some color $\gamma$ and polynomial $P:\BR^3\arrow \BR$ acting on point coordinates $\vv r$.
  We call them \emph{fundamental features}.
  \item Evaluate polynomials in fundamental features.
\end{enumerate}
This separability into two steps allows any PPSD to be {\bf evaluated in time $O(n)$} where $n$ is the number of points (here atoms), which is a major advantage over e.g.~descriptors based on rational functions, for which this is generally not possible.\\
We call a PPSD that can be written such that all polynomials in fundamental features have degree $d$ ``homogeneous 
of order $d$'' \footnote{Note that this is only the degree of the polynomial in fundamental features (step 2), it does not take into account the degrees of the polynomials used to construct the fundamental
features themselves. When we multiply out and move all summations
to the left (see Appendix \refSubsecPPSDsNormalForm) this order corresponds to the depth of the summations, since each fundamental feature comes with one summation sign}. 
The order of such a PPSD is unique, for a proof and a refinement of this notion see Appendix
\refSubsecPPSDsOnMultisets. PPSDs of order $d$ are also said to be of ``{\bf body order} $d+1$'' (this 
convention includes one atom at the origin of the coordinate system in the count).\\
In this language, there are infinitely many independent $SO(3)$--invariant PPSDs of body order 3, but the examples in \cite{Pozdnyakov2020} show that there are inequivalent configurations that cannot be distinguished by invariant functions of body orders $\leq 4$ (see Fig.~\ref{fig:overview}B).
Our {\bf Topological Completeness Theorem} (Theorem \refTheoremTopologicalCompleteness\ in Appendix \refSecCompleteTheorems) says that this problem vanishes
when we allow arbitrary body orders, even when we restrict the functions to be {\em polynomial} invariants: Any two $SO(3)$--inequivalent configurations can be distinguished by $SO(3)$--invariant PPSDs, i.e.~taking the values of \emph{all polynomial $SO(3)$--invariant} functions gives a \emph{unique} descriptor.
In general for \emph{covariant} functions the values of PPSDs change when we rotate a configuration, so this uniqueness property has to be expressed differently: We prove that there are enough $SO(3)$--covariant PPSDs to approximate any continuous $SO(3)$--covariant function of colored point sets.
\smallspace
Without bound on the number of points in the configurations it is of course necessary to use infinitely many
independent invariant functions to distinguish all $SO(3)$--inequivalent configurations, as these form an infinite dimensional space. However, we can ask how many features are necessary to uniquely identify configurations of up to $k$ points. Our {\bf Finiteness Theorem} (Theorem \refTheoremFiniteFeatures\ in Appendix \refSecCompleteTheorems) gives a linear upper bound of $6k-5$, with some guarantees for the distance of non--equivalent configurations. Its proof is based on real algebraic geometry and subanalytic geometry.
\smallspace
{\bf Practical construction:}
We will  now  show how to produce unique features in such a way that we
never leave the space of covariant features: \\
Let $\irred{l}$ be the irreducible $(2l+1)$--dimensional (real) representation
of $SO(3)$, and $Y_l:\BR^3\arrow\irred{l}$ be a $SO(3)$--covariant polynomial (which is unique on the sphere $\BS^2$ up to a scalar constant factor, see Appendix \refSubsecEquivariantMaps).
These $Y_l$ are given by (real valued) spherical harmonics of degree $l$.
We now proceed again in two stages:
\begin{enumerate}
  \item Evaluate spherical harmonics \footnote{To be precise, we multiply the spherical harmonics with radial basis functions. We use exponent $2k$ here to have only
    polynomial functions. In practice, we would rather use different, decaying functions,
    this is treated as ``case ii'' in general as one of the variations, see Appendix 
    \refSecVariants.}
    :\\
    $\sum_{\vv r\in S_\gamma} |\vv r|^{2k} Y_l(\vv r)$
    for all colors $\gamma$ and $k=0,1,2,...$ and $l=0,1,2,...$ \\
    These are covariant \emph{fundamental features} (i.e.~of order~$1$) with values in $\irred{l}$.
  \item Iterate for $d=1,2,...$: Compute Clebsch--Gordan operations \footnote{These project
    the tensor product of two representations to an irreducible component, see e.g. \cite{Unke2024e3x}.}
    $\calH^{l_1} \otimes \calH^{l_2}\arrow \calH^{l_3}$ for $|l_1-l_2| \leq l_3 \leq l_1+l_2$, where the feature in $\calH^{l_1}$ is a fundamental feature, and the feature in $\calH^{l_2}$ is of order $d$. This gives covariant features of order $d+1$.
\end{enumerate}
Clearly, this construction appears to be somewhat special, so we may ask whether it actually gives ``enough'' invariants (i.e.~achieves completeness).
This is in fact true in a very strong sense: Our {\bf Algebraic Completeness Theorem} (Theorem \refTheoremAlgCompleteness) says that {\em all}  invariant / covariant
PPSDs can be obtained as a linear combination of them; in a sense this is just the isotypical decomposition of 
the space of all PPSDs (see Appendix \refSecCovariantComplete).
\smallspace
While the above strategy to construct invariant / covariant functions
has been used e.g. in \cite{Thomas2018TensorFN,anderson2019cormorant,
willatt2019atom,Nigam2020RecursiveEA, Nigam2022,Batatia2022,musaelian2023learning}, our novel completeness theorems 
show that this avoids the potential incompleteness problem pointed out in
\cite{pozdnyakov2020incompleteness}.  In fact, by our algebraic completeness theorem we get {\em all} polynomial covariant functions,
and by the topological completeness theorem those are {\em sufficient} to approximate any continuous covariant function. We also get an algebraic completeness theorem for features constructed from
tensor products and contractions as in \cite{shapeev2016moment}, see 
{\bf Theorem~\refTheoremInvariantTheory}; this is based on classical invariant theory.
\smallspace
{\bf Computational bottleneck:}
We now turn to a particular {\em efficient} variant of our construction. 
Since invariant PPSDs of order $<4$ are 
not sufficient for distinguishing all $SO(3)$--equivalence classes, we need to construct covariants of higher body orders, i.e.~in the above procedure we need to use the Clebsch--Gordan 
products. Note that their computational cost
is independent of the number of points, and is linear in the number of products, but 
scales as $O(l^6)$ when we take the tensor product of two representations of the form
$\irred{0}\oplus...\oplus\irred{l}$. But with unrestricted number of points in our configuration, we cannot bound the $l$, even if we are just considering 
configurations on $\BS^2\subset\BR^3$: Using $Y_l:\BS^2\arrow \irred{l}$ only for $l=0,1,...,L-1$ yields a $|\calC|\cdot L^2$--dimensional vector space of fundamental features $\sum_{\vv r\in S_\gamma} Y_l(\vv r)$
(and all PPSDs are polynomials in the fundamental features). So this could only describe a configuration space of 
a dimension $\leq |\calC|\cdot L^2$, not the $\infty$--dimensional space of configurations on $\BS^2$
with an unbounded number of points.
\smallspace
Consequently, the bottleneck for a larger number of points (necessitating using larger $l$ for constructing the fundamental features) can be determined as the $O(l^6)$ Clebsch--Gordan operation.
We will now propose how to construct a subset of local descriptors that alternatively to 
 Clebsch--Gordan operations relies only on 
matrix-matrix multiplication. This procedure scales as only $O(l^3)$ for 
bilinear operations on two representations of the form
$\irred{0}\oplus...\oplus\irred{l}$. 
A similar speedup was published recently in \cite{Luo2024}, replacing the Clebsch--Gordan operation 
by the multiplication of functions. However, since multiplying functions (instead of matrices) is commutative, this does not reproduce the anti--commutative part of the Clebsch--Gordan operations. Therefore 
the construction  in \cite{Luo2024} does not have the full expressivity desired and would not satisfy our Algebraic Completeness Theorem or Theorem \ref{theorem:MatMultComplete} below. In particular, since commutative products cannot produce
pseudo--tensors, its invariants could not distinguish configurations from their mirror images.
\smallspace
{\bf Matrix Construction:} 
Our key idea for removing the computational bottleneck is to apply the Clebsch--Gordan relation
\begin{equation}
   \irred{a} \otimes \irred{b}
   \simeq
   \irred{|a-b|} \oplus \irred{|a-b|+1} \oplus... \oplus \irred{a+b}
   \label{eq:Clebsch_Gordan}
\end{equation}
``backwards'' to efficiently encode a collection of features in
$\calH^{(|a-b|)} \oplus ... \oplus \calH^{(a+b)}$ as a $(2a+1)\times (2b+1)$ matrix in
\[
  Lin(\irred{a},\irred{b})
  \simeq {\irred{a}}^* \otimes \irred{b}
  \simeq \irred{a} \otimes \irred{b}
\]
(see Appendix \refSubsecDualTensorHom) and then the matrix multiplication is a covariant map of representations
\[
    Lin(\irred{a},\irred{b}) \times Lin(\irred{b},\irred{c}) \arrow Lin(\irred{a},\irred{c}).
\]
With Schur's Lemma one can show that it can be expressed as a linear combination of Clebsch--Gordan operations, so unless some coefficients are zero, we can expect this operation to be as useful as the Clebsch--Gordan operations for constructing covariant features 
of higher body order. This is indeed the case and to formulate the
corresponding theorem, we define the involved features:\\
Let $\iota_{a,b,l}: \irred{l}\arrow Mat_{2b+1,2a+1}$ be the embedding given by \eqref{eq:Clebsch_Gordan}
and define the ``matrix moments''
\begin{equation}
   M_{a,b,l}(\colori) := \iota_{a,b,l} \sum_{\vv r \in S_\colori} Y_l(\vv r)
   \label{eq:MatrixMoments}
\end{equation}
which are $(2b+1)\times(2a+1)$ matrices 
(see Appendix \refSecMatrixMoments\ for some examples for explicit formulas).
Then the result of the multiplication
\begin{equation}
   M_{a_{m-1}, a_m, l_m}(\gamma_m)\cdot ... 
   \cdot M_{a_1, a_2, l_2}(\gamma_2) \cdot M_{0,a_1, l_1}(\gamma_1)
   \label{eq:matrix_invariants}
\end{equation}
\hskip 23mm
\includegraphics[width=0.66\textwidth]{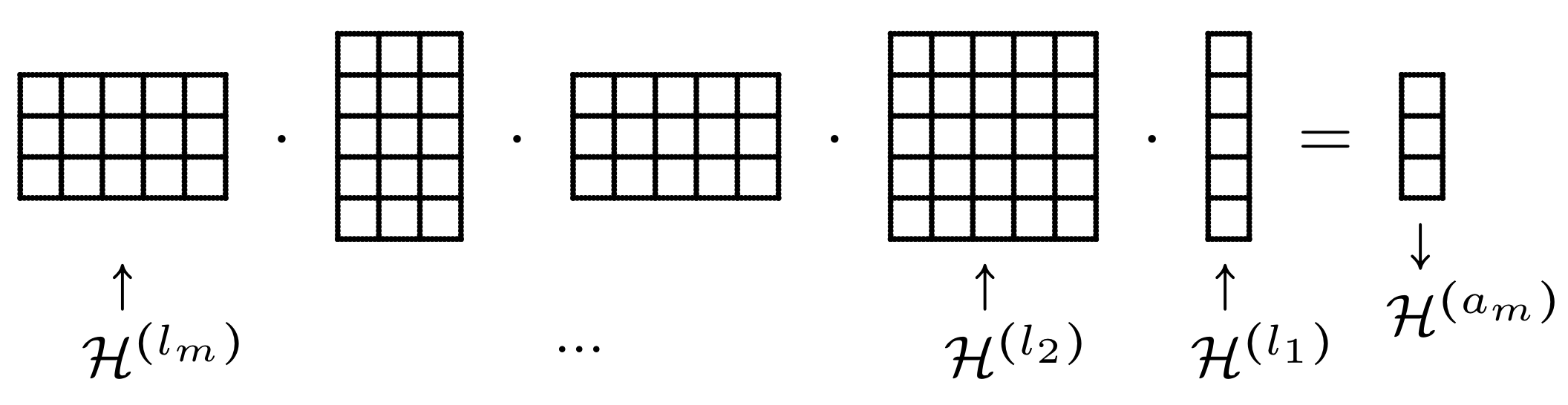}
\\
\noindent
with $ l_1=a_1$ and 
$
 |a_1-a_2| \leq l_2 \leq a_1+a_2,...,|a_{m-1}-a_m| \leq l_m \leq a_{m-1}+a_m
$
are covariant $a_m\times 1$ matrices, i.e. vectors in $\irred{a_m}$, given by polynomials of degree $l_1+...+l_m$, and computing them takes $O(m\cdot a^3)$ steps for an upper bound $a \geq a_i$.
\setcounter{theorem}{4} 
\begin{theorem}[Algebraic Completeness for features from matrix multiplication]\label{theorem:MatMultComplete}\phantom{.}\\
  Any $SO(3)$--covariant feature with values in a $\irred{l}$ can be written as a linear combination
  of the $SO(3)$--covariants \eqref{eq:matrix_invariants} with $a_m=l$.\\
  For $O(3)$--covariants it is enough to use those features given by \eqref{eq:matrix_invariants} 
  with the appropriate parity of $l_1+...+l_m$.
\end{theorem}
\noindent
For the proof see Appendix \refSecProofTheoremMatMult.
\smallspace
{\bf Learning a linear combination:}
While computing one given invariant of the form \eqref{eq:matrix_invariants} would not be more efficient than 
with Clebsch--Gordan operations (as it would waste whole matrices for encoding only one feature), for applications in
Machine Learning we always compute with linear combinations of features (with learnable coefficients), and both the
Clebsch--Gordan operation and Matrix Multiplication define maps
\[
   \left(\irred{0}\oplus...\oplus\irred{l}\right) \otimes
   \left(\irred{0}\oplus...\oplus\irred{l}\right) \ \arrow \ 
   \irred{0}\oplus...\oplus\irred{l}
\]
which are used to build up different linear combinations of covariants of higher body
order. In the Clebsch--Gordan case we also can add to the learnable coefficients of the input features
further learnable parameters that give different weights to the individual parts 
$\irred{l_1}\otimes\irred{l_2}\arrow\irred{l}$
that contribute to the same $\irred{l}$ in the output, whereas in the Matrix Multiplication case these
mixture coefficients are fixed (but depend on the shape of the matrices involved).
However, Theorem \ref{theorem:MatMultComplete} shows that using different shapes of matrices is already sufficient to generate all possible covariants, so both methods can
in principle learn the same functions.
\smallspace
{\bf Matrix of Matrices construction for efficiency:} For practical applications it is important how to organize the matrix multiplications
efficiently. In particular when using GPUs / TPUs with hardware support
for matrix multiplication, it is much more favorable to compute with a few large matrices
than with many small matrices.
Therefore we will use linear combinations of $\iota_{a,b,l}$ for $l=|a-b|,...,a+b$
to fill a $(2b+1) \times (2a+1)$ matrix, and pack $r \times r$ small matrices for $a,b$ in 
$\{l_1, l_2,...,l_r\}$ into a large square matrix of side length $(2l_1 +1) + ... + (2l_r+1)$.\\
Then $k-1$ such matrices are multiplied to get a matrix built out of covariants
of body order $k$.
\smallspace
\includegraphics[width=0.7\textwidth]{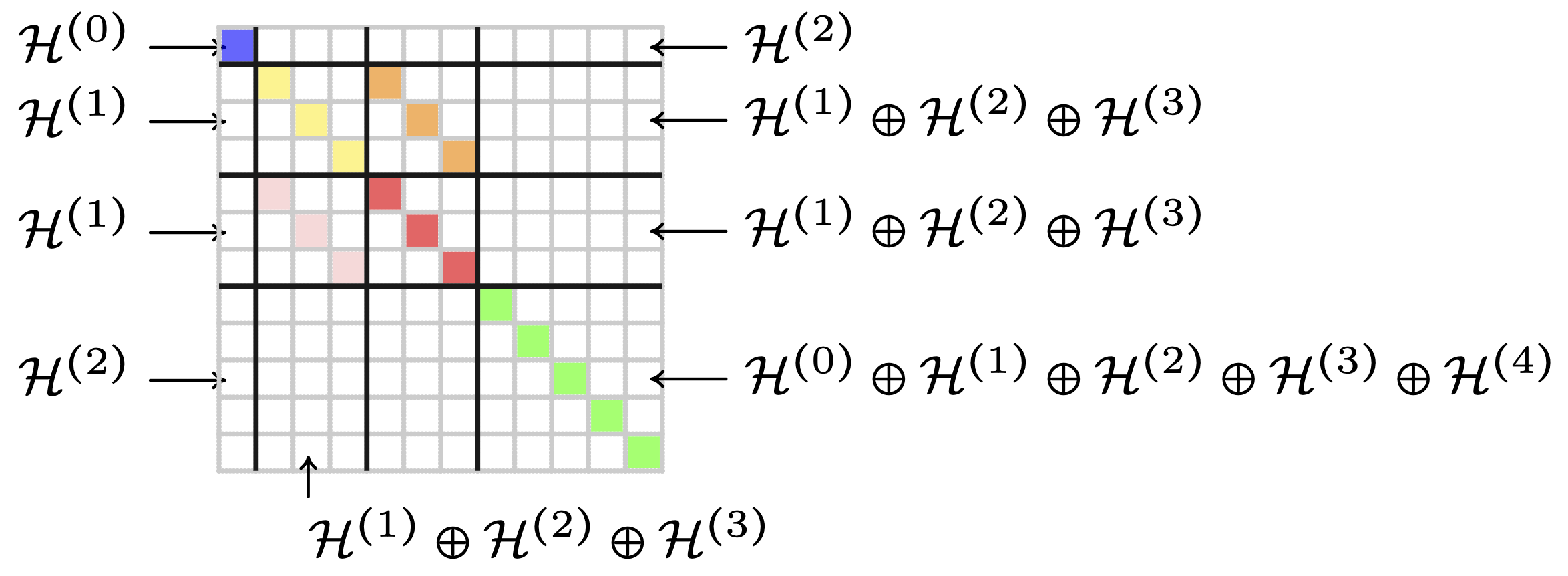}
\smallspace
This matrix can then be applied to $n_1$ column vectors
from $\irred{l_1}\oplus \irred{l_2} \oplus ... \oplus \irred{l_r}$ 
to get covariant vectors of body order $k+1$.
\smallspace
\hskip 8mm
\includegraphics[width=0.6\textwidth]{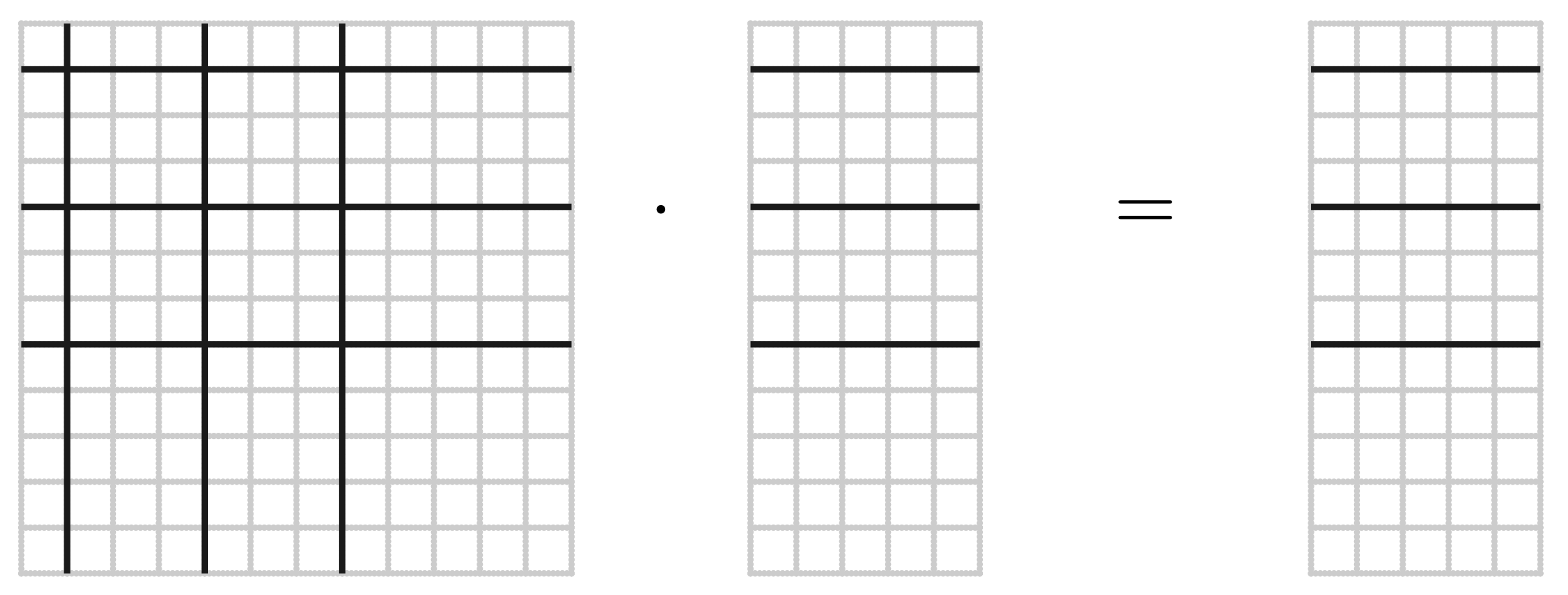}
\smallspace
If the end result should be scalars, we can take scalar
products of the $n_1$ column vectors in $\irred{l}$ with $n_2$ new
covariants in $\irred{l}$ to obtain
$r\cdot n_1\cdot n_2$ invariants of body order $k+2$;
also the traces of the 
square submatrices of the matrix product give invariants of body order $k$ (marked in
color in the above example diagram).
\smallspace
{\bf Full architectures:}
The proposed matrix products approach can be readily used to {\em replace} Clebsch--Gordan operations across all possible learning architectures giving rise to significant efficiency gains. 
\\
As a proof of concept, in the following experiments we will focus on the simplest such architecture which only
computes a {\em linear combination} of many such invariants, see Appendix \refSecAlgorithm\ for code and more details (e.g. in 
practice we may want to shift the matrices by the identity to obtain a similar effect to skip connections in ResNets.)
\\
Extensions of this minimal architecture could use a deep neural network instead of a linear combination
of invariants, or can use nonlinear activation functions to modify the matrices obtained in intermediate
steps. In architectures using several layers of Clebsch--Gordan operations, such  activation
functions are restricted to functions of the scalar channel, since ``you cannot apply a
transcendental function to a vector''. Maybe surprisingly, in our matrix formulation this actually
becomes possible: Applying any analytic function to our $(2l+1)\times(2l+1)$
matrices (not element wise, but e.g. implemented as Taylor series for matrices) is also a covariant operation! Notably, Matrix exponentiation has been suggested as an efficient 
and useful operation in Neural Networks in \cite{Fischbacher2020}.
\section*{Experimental results}
Our methods yield complete representations and can thus indeed distinguish (molecular) configurations that require higher order features (see \cite{Pozdnyakov2020,Pozdnyakov2022}). This is demonstrated experimentally in Fig.~\ref{fig:overview}B/D.
\smallspace
In Fig.~\ref{fig:overview}C we used the library E3x (\cite{Unke2024e3x}), which allows switching between full tensor layers using the
Clebsch--Gordan operation and ``Fused Tensor Layers'' for which we implemented matrix multiplication instead of the Clebsch--Gordan operations. The plot shows the inference 
run time measured on CPUs for computing a function defined by two Tensor layers, depending
on the the setting of ``max degree'' and whether full or fused layers were used.
\smallspace
In another synthetic experiment, we learn an invariant
polynomial of degree 10 with
Clebsch--Gordan operations and with our matrix multiplication framework, and plot the training 
curves averaged over 10 data sets.
\\ \noindent
\includegraphics[width=0.4\textwidth]{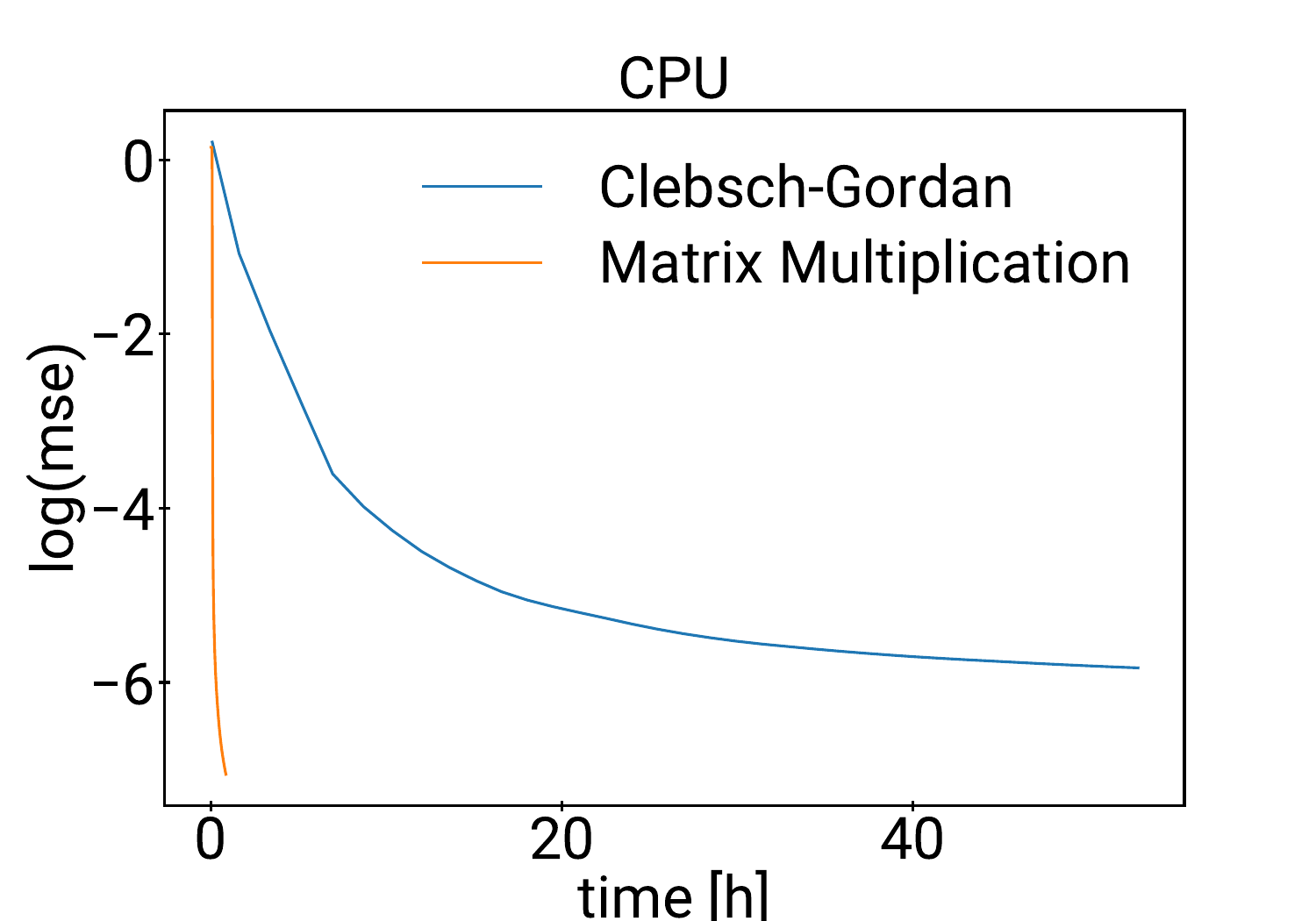}
\hskip-2mm
\includegraphics[width=0.4\textwidth]{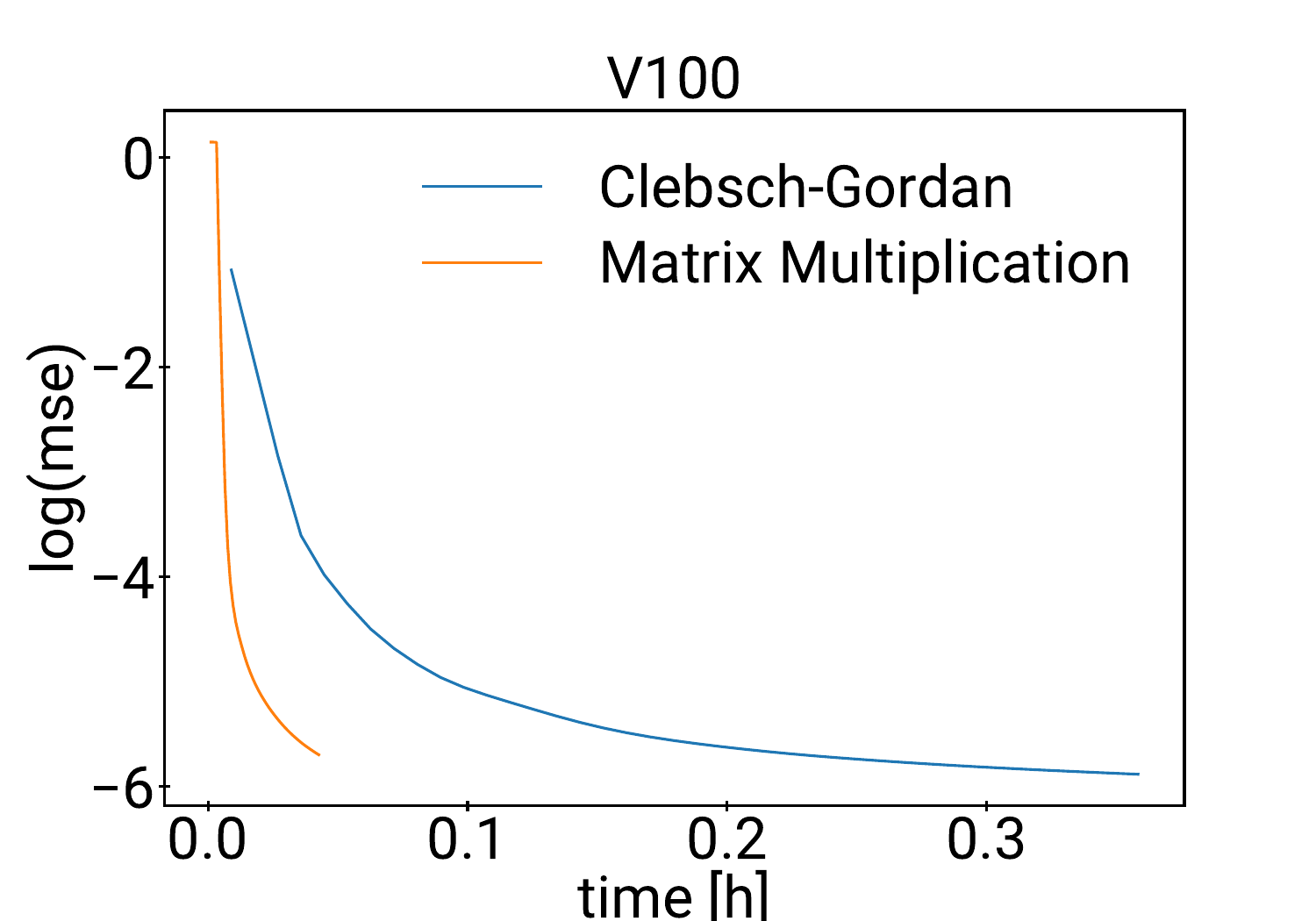}
\smallspace
On CPUs, almost all the time is spent in the Clebsch--Gordan operations, and replacing them
by the matrix multiplication method makes the training faster by a factor over 100.
When using GPUs, the speedup is not quite as dramatic, but still a factor of 8.4 on V100. (Details in Appendix~\refSecExperimentDetails)
\smallspace
As a first demonstration of our framework for atomistic simulations we show that with the simple architecture that \emph{linearly combines} the resulting polynomial invariants, we 
can learn forces with local features alone that, interestingly, can match the accuracy of other more complex  methods which use
several message passing / self attention steps with nonlinear networks (So3krates, \cite{frank2022so3krates}) or global kernel methods (sGDML, \cite{chmiela2019sgdml}). Specifically, our experiments (see Table \ref{tab:MD22}) show that accuracies align, notably, independent of the molecule sizes, see Appendix \refSecExperimentDetails\ for details. 
Since our model is just a linear combination of features of known body order and $L$,
in future studies, one could use such models to investigate body order expansions or study the influence of larger $L$s in detail.
\smallspace
Appropriately representing chemical structure and atomic environments in molecules and materials is an important prerequisite for accurate machine learning models in chemistry.
Ideal descriptors are unique, computationally efficient, and covariant.
In this work we have established an algebraic framework that enables a practical construction of provably {\em complete} system(s) of features with these desired properties that holds for any 3D point configurations. Apart from the abstract theoretical contribution of this work, we show that our construction can be readily implemented as matrix-matrix multiplication -- reducing computational complexity from $O(l^6)$ to $O(l^3)$ compared to Clebsch-Gordan operations. This yields large efficiency gains while maintaining the performance level of standard machine learning models for atomistic simulation.
\smallspace
In summary, our theoretically well founded unique, covariant, and efficient
descriptors provide a versatile basis for future atomistic modeling and potentially other applications of machine learning on point configurations. 
\begin{table}[ht]
    \centering
    \begin{tabular}{l|r|r||c|c|c}
    molecule & atoms & samples & sGDML & ours & So3krates\\
    \hline
    Ac-Ala3-NHMe      &   42 &   6000 &  0.80  &   0.47   &  0.24 \\
    DHA               &   53 &   8000 &  0.75  &   0.42   &  0.24 \\
    AT-AT             &   60 &   3000 &  0.69  &   0.43   &  0.22 \\
    stachyose         &   85 &   8000 &  0.67  &   0.33   &  0.44 \\
    AT-AT-CG-CG       &  118 &   2000 &  0.70  &   0.48   &  0.33 \\
    buckyball catcher &  148 &    600 &  0.68  &   0.27   &  0.24 \\
    Nanotubes         &  370 &    800 &  0.52  &   0.77   &  0.73 \\
    \end{tabular}
    \caption{Comparison of force accuracies [kcal/mol/\AA] for our simple linear combination of polynomial features with two more sophisticated models.}
    \label{tab:MD22}
\end{table}
\smallspace
\section*{Acknowledgement}
The authors acknowledge valuable discussions with Romuald Elie and Zhengdao Chen.\\
Correspondence to HM (hartmutm@google.com) and KRM (klausrobert@google.com).
\newcommand{\lalgso}{{\mathfrak{so}}}      
\newcommand{\refTheoremMatMultComplete}{5}
\newcommand{\eqrefMatrixInvariants}{(3)}
\newcommand{\refFigOverview}{1}
\newcommand{\refTabMD}{I}
\appendix
\setcounter{theorem}{0} 
\noindent{\huge \bf Appendix}
\section*{Organization of the appendix}
\noindent
{\footnotesize(Headlines in bold are hyperrefs to the corresponding section.)}
\smallspace
\hyperref[sec:OverviewDiagram]{\bf{A. Overview diagram statements and proofs}}\\
Page \pageref{sec:OverviewDiagram}: We give an overview diagram that explains the main statements and their
logical connection. This is meant to accompany the later sections, it is not intended to be 
fully read and understood on its own.
\smallspace
\hyperref[sec:Variants]{\bf{B. Variants of \texorpdfstring{$G, X$}{G, X}, and function spaces}}\\
Page \pageref{sec:Variants}:
Most statements will be true in different settings that we consider, we explain these settings here.
The most important distinction is between the polynomial case (which is used in the main part)
and the practically more relevant case of combining polynomial functions of the angular part
with more general radial functions; some first relations between these two cases are explained.
\smallspace 
\hyperref[sec:FiniteConditions]{\bf{C. Weaker conditions on radial functions}}\\
Page \pageref{sec:FiniteConditions}:
We prove two ``finitary'' consequences of the two conditions on radial basis functions.
In many cases one can directly use these consequences instead of the original conditions
to prove statements about configurations with fixed numbers of points.
\smallspace
\hyperref[sec:CompleteTheorems]{\bf{D. General theorems}}\\
Page \pageref{sec:CompleteTheorems}:
We give exact formulations of the general theorems (Theorems 1 -- 4) in the different 
settings we consider.
\smallspace
\hyperref[sec:Algorithm]{\bf{E. Pseudo code for algorithm}}\\
Page \pageref{sec:Algorithm}:
We detail the simplest algorithm for parameterized invariants using matrix multiplication
and the ``matrix of matrices'' approach.
\smallspace
\hyperref[sec:Incompleteness2point]{\bf{F. Incompleteness of 3 body functions}}\\
Page \pageref{sec:Incompleteness2point}:
We illustrate with 2-dimensional examples why 3 body functions are not enough
to uniquely characterize point configurations on the circle.
\smallspace
\hyperref[sec:RepresentationTheory]{\bf{G. Some background from representation theory}}\\
Page \pageref{sec:RepresentationTheory}:
Some specific points from the representation theory of compact groups that are used
in the following.
\smallspace
\hyperref[sec:PPSDs]{\bf{H. Polynomial point set descriptors (PPSDs)}}\\
Page \pageref{sec:PPSDs}:
We define PPSDs, compare them to functions on point configurations of fixed size, 
and prove the uniqueness of a Normal Form, which is used in the proof of Theorem~\ref{theorem:InvariantTheory}.
\smallspace
\hyperref[sec:ProofTopologicalCompleteness]{\bf{I. Proof of Theorem~\ref{theorem:TopologicalCompleteness}}}\\
Page \pageref{sec:ProofTopologicalCompleteness}:
Topological completeness: Using {\em all} invariant (/covariant) polynomial point set descriptors (with values
in a given representation of $G$) gives ``complete'' sets of descriptors. These are mainly topological 
arguments (embedding into finite dimensional representation, separation of compact orbits, approximation
by polynomials on compact subsets).
This will be complemented in later sections by algebraic completeness theorems that say that specific
constructions span the relevant subspaces of polynomials.
\smallspace
\hyperref[sec:ProofFiniteFeatures]{\bf{J. Proof of Theorem~\ref{theorem:FiniteFeatures}}}\\
Page \pageref{sec:ProofFiniteFeatures}:
Finiteness: Assuming polynomial functions (case i), or analytic radial basis functions and
points in a compact region $X$ (case 2ii), we prove
that $5n-6$ features are enough to distinguish $G$--equivalence classes of configurations of $n$ points in $\BR^3$. We also give counterexamples that this statement fails in case 3ii (no
compactness assumption) or when we only require the radial basis functions to be smooth instead of analytic.
From weaker assumptions we can at least show that finitely many features are enough.
\smallspace
\hyperref[sec:ProofAlgCompleteness]{\bf{K. Proof of Theorem~\ref{theorem:AlgCompleteness}}}\\
Page \pageref{sec:ProofAlgCompleteness}:
Algebraic completeness for features based on spherical harmonics and Clebsch--Gordan operations.
This is essentially the isotypical decomposition of the $G$--representation given by the PPSDs.
\smallspace
\hyperref[sec:ProofInvariantTheory]{\bf{L. Proof of Theorem~\ref{theorem:InvariantTheory}}}\\
Page \pageref{sec:ProofInvariantTheory}:
Algebraic completeness for features based on tensor moments and tensor contractions.
This is added for completeness (and because the setting of PPSDs and their Normal Form allows an easy
and direct reduction to the first fundamental theorem of invariant theory); but this theorem is not used
in the rest of the paper.
\smallspace
\hyperref[sec:MatrixMomentEx]{\bf{M. Matrix moments examples}}\\
Page \pageref{sec:MatrixMomentEx}:
The $3\times 3$ and $5\times 5$ matrix moments given as concrete matrices of polynomials.
\smallspace
\hyperref[sec:ProofTheoremMatMult]{\bf{N. Proof of Theorem \refTheoremMatMultComplete}}\\
Page \pageref{sec:ProofTheoremMatMult}:
Algebraic completeness for features based on spherical harmonics and matrix multiplication.
\smallspace
\hyperref[sec:JaxMatrixMultiplication]{\bf{O. JAX implementation of matrix multiplication}}\\
Page \pageref{sec:JaxMatrixMultiplication}:
A peculiar run time observation for large numbers of small matrix multiplication on accelerators with
hardware support for matrix multiplication. 
This was a major motivation for the ``matrix of matrices'' construction.
\smallspace
\hyperref[sec:ExperimentDetails]{\bf{P. Details for experiments}}\\
Page \pageref{sec:ExperimentDetails}:
Description of the numerical experiments on synthetic data and quantum chemistry data.
\section{Overview diagram statements and proofs}
\label{sec:OverviewDiagram}
\includegraphics[width=\textwidth]{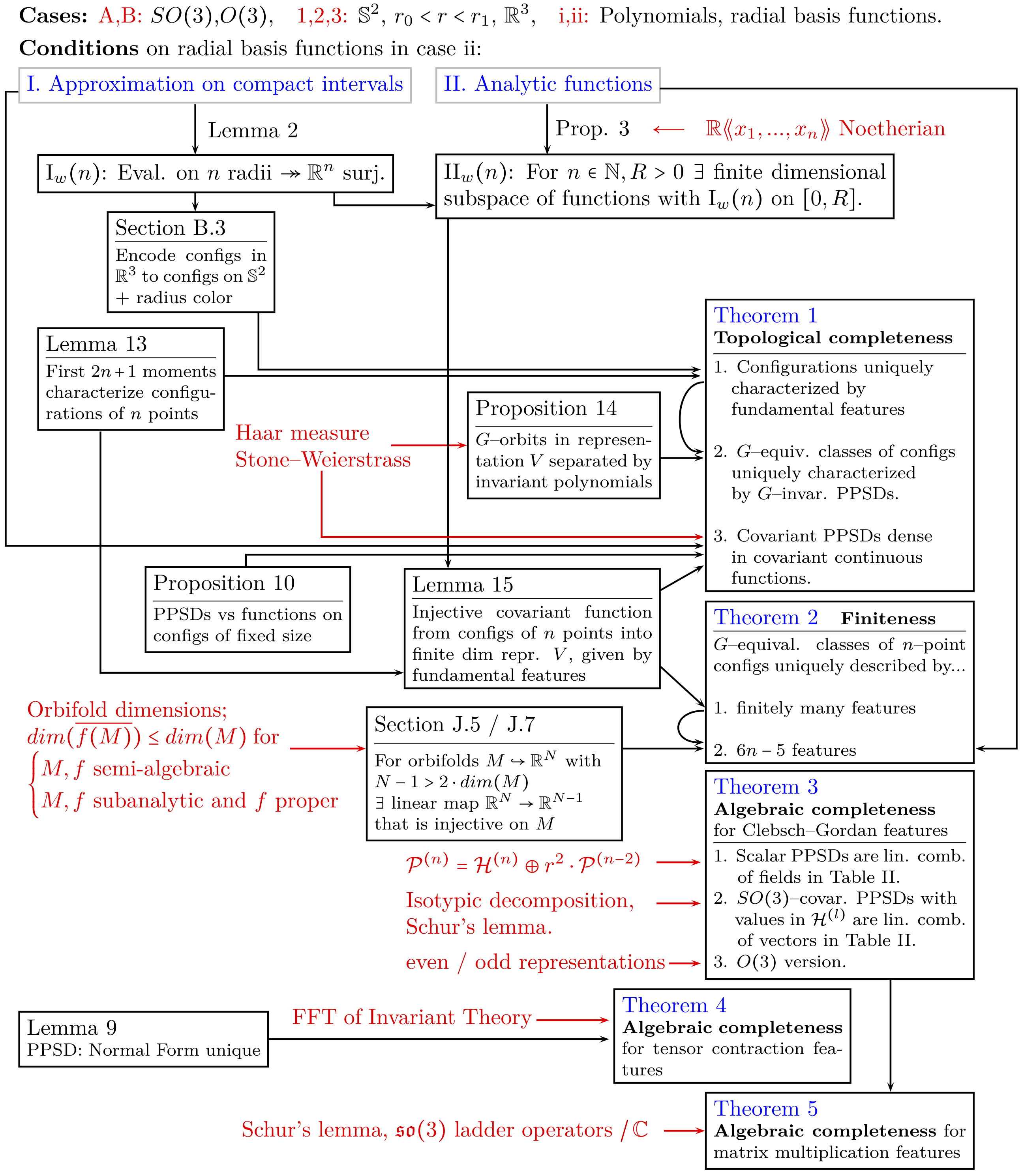}
\smallspace
The boxes refer to statements formulated and proven in this appendix, the red
text snippets refer to statements that can be found in the references.
Arrows show which statements are used in the proof of another statement. (Crossings
of lines have no particular significance, the arrows should be considered as independent.)
\smallspace
{\bf Disclaimer:} This diagram is meant as a slightly simplified overview
to accompany the actual definitions, theorems, and proofs. The short formulations
should remind the reader which statement is meant, but not give a complete
formulation of the prerequisites and detailed statements.
\smallspace
Explanations for some specific parts of the diagram:\\
\begin{itemize}
  \item Theorem \ref{theorem:TopologicalCompleteness}, part 1:\\
    Lemma \ref{lemma:FundFeatDetConfig} is enough in case i (polynomials), but we need the arguments from section \ref{subsec:VariantsDistinguishingCapabilities}
    for case ii (radial basis functions).
  \item Haar measure:\\
    This means the existence of an invariant measure on our compact
    group $G=SO(3)$ or $G=O(3)$, it is used to average polynomials over
    $G$ to get $G$--invariant (or $G$--covariant) polynomials.
  \item Lemma \ref{lemma:InjMapToVS}:\\
    In case i (polynomials) the statement as formulated in the diagram follows from Lemma \ref{lemma:FundFeatDetConfig}.\\
    The actual formulation of Lemma \ref{lemma:InjMapToVS} in Section \ref{subsec:ConfFinitelyManyFeatures} is for case ii (radial basis functions) and assumes the points in the configurations are inside some large ball with radius $R$. This allows applying II$_w(n)$, but only gives Theorem \ref{theorem:FiniteFeatures} in cases 1,2 (and it cannot give more, since Theorem \ref{theorem:FiniteFeatures} is not true in case 3ii).
  \item Theorem \ref{theorem:TopologicalCompleteness}, part 2:\\
    We could use Lemma \ref{lemma:InjMapToVS} together with Proposition \ref{prop:G_invariant_functions}, but that requires the strong condition II, or its finite version II$_w(n)$.
    Instead we use the compactness of the orbits and the weak condition I$_w(n)$, which is enough for Theorem \ref{theorem:TopologicalCompleteness} part 2.
  \item Proposition \ref{prop:NFPolynomial_unique} / Theorem \ref{theorem:TopologicalCompleteness},
    part 3:\\
    The full density statement as formulated in Appendix \ref{sec:CompleteTheorems} is requiring density only on compact subsets of the (infinite dimensional) configuration space, Proposition \ref{prop:NFPolynomial_unique} relates this to (finite dimensional) bounded pieces of the configuration space of a fixed number of points of each color.
  \item Theorem \ref{theorem:FiniteFeatures}: This is only true in cases
    1, 2i, 2ii, 3i, we give counterexamples for case 3ii.
  \item $\calP^{(n)} = \irred{n} \oplus r^2\cdot \calP^{(n-2)}$:\\
    $\calP^{(n)}$ are the homogeneous polynomials of degree $n$ in $x,y,z$,
    $\irred{n}$ is the subspace of homogeneous polynomials, 
    $r^2=x^2+y^2+z^2$.
  \item Schur's lemma:\\
    This is printed in red since it is a well known theorem of representation theory. However, we also include its proof as Lemma \ref{lemma:Schur} since we use it not (as usual) for complex representations, but for odd dimensional real representations.
  \item FFT:\\
    First Fundamental Theorem, see Appendix \ref{sec:ProofInvariantTheory}.
\end{itemize}

\section{Variants of \texorpdfstring{$G, X$}{G, X}, and function spaces}
\label{sec:Variants}
\subsection{Description of variants}
\noindent
We will distinguish the different set ups based on three criteria:\\
The symmetry group (options A/B), the point domains (options 1/2/3), and the function
spaces (i/ii). We will refer to the arising cases by combinations of these tags, e.g.
the tag ``2i'' will mean any symmetry group (i.e. options A or B), point domain 2, and function class i. Most theorems are valid for any of these cases, with slightly differing proofs.
The exception is the finiteness theorem (Theorem \ref{theorem:FiniteFeatures}), which is valid
only in the cases 1, 2, and 3i, but not in the case 3ii.
We now describe these three criteria in turn.
\vskip 5mm \noindent
{\bf \underline{Symmetry group $G$} }
\vskip 5mm \noindent
We can distinguish the symmetry groups
\begin{enumerate}
    \item[A.] G=SO(3)
    \item[B.] G=O(3)
\end{enumerate}
{\bf Usage:}\\
We would construct $O(3)$ invariants for scalars that are invariant under reflection.
If the scalars we want to approximate can distinguish between a point set and its mirror image,
then we need $SO(3)$--invariants (or if we know that it is a pseudoscalar, i.e.\ gets a factor $(-1)$
when reflected, we model it as a $O(3)$--covariant). 
More generally for covariants, $O(3)$ gives finer information, so if we have
the information how the output should change under reflection, we would use $G=O(3)$.
\smallspace
{\bf Representations:}\\
The group $O(3)$ is the direct product of $SO(3)$ and $\{Id, -Id\}$.
Therefore, irreducible representations of $\rho_{l,\sigma}: O(3)\arrow O(2l+1)$ correspond to irreducible
representations $\rho_l:SO(3)\arrow SO(2l+1)$ with an additional sign $\sigma\in \{\pm 1\}$ such that
$\rho(-Id_3) = \sigma \cdot Id_{2l+1}$. 
For $\sigma = 1$ we call the representation even, for $\sigma=-1$ odd.
\vskip 5mm\noindent
{\bf \underline{Point domains $X\subseteq \BR^3$} }
\vskip 5mm
\noindent
We distinguish three cases of subsets of $\BR^3$ from which the points are taken:
\begin{enumerate}
    \item Sphere: $X=\BS^2$
    \item Spherical shell: $X=\{\vv r \in \BR^3\,\Big|\, r_0 \leq |\vv r| \leq r_1\}$
    \item Full space: $X=\BR^3$
\end{enumerate}
{\bf Usage:}\\
In chemistry we may want to model energy contributions from each atom of a molecule. If another atom
gets too close to the central atom (that we use as origin of our coordinates), the energy approaches
infinity. Since we are usually only interested in conformations below some energy cutoff, this also
gives a lower limit $r_0$ of the distances between atoms. On the other hand, we may only use contributions
of atoms that are within a certain radius $r_1$ from the central atom, and ignore or model separately 
longer range interactions. So for this application, a domain of the form 2, i.e.
\[
    X=\{\vv r \in \BR^3\,\Big|\, r_0 \leq |\vv r| \leq r_1\}
\]
seems the most appropriate. The variant 1 ($X=\BS^2$) is mainly important for us as an intermediate
step to prove theorems about variants 2 and 3, and variant 3 is mainly to state 
general theorems without specifying $r_0, r_1$.
\vskip 5mm\noindent
{\bf \underline{Function spaces} }
\vskip 5mm
\noindent
When $X$ is not only the sphere, we also allow the function spaces:
\begin{enumerate}[label={\roman*.}]
    \item polynomials, or
    \item linear combinations of the constant 1 and of functions of the form
      \[
         \vv r\ \mapsto \ g_i(|\vv r|) \cdot P\Big(\frac{\vv r}{|\vv r|}\Big) 
      \]
      with a polynomial $P$ and $g_i$ one of a set of continuous radial basis
      functions $g_i:\BR_{\geq 0}\arrow \BR$ with $g_i(0)=0$ for $i$ in some index set $I$.
\end{enumerate}
{\bf Conditions on the radial basis functions:}\\
We required $g_i(0)=0$ above because otherwise the product with a non-constant polynomial on $\BS^2$
would not be well defined at the origin. Furthermore, we require
\begin{enumerate}
    \item[I.] Any continuous function on a compact interval $[a,b]$ with $0<a<b$ can be 
      approximated uniformly by linear combinations of the $g_i$.
    \item[II.] The $g_i$ are analytic functions.
\end{enumerate}
The condition II is only fully needed in Theorem \ref{theorem:FiniteFeatures} for obtaining the
upper bound on the number of features needed for unique characterization. For other statements,
weaker conditions are enough, see section \ref{sec:FiniteConditions} below.
\smallspace
{\bf Usage:}\\
For chemistry applications the functions on $X=\{\vv r \in \BR^3\,\Big|\, r_0 \leq |\vv r| \leq r_1\}$
that model the influence of atoms around a central atom will most likely decay towards infinity, in
fact, we would use functions that are 0 for $r\geq r_1$ to limit the number of atoms we have to process
per central atom. Similarly,  need a better resolution for small $r$ than for large $r$ will be needed.
So while with polynomials in $r^2$ we can in principle approximate any continuous function on
$[r_0, r_1]$, it may be more efficient to use radial functions that are tailored to the problem
at hand. Since $G$ does not mix points of different radii, using these functions instead of 
polynomials of $r^2$ as the radial part does not change the theory significantly.
\smallspace
{\bf Fundamental features:}\\
In case ii, we write $\calR$ for the vector
space of allowed radial functions, so for 2ii these are the $h:[r_0,r_1]\to \BR$ that are linear
combinations of the $g_i$, and for 3ii these are the $h:\BR_{\geq 0}\to \BR$ that are the linear
combinations of the $g_i$ (so in particular, we have $h(0)=0$ for all $h\in \calR$).
\smallspace
In the case i any PPSD can be given as a polynomial in fundamental features 
\[
    \sum_{\vv r\in S_\gamma} P(\vv r)
\]
for a polynomial $P:\BR^3\to\BR$ (see section \ref{subsec:PPSDsNormalForm}).\\
In the case ii we consider the features
\[
    \sum_{\vv r\in S_\gamma} c+P\Big(\frac{\vv r}{|\vv r|}\Big)\cdot h(|\vv r|)
\]
for $c\in \BR$, $P:\BS^2\to\BR$ a polynomial (considered as a function on $\BS^2$) and $h\in \calR$
as the fundamental features. (The additive constant $c$ is necessary since in case 3ii $h(0)=0$;
this ensures that --- like in the case i --- the $|S_\gamma|$ are also fundamental features.)
We then define the PPSDs to be the polynomials in the fundamental features.
(This is a slight abuse of the acronym PPSD, since as functions on configurations with $k$ points, i.e. 
on $\BR^{3k}$ these will no longer be polynomials in general.)
\smallspace
To keep arguments in cases i and ii similar, we define in case i the vector space $\calR$ 
to be the functions $f(r)$ that can be given by a polynomial $P$ as $f(r) = P(r^2)$ and
satisfy $f(0)=0$.\\
Then the fundamental features that {\em only depend on the radius} are
\[
   \sum_{\vv r \in S_\gamma}  f(|\vv r|)\qquad \hbox{for}\ f\in \BR+\calR.
\]
in both cases i and ii.
\smallspace
In case i the
\[
   \sum_{\vv r \in S_\gamma}  P(\vv r) f(|\vv r|)\qquad \hbox{for}\ f\in \calR,\ P\ \hbox{polynomial}
\]
are also fundamental features, and for case ii
\[
   \sum_{\vv r \in S_\gamma} P\Big(\frac{\vv r}{|\vv r|}\Big) f(|\vv r|)\qquad \hbox{for}\ f\in \calR,
   \ P\ \hbox{polynomial}
\]
are also fundamental features.

\subsection{Separation properties of radial functions}
\label{subsec:RadialSeparation}
\noindent
We here point out a property of the radial functions $\calR$
that will be used in different proofs and follows from Condition I alone.\\
Its proof is based on the following well known property of polynomials:
\begin{lemma}\label{lemma:LagrangeInterpolation}
Let $x_0, x_1, ..., x_n$ be $n+1$ different real numbers, then there is a polynomial
$P(x)$ of degree $n$ with $P(x_0)=1$ and $P(x_i)=0$ for $i=1,2,...,n$.\\
More generally, for any values $v_0,...,v_n\in\BR$ there is a polynomial $Q(x)$ of degree $\leq n$
with $Q(x_i) = v_i$.
\end{lemma}
\begin{proof}
For the first part, set 
\[
   P(x) := \frac{(x-x_1)\cdot (x-x_2) \cdot ... \cdot (x-x_n)}
                {(x_0-x_1)\cdot (x_0-x_2) \cdot ... \cdot (x_0-x_n)}
\]
For the general case, build for all $i=0,1,...,n$ in the same way polynomials $P_i(x)$ for which $P_i(x_j)=\delta_{ij}$ for all $j=0,1,...,n$, and then set
\[
   Q(x) := v_0 \cdot P_0(x) + ... + v_n\cdot P_n(x).
\]
\end{proof}
\noindent
For radial functions we get from this:
\begin{lemma}[Separation of radial functions] \label{lemma:RadialSeparation}\phantom{.}\\
Assume only Condition I. Then for any list of radii $0 \leq t_0 < t_1 < ... < t_n$ there exist 
functions $f_i\in \BR+\calR$ such that 
\begin{equation}
  f_i(t_i)=1 \quad \hbox{and}\qquad f_i(t_j)=0\quad \hbox{for}\ j\neq i.
  \label{eq:DeltaFunctions}
\end{equation}
\end{lemma}
\begin{proof}
  {\bf Case i: Polynomials}\\
    With Lemma \ref{lemma:LagrangeInterpolation} we can find a polynomial $P_i(t)$ such that
    $P_i(t_j^2)=\delta_{ij}$ for $i=0,1,2,...,n$ and this gives the $f_i\in \BR+\calR$  of the
    form $f_i(r):= P_i(r^2)$ with \eqref{eq:DeltaFunctions}.
  \smallspace
  {\bf Case ii: General radial basis functions}\\
  Let $f:\BR+\calR\arrow \BR^{n+1}$ be the linear map that evaluates
  a function $g\in \BR+\calR$ at the points $t_0,t_1,...,t_n$. 
  We claim that this map is surjective:
  Let $v := (v_0,...,v_n)$ be a given $(n+1)$--tuple, then
  construct a polynomial $p$ with $p(r_i) = v_i$.
  By assumption on the basis functions $g_i$ there must be
  a sequence of functions in $\BR+\calR$ that approximates $p$, and hence 
  $\vv v$ must lie in the closure of $f(\BR+\calR)$. 
  Since this image $f(\BR+\calR)$ is a vector subspace of $\BR^{n+1}$,
  it is closed, and hence $\vv v$ must already lie in $f(\BR+\calR)$.
  Now this surjectivity shows that also in the case ii we have functions 
  $f_0,...,f_n\in \BR+\calR$ with \eqref{eq:DeltaFunctions} as in the case i.
\end{proof}

\subsection{Capabilities to distinguish configurations}
\label{subsec:VariantsDistinguishingCapabilities}
\noindent
As an illustration how the different cases of function spaces and point domains are related
to each other, we here prove that they have the ``same capabilities to distinguish between
configurations'' --- this will be used in the first part of the topological completeness theorem:
Any two different configurations of colored points can be distinguished by fundamental features.
\smallspace
Obviously, the claim for 2i, 2ii, 3i, 3ii implies the claim for case 1, so we only have
to prove the other direction.\\
Assume we know this is true in case 1, i.e. on $X=\BS^2$, and we have a 
pair of configurations in $\BR^3$ with colors in $\calC$ that we want to distinguish by
fundamental features.
Let $t_0=0$ and  $0 < t_1 < ... < t_n$ be the positive radii appearing among any points in one of
the two configurations.\\
We first note that by Lemma \ref{lemma:RadialSeparation} radial functions can pick out any
particular radius. We will use the functions of Lemma \ref{lemma:RadialSeparation} to
show that encoding configurations on shells of $n$ radii $r_1,...,r_n$ is equivalent to encoding 
configurations on $\BS^2$ with $n$ times more colors:
\smallspace
First, we can discard a potential point at the origin:
From the value of the fundamental feature $\sum_{\vv r\in S_\gamma} f_0(|\vv r|^2)$ we already
see whether the origin is in a configuration and which color it has. So it is enough to consider
the points outside of the origin. 
\smallspace
We can now ``reduce'' the two configurations in $\BR^3$ colored by $\calC$ to two configurations
on $\BS^2$
colored by $\calC' := \calC \times \{1,2,...,n\}$: Map any point $\vv r\in\BR^3$ with color 
$c\in \calC$ and radius $|\vv r| = t_i$ to $\vv r / |\vv r|\in \BS^2$ with color $(c, i)\in \calC'$.
On the other hand, given one of the derived configurations on $\BS^2\times\calC'$, we can 
reconstruct the original configuration on $\BR^3\times \calC$, and the original configurations
are $G$--equivalent if and only if the reduced configurations are $G$--equivalent.
\smallspace
Since we assume we know the statement for $X=\BS^2$, there is a fundamental feature
$\sum_{\vv r'\in S'_{(\gamma, j)}} P(\vv r')$ with $(\gamma,j)\in \calC'$ that is able to distinguish them.
By decomposing $P$ into its homogeneous components, we can assume $P$ is a homogeneous polynomial
of degree $d$. But then we can re-write the fundamental feature on $X=\BS^2$ as a fundamental feature
in the original setup using the $f_i\in \BR+\calR$ from Lemma \ref{lemma:RadialSeparation}:\\
\[
  \sum_{\vv r'\in S'_{(\gamma, j)}} P(\vv r') 
  = \sum_{\vv r \in S_\gamma} f_j(|\vv r|) P\Big(\frac{\vv r}{|\vv r|}\Big)
  = \sum_{\vv r \in S_\gamma} f_j(|\vv r|) \cdot t_j^{-d} \cdot P(\vv r)
\]
(we use the second expression in case ii and the third in case i).
So then this feature in the original setup of cases 2 and 3 must also be able to distinguish the
two configurations, q.e.d.

\section{Weaker conditions on radial functions}
\label{sec:FiniteConditions}
\noindent
In case ii we required Conditions I, II from the radial basis functions $g_i$.
These conditions are enough to prove all theorems we are going to prove, and are partly selected
for their simple formulation. However, for most theorems they are more restrictive and idealized
than they need to be; in particular Condition I can only be satisfied for an infinite collection
of $g_i$.
\smallspace
Here we formulate two consequences of Conditions I, II that are used in some proofs, and that are 
enough on their own for some statements.  These are weaker conditions, and they are parameterized
by a natural number $n$ and can also be applied to finite sets of $g_i$ and give results for point
sets up to a given size.
\smallspace
\begin{itemize}
    \item Condition I$_w(n)$:\\
       For any $n$ different real numbers $t_i>0$ there is a linear combination $f$
       of the $g_i$ such that $f(t_1)=1$ and $f(t_i)=0$ for $i>1$.
    \item Condition II$_w(n)$:
       For any $R>0$ there is a {\em finite} subset $J$ of indices such that
       for any $n$ different numbers $0< t_i<R$ there is a linear combination $f$ of
       the $g_j$ with $j\in J$ such that $f(t_1)=1$ and $f(t_i)=0$ for $i=1,2,...,n$.
\end{itemize}
As an example, the $g_i(t):=t^i$ for $i=1,2,...n$ satisfy Conditions I$_w(n)$ and II$_w(n)$.
\smallspace
Obviously, Condition II$_w(n)$ is stronger than Condition I$_w(n)$. 
Lemma \ref{lemma:RadialSeparation} above showed that $I \Rightarrow I_w(n)$ for all $n$.
The main result of this section will be that Condition II$_w(n)$ follows from 
Condition I$_w(n)$ + Condition II (Proposition \ref{prop:AnalyticImpliesIIw} below).
\smallspace
We can also include $t=0$ in these conditions if we allow also the constant 1 in the linear 
combinations for $f$:
\begin{itemize}
    \item Condition $I'_w(n)$:\\
       For any $n$ different real numbers $t_i\geq 0$ there is a linear combination $f$
       of 1 and the $g_i$ such that $f(t_1)=1$ and $f(t_i)=0$ for $i>1$.
    \item Condition $II'_w(n)$:
       For any $R>0$ there is a {\em finite} subset $J$ of indices such that
       for any $n$ different numbers $0\leq t_i<R$ there is a linear combination $f$ of 1 and
       the $g_j$ with $j\in J$ such that $f(t_1)=1$ and $f(t_i)=0$ for $i=1,2,...,n$.
\end{itemize}
The primed versions are equivalent to the original conditions:
If $t_1>0$, the condition $f(0)=0$ is automatically satisfied since
all $g_i$ for $i\in I$ satisfy $g_i(0)=0$; if on the other hand $t_1=0$,
then set $f:= 1 + \sum_{j=1}^m \alpha_j g_{i_j}$ and solve for 
$\sum_{j=1}^m \alpha_j g_{i_j}(t_k) = -1$ for $k=2,3,...,n$ (we can
find such a linear combination with the same argument as we did in the
second part of Lemma \ref{lemma:LagrangeInterpolation}.)
\smallspace
Note that the formulation of Condition II$_w(n)$ involved an upper bound $R$, and the
reason is that without this finite upper bound $R$ it does {\em not} follow from
Conditions I and II:
\smallspace
{\bf Counterexample} for II$_w(n)$ without finite $R$:\\
Assume the radial basis functions are (the constant 1 and) $\sin\big(\frac{2a-1}{2^b}\cdot r\big)$,
and $\cos\big(\frac{2a-1}{2^b}\cdot r\big)-1$ for natural numbers $a,b=1,2,...$.
Linear combinations of these functions are enough to approximate
any continuous function in any compact interval of $r$, but for finitely many of them 
there will be an upper bound $b<B$ and then these functions cannot distinguish distance $r$ from
$r + 2^B\pi$, so even Condition II$_w(2)$ would not hold without finite $R$.
\smallspace
We also cannot relax the Condition II from analytic to smooth:
\smallspace
{\bf Counterexample} for II$_w(n)$ with smooth functions:\\
Let $\{g_i\,|\,i\in I\}$ be all smooth functions that are of the form
$f(t)=c + d \cdot (t-1)^2$ in an arbitrarily small neighborhood of $t=1$.
Let $K$ be a compact interval $[a,b]$ with $a\leq 1 \leq b$.
Then any continuous function on $K$ can be uniformly approximated
by using a smooth partition of one and interpolating between the original function on
$K \setminus [1-\eps, 1+\eps]$, and a function of the form $c + d\cdot (r-1)^2$ on
$[1-\eps, 1+\eps]$ and letting $\eps$ go to 0. So these functions would
satisfy the strong condition $I$.\\
However, given only finitely many basis
functions $g_i$, there would be a neighborhood of 1 in which every linear combination 
of these $g_i$ is symmetric
around 1, and these functions could not distinguish between $r=1-\eps$ and $r=1+\eps$
for $\eps$ small enough, so these function would not satisfy Condition II$_w(n)$ for $n=2$.
\smallspace
This counterexample shows the importance of the diagonal 
$\Delta := \{(t,t)\in K\times K\,|\, t\in K\}$: While for each compact subset of 
$K\times K \setminus \Delta$ we can find finitely many functions that distinguish between 
$r\neq r'$ when $(r,r')$ is in the compact subset, we cannot achieve this for all of
$K\times K \setminus \Delta$ because this is not a compact set. So the proof in the 
analytic case will use local conditions that also include the diagonal.
\smallspace
\begin{proposition}\phantom{.}\\
\label{prop:AnalyticImpliesIIw}
Conditions I$_w(n)$ + II imply Condition II$_w(n)$.
\end{proposition}
\begin{proof}
For given $R>0$ let $K$ be the compact interval $[0,R]$.\\
Since we need to include 0 in $K$ to use compactness, we
will use the version $I'_w(n)$.
We will continue to use $\calR$ for the vectors space
generated by the $g_i$, and consider them as functions on $[0,R]$.
We add a new symbol $*$ to $I$ and set $g_*(t) := 1$, then 
the $g_i$ with $i\in I^* := I \cup \{*\}$ generate $\BR + \calR$.
\smallspace
We define $\Delta_n$ as the subset of $(t_1, t_2,...,t_n)\in K^n$ 
for which two (or more) of the numbers $t_1, t_2,...,t_n$ are equal,
and for $i_1,...,i_n\in I^*$ we define the functions
\[
   f_{i_1,...,i_n}(t_1,t_2,...,t_n) :=
    \left| \begin{matrix}
    g_{i_1}(t_1) & g_{i_1}(t_2) & ... & g_{i_1}(t_n) \\
    g_{i_2}(t_1) & g_{i_2}(t_2) & ... & g_{i_2}(t_n) \\
    \vdots & \vdots & \ddots & \vdots \\
    g_{i_n}(t_1) & g_{i_n}(t_2) & ... & g_{i_n}(t_n)
    \end{matrix}\right| 
\]
and subsets
\[
  Z_{i_1,...,i_n} := \{(t_1,...,t_n)\in K^n\,\big|\, f_{i_1,...,i_n}(t_1,t_2,...,t_n)=0\}.
\]
If $(t_1,t_2,...,t_n)\in \Delta_n$, then two (or more) columns in the determinant above are the same, so 
$f_{i_1,...,i_n}(t_1,t_2,...,t_n)=0$. On the other hand, if $t_1,t_2,...,t_n$ are all different,
then Condition $I'_w(n)$ implies we can find functions $h_1,...,h_n\in \BR + \calR$ such that
$h_j(t_k) = \delta_{j,k}$; this means the $n$ column vectors $g_i(t_1),...,g_i(t_n)$ for such $i\in I^*$ that occur in 
one of the $f_j$ cannot be linearly independent, which means there must be $n$ indices $i_1,...,i_n$ such that
the $n\times n$--matrix of the $g_{i_l}(t_k)$ is non-singular. Taken together, this means that
\[
   \Delta_n = \bigcap_{(i_1,...,i_n)\in (I^*)^n}  Z_{i_1,...,i_n}. 
\]
Furthermore, for $(t_1,t_2,...,t_n)\notin\Delta_n$, we have $i_1,...,i_n$ such
that $f_{i_1,...,i_n}(t_1,t_2,...,t_n)\neq 0$ and since all functions are continuous,
this is still true in some open neighborhood $U_{t_1,t_2,...,t_n}$ of $(t_1,t_2,...,t_n)$.
This means we have an open neighborhood of $(t_1,t_2,...,t_n)$ in which
\[
   \Delta_n \cap U_{t_1,t_2,...,t_n} = \emptyset = Z_{i_1,...,i_n} \cap U_{t_1,t_2,...,t_n}. 
\]
Now we need to find something similar for $(t_1,t_2,...,t_n)\in\Delta_n$ (this is the part
that would not be possible if we only require smoothness of the $g_i$). If we can find 
an open neighborhood of $(t_1,t_2,...,t_n)\in\Delta_n$ such that in this neighborhood
$\Delta_n$ is given as intersection of finitely many $Z_{i_1,...,i_n}$, this will allow us
to use a compactness argument and get the required finite set $J$. To do this, we will
look at a sort of ``infinitesimal neighborhood''; the ring of functions in this
``infinitesimal neighborhood'' is given by germs of analytic functions, i.e. 
functions that can be expressed locally as a power series in some neighborhood of 
$(t_1,t_2,...,t_n)$ with some convergence radius $\eps>0$. This ring of convergent
power series is the same (up to isomorphism) at every point $(t_1,t_2,...,t_n)$, in 
particular we can identify it with the ring $\BR\llangle x_1,...,x_n \rrangle$ of 
convergent power series in $n$ variables, which is a Noetherian ring, see e.g. Definition 2A
and Theorem 102 in \cite{Matsumura1980}. (The ring of germs of smooth functions is not Noetherian.)
Let $\calI_{t_1,t_2,...,t_n}$ be the ideal in the ring of convergent power series at $(t_1,t_2,...,t_n)$ 
generated by all $f_{i_1,...,i_n}$ for $(i_1,...,i_n)\in (I^*)^n$. Then, since the ring is Noetherian,
this ideal can already be generated by a subset $J_{t_1,t_2,...,t_n}$ of finitely many tuples
$(j_1,...,j_n)\in (I^*)^n$. If $\eps>0$ is the smallest convergence radius for the  $f_{j_1,...,j_n}$ with
$(j_1,...,j_n)\in J_{t_1,t_2,...,t_n}$, let $U_{t_1,t_2,...,t_n}$ be the open $\eps$--neighborhood of $(t_1,t_2,...,t_n)$.
Then this means that
\[
   \Delta_n \cap U_{t_1,t_2,...,t_n} = \bigcap_{(j_1,...,j_n)\in J_{t_1,t_2,...,t_n}}  Z_{j_1,...,j_n} \cap U_{t_1,t_2,...,t_n}. 
\]
Now we can apply compactness: Since the open subsets $U_{t_1,t_2,...,t_n}$ cover the compact set $K^n$,
already finitely many of them do, let $J_n$ be the union of the corresponding finitely many $J_{t_1,t_2,...,t_n}$.
That means we have
\[
   \Delta_n = \bigcap_{(j_1,...,j_n)\in J_n}  Z_{j_1,...,j_n}. 
\]
So for any $n$ different numbers $0\leq t_i<R$ there must be a tuple $(j_1,...,j_n)\in J_n$ with
$f_{j_1,...,j_n}\neq 0$, so we can find a linear combination $f$ of those $g_j$ with $j\in J$
such that $f(t_1)=1$ and $f(t_i)=0$ for $i=1,2,...,n$, so II$_w(n)$ is satisfied.
\end{proof}

\section{General theorems}
\label{sec:CompleteTheorems}
\noindent
In this section we formulate the general, abstract theorems described informally in the main text.
\smallspace

\begin{theorem}[Topological Completeness]
  Let $G$ be either $O(3)$ or $SO(3)$.
  \begin{enumerate}
  \item  Any two different configurations of points in $X\subseteq\BR^3$ colored by $\calC$
      can be distinguished by their fundamental features.
  \item Any two configurations of points in $X\subseteq\BR^3$ colored by $\calC$ that are not \emph{equivalent under $G$}
  can be distinguished by \emph{$G$--invariant} PPSDs.
  \item \emph{$G$--covariant} PPSDs are uniformly dense in \emph{$G$--covariant} continuous functions on compact subsets of configurations colored by $\calC$.
  \end{enumerate}
  \label{theorem:TopologicalCompleteness}
\end{theorem}
\noindent
This theorem applies to all cases 1,2i,2ii,3i,3ii. 
For a proof (and an exact description of the third statement) see Appendix \ref{sec:ProofTopologicalCompleteness}.
In case ii, it only uses Condition I for the first two parts, 
and additionally Condition II$_w$ (or the stronger Condition II) for the last part.
\smallspace
These are infinitely many features, and as explained in the main text, without bound on the number of points
a finite subset cannot give unique descriptors.
For a given bound on the number of points one can show:
\begin{theorem}[Finiteness] \label{theorem:FiniteFeatures}
Given a set of $G$--invariants that distinguishes all $G$--equivalence classes of colored
point sets with at most $k_1,...,k_c$ points of colors $1,...,c$, and $k:=k_1+...+k_c$,
there are $2\dim(X)\cdot k - 5$ linear combinations of them that already distinguish all
equivalence classes.
\end{theorem}
\noindent
With $dim(X)=dim(\BR^3)=3$ in case 3 this gives the bound $6k-5$ stated in the main part.
This theorem is true in the cases $1,2i,2ii,3i$, but not in case $3ii$.
For the proof we use dimension arguments, depending on the 
function classes used, this uses dimensions of semi--algebraic sets or of subanalytic sets. 
In case 2ii it uses Conditions I and II. The weaker statement that finitely many features
suffice for configurations of $k$ points can be proven from only Conditions I and II$_w(2k)$.
\smallspace
The proof also shows that 
not only such linear combinations exist, but that $6k-5$ linear combinations ``picked at
random'' will have this property with probability 1, and that reducing a set of features
to this size will reduce distances between configurations by at most a constant factor.
See Appendix \ref{sec:ProofFiniteFeatures} for further discussion and proofs.
\smallspace
To describe the algebraic completeness theorem in more detail, we use the schematic depiction of 
PPSDs in Table \ref{tab:BodyOrderAndL}:
\newpage
\begin{table}[h]
    \centering
    \includegraphics[width=0.7\textwidth]{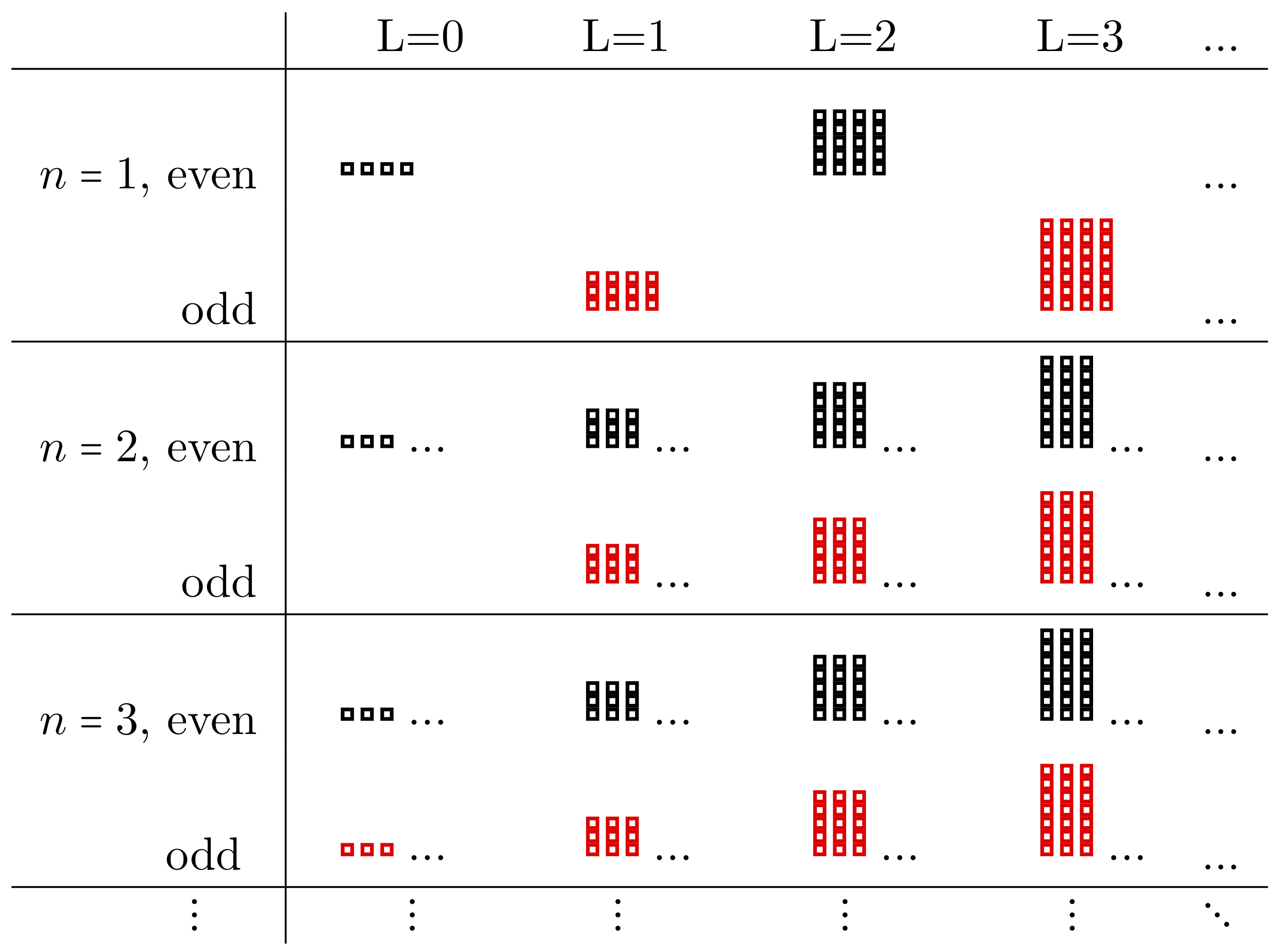}
    \caption{
    \textbf{Schematic depiction of PPSDs by body order and isotypical component:}
    In the first row, for each $l$ and color a $2l+1$--dimensional vector, sum of the spherical
    harmonics $Y_l(\vv r)$ over all points $\vv r$ of that color.
    The $n$--th row for $n>1$ consists of the results of all Clebsch--Gordan operations of rows $n-1$ and $1$.
    }
    \label{tab:BodyOrderAndL}
\end{table}

\smallspace
Given a point set in $\BR^3$ colored by $c$ colors, the first row contains fundamental
features. At $L=l$ these are in case...
\begin{itemize}
  \item[1]: $c$--tuple of vectors of dimension $2l+1$. These are the sum of the spherical
harmonics $Y_l(\vv r)$ summed over all points $\vv r$ of color $i=1,2,...,c$.
\end{itemize}
and in cases 2,3 depending on the function class in case...
\begin{itemize}
  \item[i]: Infinitely many vectors of dimension $2l+1$, sums of 
    $|\vv r|^{2k}\cdot Y_l(\vv r)$ of points of one color.
  \item[ii]: Infinitely many vectors of dimension $2l+1$, sums of 
    $g_i(|\vv r|)\cdot Y_l\big(\vv r/|\vv r|\big)$ of points of one color.
\end{itemize}
The $n$--th row consists of the results of all Clebsch--Gordan operations of
rows $n-1$ and $1$. In each row $n>1$ at column $L=l$ there are infinitely many vectors of dimension $2l+1$.
\\
We can now prove:
\begin{theorem}[Algebraic completeness when using Spherical Harmonics and Clebsch--Gordon operations]\label{theorem:AlgCompleteness} \hfill
    \begin{enumerate}
      \item All scalar PPSDs are some linear combination of fields in this schema.
      \item Any \emph{$SO(3)$--covariant} PPSD with values in $\irred{l}$ is a linear combination of vectors in the $l$--th column.
      \item Any \emph{$O(3)$--covariant} PPSD with values in $\irred{l}$ is a linear combination of vectors in the $l$--th column of the appropriate parity.
    \end{enumerate}
\end{theorem}
\noindent
This theorem is valid for all cases 1,2i,2ii,3i,3ii.
For a proof (and details for part 3) see Appendix \ref{sec:ProofAlgCompleteness}.
The first two parts of the topological completeness theorem can be expressed in this schema
using the values of the functions on configurations:
\begin{itemize}
  \item Two colored point sets have the same first row if and only if they are equal.
  \item Two colored point sets have the same first column if and only if they are $SO(3)$--equivalent.
\end{itemize}
\smallspace
In this schema, the $(n+1)$--body information appears in row $n$. So the results of \cite{pozdnyakov2020incompleteness} show that the first 3 rows are not enough to 
distinguish all $SO(3)$--equivalence classes of colored point sets (even these are infinitely
many invariants). However, Theorems 1 and 3 together show that without the
restriction of rows we can distinguish all equivalence classes.
\smallspace
There is in fact a different route to describe $SO(3)$--invariant PPSDs that does not use the
decomposition into irreducible components. This starts with the moment tensors $\sum_{\vv r\in S_\gamma} \vv r^{\otimes k}$
as fundamental features, and then applies tensor products and contractions. For $G=SO(3)$, also the
3-tensor giving the orientation, i.e. the Levi-Civita symbol $\eps_{ijk}$ is allowed in the tensor products.
If any such sequence of tensor products and contractions arrive at a scalar, this must be an $SO(3)$--invariant feature of the colored point set, and one can show:
\begin{theorem}[Algebraic completeness for tensor contraction features]\label{theorem:InvariantTheory} \hfill\\
Any  $SO(3)$--invariant PPSD is a linear combination of these scalars from contractions of tensor products.
\end{theorem}
\noindent
We formulated this theorem for case i only (polynomial functions).
This statement follows from the uniqueness of a normal form of PPSDs and the First 
Fundamental Theorem of Invariant Theory for the group $SO(3)$, see Appendix \ref{sec:ProofInvariantTheory}. An analogue for polynomial functions in a fixed number of points and $G=O(3)$ and $|\calC|=1$  was proved in \cite{shapeev2016moment}.
\smallspace
This approach gives a much smaller list of invariants for small $|\calC|$, but that is mainly an advantage 
for $|\calC|=1$ which is less relevant for applications in chemistry. In the straightforward
implementation they are expensive to calculate, and we will not make use of them in the rest of the paper.
%
\section{Pseudo code for algorithm}
\label{sec:Algorithm}
When we want to learn an invariant property $f(\{(\vv r_i, \gamma_i)\})$ of
colored point sets, we can use different Machine Learning algorithms to match it 
to parameterized functions $g(\{(\vv r_i, \gamma_i)\}, \theta)$ for a collection of
parameters $\theta$. For example, for Stochastic Gradient Descent we would start 
with random initial parameters $\theta_0$, and then in each step improve the $\theta$
to minimize a loss function of $\theta$ that depends on how close the 
$g(\{(\vv r_i, \gamma_i)\}, \theta)$ match the $f(\{(\vv r_i, \gamma_i)\})$ for the 
configurations in a training set. The description in the main part gives invariant / covariant
features of configurations $\{(\vv r_i, \gamma_i)\}$ that can be used to 
construct such parameterized functions $g(\{(\vv r_i, \gamma_i)\}, \theta)$.
In the following, we give pseudo code for an implementation of the simplest such parameterized function, 
which just computes a linear combination of invariant features.
\smallspace
Unless we have some special information that the true function $f$ must be homogeneous of a certain degree
in the fundamental features, we will in general use $n_{mat}>1$ matrix products with different lengths $b_i$,
to be able to mix invariants of different body order.\\
Since we want invariants, the covariants of form \eqrefMatrixInvariants\ should actually
be scalars, i.e. the output is in $\irred{a_m}$ with $a_m=0$, so the products look like
\begin{equation*}
   M_{a_{m-1}, 0, l_m}(\gamma_m)\cdot ... 
   \cdot M_{a_1, a_2, l_2}(\gamma_2) \cdot M_{0,a_1, l_1}(\gamma_m)
\end{equation*}
\hskip 30mm\includegraphics[width=0.55\textwidth]{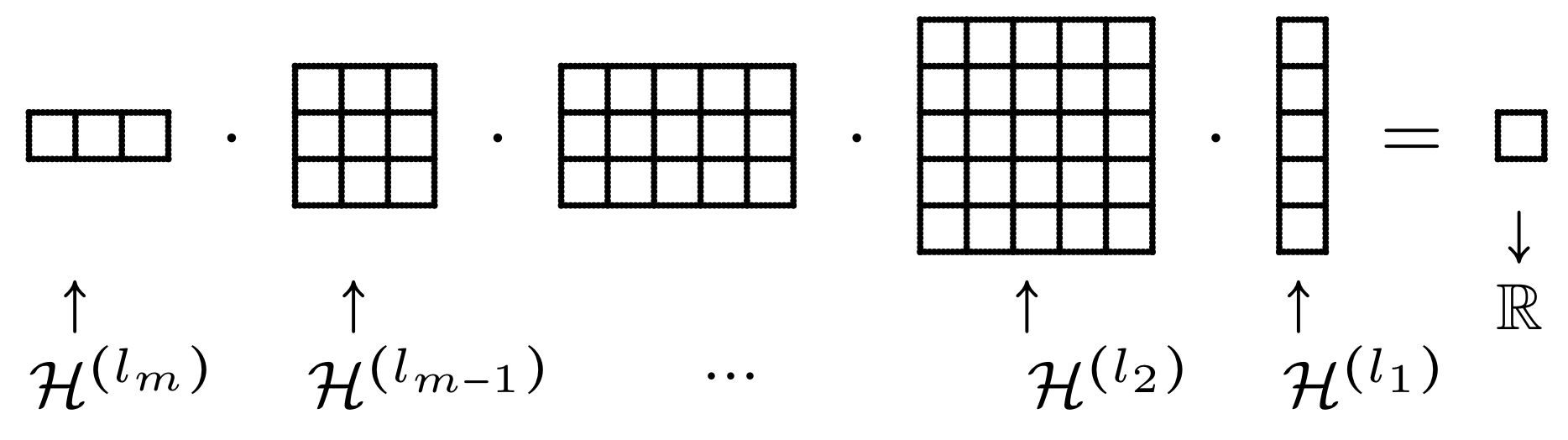}
\\
\noindent
In a concrete implementation, we can choose to re--use some of the computations for different invariants.
Here we will do this in a simple way -- we re--use the matrix product in the middle for $n_{vec}$ different
pairs of vectors at both ends.\\
Bundling $n_{vec}$ vectors of size $2l_1+1$ on the right end can be expressed by removing the first vector and
using a $(2l_1+1) \times n_{vec}$ matrix $V$ on the right:
\smallspace
\phantom{.}\hskip 30mm\ 
\includegraphics[width=0.6\textwidth]{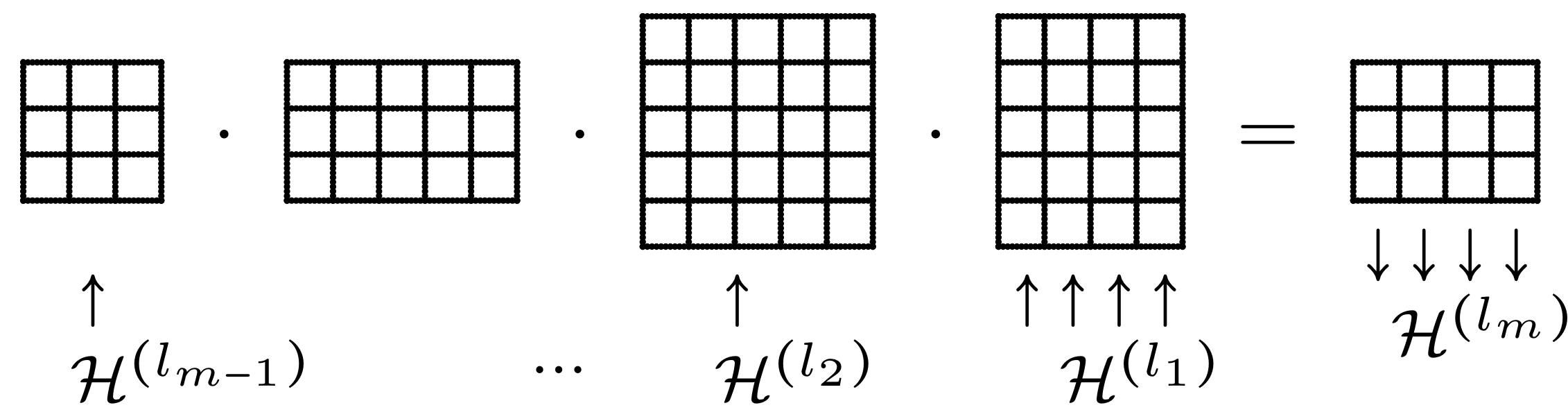}
\\
Each of the resulting $n_{vec}$ vectors in $\irred{l_m}$ can then be paired with each of another $n_{vec}$ vectors
in $\irred{l_m}$ to give $n_{vec}^2$ scalar products.

To apply the ``matrix of matrices'' approach, we pick lengths $2l_i+1$ for $i=1,2,...,r$ which then leads
to vectors of length $K$ and $K\times K$ matrices with 
\[
    K  := (2 l_1 + 1) + (2 l_2 + 1) + ... (2 l_r + 1).
\]
In the Algorithm 1 outlined below these $K$--dimensional vectors are computed in step 2 as linear combinations
(with learnable coefficients) of the fundamental features computed in step 1.
The $K\times K$--matrices are computed in step 3.\\
In step 4 we compute the matrix products 
of $b_i$ $K\times K$--matrices with $n_{vec}$ $K$--dimensional vectors collected in 
a $K\times n_{vec}$ matrix $V$. The resulting $n_{vec}$ $K$--dimensional vectors are then
decomposed into $n_{mat}\cdot n_{vec}$ vectors of dimensions $2 l_i+1$ for $i=1,2,...,r$,
taking the scalar products with $n_{vec}$ vectors of dimensions $2 l_i+1$ for $i=1,2,...,r$ then
gives $n_{mat}\cdot r \cdot n_{vec}^2$ invariants.
\\
In the final step 5, we use some more parameters to output a linear combination of 
the invariants. The completeness theorems say that such linear combinations are already capable of approximating
any invariant feature, and we will use this also in a concrete application.
\smallspace
\begin{algorithm}
\caption{Parameterized $SO(3)$--invariant}\label{alg:WithoutRadialFunctions}
  \DontPrintSemicolon
  {\bf Input:} Points $\vv r_i\in \BR^3$ of color $\gamma_i\in \calC$ for $i=1,...,n$\;
  {\bf Hyperparameters:}\;
  \ \ Number of vectors for scalar products: $n_{vec}$\;
  \ \ Number matrix products $n_{mat}$\;
  \ \ Integers $0\leq l_1 \leq ... \leq l_r$ corresponding to matrix\;
  \ \ \ \ \ sides $2l_i+1$ of submatrices.\;
  \ \ For $i=1,...,n_{mat}$: Integers $b_i\geq 0$ (corresponding to\;
  \ \ \ \ \ body orders $b_i+2$).\;
  {\bf Parameters:} Coefficients for linear combinations in\;
  \ \ vectors, matrices, and combined invariants.\;
  {\bf Compute:}\;
  \ \ 1. Spherical harmonics:\;
  \ \  \ \ For $l = 0,...,2 l_r$, $\gamma\in \calC$ set
       $Y_l(\gamma) := \sum_{\vv r\in S_\gamma}Y_l(\vv r)$   \;
  \ \ 2. Vectors:\;
  \ \ \ \ For $i=1,...,r$:\;
  \ \ \ \ \ \ Compute $2\cdot n_{mat} \cdot n_{vec}$ linear comb. of the
              $Y_{l_i}(\gamma)$ \;
  \ \ 3. Matrices:\;
  \ \ \ \ For $(a,b)\in \{l_1,...,l_r\}^2$:\;
  \ \ \ \ \ \ Compute $b_1+...+b_{n_{mat}}$ matrices of shape $a\times b$\;
  \ \ \ \ \ \ by linear combinations of $\iota_{a,b,l}(Y_l(\gamma))$ for\;
  \ \ \ \ \ \ $l=|a-b|,...,a+b$.\;
  \ \ \ \ Assemble them to $b_1+...+b_{n_{mat}}$ square matrices\;
  \ \ \ \ \ \ of shape $l_1+...+l_r$.\;
  \ \ 4. Products:\;
  \ \ \ \ For $i=1,...,n_{mat}$:\;
  \ \ \ \ \ \ Assemble $n_{vec}$ column vectors from 2. into $V$,\;
  \ \ \ \ \ \ use matrices from 3. to compute products \;
  \ \ \ \ \ \ $W := M_1 \cdot ... \cdot M_{b_i} \cdot V$.\;
  \ \ \ \ \ \ Take all scalar products of irreducible parts of\;
  \ \ \ \ \ \ columns of $W$ with vectors from 2.\;
  \ \ \ \ This gives $n_{mat} \cdot r \cdot n_{vec}^2$ invariants.\;
  \ \ 5. {\bf Output:}\;
  \ \ \ \ Linear combination of these invariants.
\end{algorithm}
\smallspace
In the main part we also mentioned another way to extract invariant scalars from a covariant matrix of matrices:
Take the traces of the square sub--matrices. This is implemented in \texttt{E3x}, together with methods to 
build up a matrix of matrices from vectors in irreducible representations. This variant of 
Algorithm~\ref{alg:WithoutRadialFunctions} (for simplicity, for one matrix product) 
would then be formulated with \texttt{E3x} as:
\smallspace
{\scriptsize
\begin{lstlisting}
from e3x.matrix import matmat
from e3x.so3 import irreps
from jax import numpy as jnp

def f(params, conf, max_degree, ls, mult, n_factors, shift_by_id):
  """Function approximation by matrix products.

  Args:
    params: List of parameters.
    conf:   Float[n_colors, n_points, 3]  Configuration of points
    l_max:  Maximal degree (L) of irreducibles in the matrices.
    ls:     List of L's for the submatrices.
    mult:   Multiplicity of each L.
    n_factors: Number of factors in matrix product.
    shift_by_id: Shift matrix multiplication by identity matrix.

  Returns:
    Estimated function.
  """
  sh = irreps.spherical_harmonics(conf, max_degree=max_degree)
  sum_sh = jnp.sum(sh, axis=1)           # [n_colors, (L+1)**2]
  sh_features = jnp.transpose(sum_sh)    # [(L+1)**2, n_colors]
  primary_features = matmat.combine_irreps(sh_features, params[0], 'high')
  product = matmat.make_square_matrix(primary_features, ls, mult, max_degree,
                                      shift_by_id, 'high')
  for i in range(1, n_factors):
    primary_features = matmat.combine_irreps(sh_features, params[i], 'high')
    matrix_features = matmat.make_square_matrix(primary_features, ls, mult,
                                                max_degree, shift_by_id, 'high')
    product = jnp.matmul(product, matrix_features, precision='high')
  prod_traces = matmat.get_traces(product, ls, mult, shift_by_id)
  result = jnp.dot(prod_traces, params[-1], precision='high')
  return result
\end{lstlisting}
}
and the parameters would be initialized by
\smallspace
{\scriptsize
\begin{lstlisting}
def init_params(
    key, ls, mult, max_degree, n_colors, n_factors,
    factor_mat, factor_final
):
  keys = jax.random.split(key, n_factors + 1)
  dict_irreds = matmat.make_dict_irreps_mult(ls, max_degree = max_degree)
  init = matmat.init_matrix_irreps_weights
  params = [
      init(keys[i], n_colors, dict_irreds, mult, factor_mat)
      for i in range(n_factors)
  ]
  params.append(
      jax.random.normal(keys[-1], (len(ls) * mult**2,)) * factor_final
  )
  return params
\end{lstlisting}
}
\smallspace
where \texttt{factor\_mat} would need to be adjusted such that the product of matrices does not go to infinity
or zero. It is easier and more efficient to use a product of matrices ``shifted by the identity matrix''
\[
   (Id + A_1) \cdot (Id + A_2) \cdot ... \cdot (Id + A_n) - Id
\]
where the $A_i$ are initially small. Here the identity matrices $Id$ are similar to skip connections in ResNets \cite{He2015DeepRL} and allow
a smooth learning of all factors, even for large $n$. In the code using \texttt{E3x} both methods are supported: The original matrix product corresponds to setting the {\texttt bool} parameter 
\texttt{shift\_by\_id} to \texttt{False}, for the second method it would be set to \texttt{True}.
\smallspace
We have pointed out that both the Clebsch--Gordan operation and matrix multiplication define maps
\[
   \left(\irred{0}\oplus...\oplus\irred{l}\right) \otimes
   \left(\irred{0}\oplus...\oplus\irred{l}\right) \ \arrow \ 
   \irred{0}\oplus...\oplus\irred{l}
\]
which are used to build up different linear combinations of covariants of higher body
order, and that while the Clebsch--Gordon maps are of complexity $O(l^6)$, the matrix multiplication 
is of complexity $O(l^3)$. However, when we use the matrix multiplication approach, we also need to convert
spherical harmonics of degrees $0,...,l$ into a $(l+1)\times(l+1)$--matrix (and similar for rectangular matrices).
While this can be efficiently encoded as a tensor contraction ({\tt einsum} in JAX), it is theoretically of order $O(l^4)$.
In theory, this can be avoided using complex spherical harmonics, which makes this also an $O(l^3)$ operation. However, for
moderate $l$ it seems the simple and efficient tensor operation is preferable to the more complicated operation using complex
numbers, so in {\tt E3x} we use the simple implementation.
\smallspace
While these linear combinations of step 5 of the outlined Algorithm 1 are enough in theory, it is very easy and
may be beneficial in practical applications to add e.g.\ a neural network computation on top of the invariants
to further increase the flexibility of the resulting functions.
\\
Further easy flexibility increases could be obtained (if necessary) by parameterized operations on the partial
matrix products $M_1\cdot...\cdot M_j$, e.g. by using linear combinations of different $a\times b$ submatrices
and the transpose of $b\times a$ submatrices or by multiplying $a\times b$ submatrices with learnable non-linear
scalar functions of the traces of $c\times c$ submatrices.\\
Other possible operations include fast implementations of transcendental matrix functions like $exp$ on $a\times a$ submatrices, and linear combinations of the individual $\irred{l}$ components of different submatrices, as implemented
by the \texttt{FusedTensor} layers in \texttt{E3x} for $r=1$.

\section{Incompleteness of 3 body functions}
\label{sec:Incompleteness2point}
An $SO(3)$--invariant PPSD of order $m$ (i.e. a $(m+1)$ body function) gives the same 
value on point sets $S, S'$ if their (multi--)set of ``subsets of cardinality $m$ modulo
$G$'' is the same. 
\smallspace
To illustrate this in the case $m=2$ note that 2--element subsets of $\BS^2$
modulo $SO(3)$ are given by the angle between them (or, equivalently, their scalar product). \\
So $SO(3)$--invariant 3 body functions on $\BS^2$ only depend on the 
(multi--)\emph{set} of angles between the points.
\smallspace
We now give examples of point configurations on the circle $\BS^1\subset\BS^2$ that are substantially different, 
but still have the same set of angles.\\
For the first example, we divide the unit circle into 30 equal angles and position points
at a subset of these 30 angles, so the points are given by a residue class modulo 30. \\
The points at
\[
    0, 1, 8, 11, 13 \qquad \hbox{or} \qquad 0, 10, 11, 13, 18
\]
give the same \emph{set} of differences mod 30 (each difference occurs only once):
\[
    \{ 1, 2, 3, 5, 7, 8, 10, 11, 12, 13\}
\]
\smallspace
\par\noindent
\includegraphics[width=0.7\textwidth]{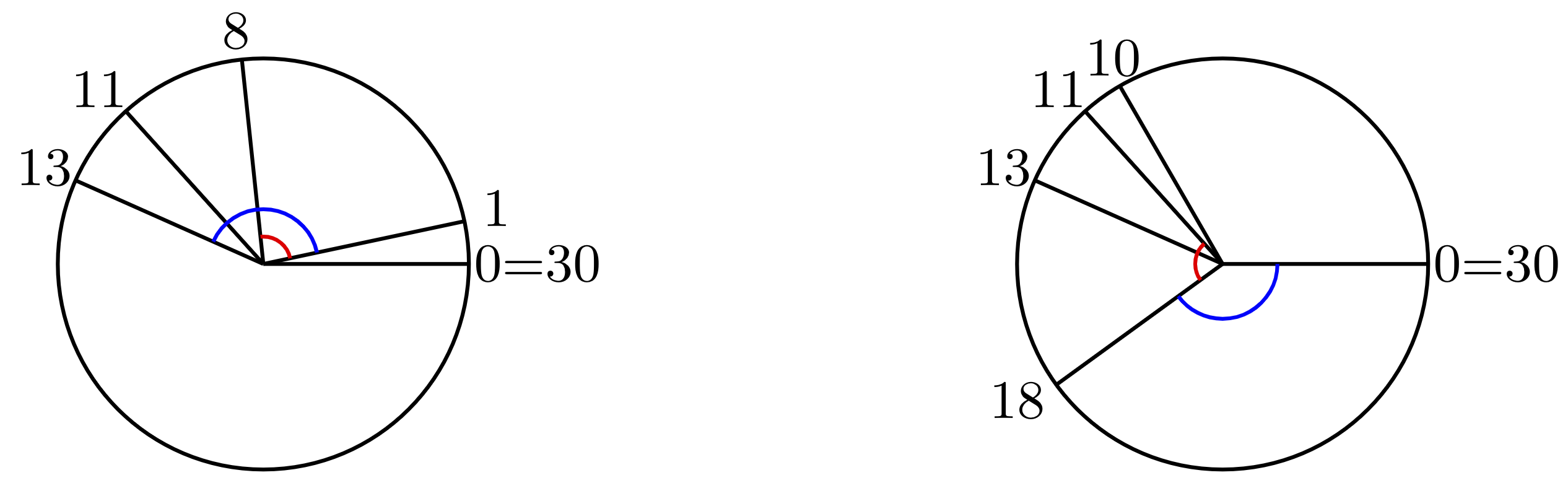}
\smallspace
While this may look like some coincidence, there are actually whole submanifolds of pairs of
configurations that cannot be distinguished, here is the simplest example:
\smallspace
The four points at angles $0,\alpha,\alpha + \frac \pi 2,\pi$ or $ 0,\alpha,\pi,\alpha - \frac \pi 2$
give the same \emph{set} of angles between them:
$\left\{ \alpha,\frac \pi 2 + \alpha,\frac \pi 2 - \alpha,\pi-\alpha,\frac \pi 2, \pi \right\}$
\smallspace
\par\noindent
\includegraphics[width=0.7\textwidth]{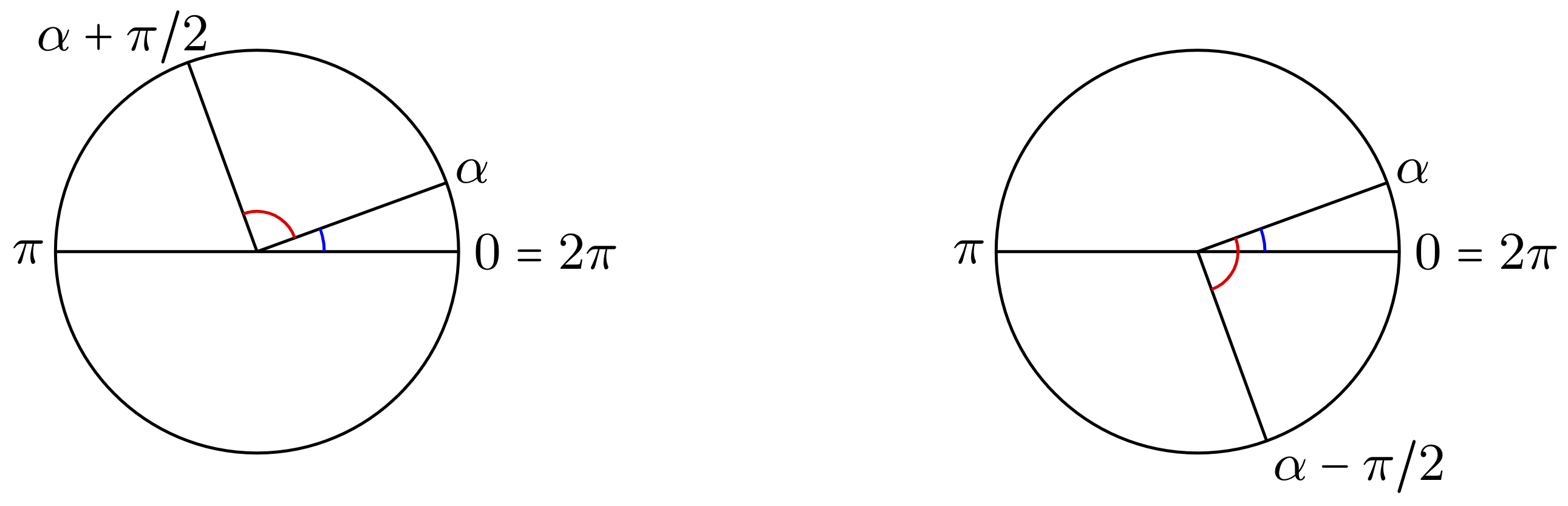}
\smallspace
For more examples and a closer analysis, see \cite{Pozdnyakov2020}.
%
\section{Some background from representation theory}
\label{sec:RepresentationTheory}
We collect here some notations and remarks about representation theory.
Most of the time we focus on representations on real vector spaces, but for 
the comparison between matrix products and Clebsch--Gordan operations we 
will also make use of complex representations.
\smallspace
An elementary introduction to representations over the real numbers, tailored 
to our purposes, can be found in \cite{Unke2024e3x},
we only repeat here some notations and point out some special details.
\subsection{The irreducible representations \texorpdfstring{$\irred{L}$ of $SO(3)$}{of SO(3)}}
\noindent
There is one irreducible representation of $SO(3)$ of degrees $2L+1$ for
$L=0,1,2,...$ given by the harmonic homogeneous polynomials of total 
degree $L$ in three variables $x,y,z$; we denote this real vector space
by $\irred{L}$.
Every irreducible representation over $\BR$ is isomorphic to exactly one of these $\irred{L}$, and likewise every irreducible representation over $\BC$ is isomorphic
to exactly one of the corresponding complex representations 
$\irred{L}\otimes_\BR \BC$.

One particular basis of the vector space $\irred{L}$ is given by the real--valued
spherical harmonics $Y_l^m$ with $m=-l,..., l-1, l$.

\subsection{Equivariant maps \texorpdfstring{$\BS^2\arrow\irred{L}$}{from the sphere to an 
irreducible representation}}
\label{subsec:EquivariantMaps}
\noindent
Let $(\rho_L,\irred{L})$ be an irreducible representation of $SO(3)$. An $SO(3)$--equivariant
function $f:\BS^2\arrow\irred{L}$ is already determined by its value at one point, e.g. at $(0,0,1)^T$, since for every other point $\vv u\in \BS^2$ there is a rotation $g\in SO(3)$ that moves $(0,0,1)^T$ to $\vv u$. Furthermore, the point $(0, 0, 1)^T\in \BS^2$ is the fixed point of
the subgroup $T\subset SO(3)$ of rotations around the $z$--axis. So it must be mapped by the equivariant $f$ to a fixed point of the group $\rho_L(T)$ in $\irred{L}$. However, 
this fixed point set in $\irred{L}$ is only a 1-dimensional vector space: Over $\BC$
a basis of the $2l+1$--dimensional $\irred{L}$ is given by the $2l+1$ complex spherical harmonics of weight $-l,...,l-1,l$  (see \ref{subsec:ComplexSphericalHarmonics}), only the
subspace of weight 0 is invariant under 
$T$, over $\BR$ it is given by the real spherical harmonics $Y_l^0$.\\
Therefore, $(0,0,1)^T$ must be mapped to a point on that line, so all equivariant maps  $f:\BS^2\arrow\irred{L}$
are the same up to a multiplicative factor.

\subsection{Schur's lemma}
\label{subsec:SchursLemma}
\noindent
Given representations with decompositions into irreducibles, Schur's lemma describes
what the maps between those representations can be. It is almost
trivial to prove, but has powerful consequences. We give a proof here since usually it is only formulated and proven for complex vector spaces, but we want to apply it
also to real representations of $SO(3)$.
\begin{lemma}[Schur's lemma]\label{lemma:Schur}
Let $V,W$ be $K$--vector spaces for $K=\BR$ or $K=\BC$, and let $\rho_1:G\arrow GL(V)$, $\rho_2:G\arrow GL(W)$ be {\bf irreducible} representations of $G$, and $f:V\arrow W$ a covariant linear map.
\begin{enumerate}
    \item $f$ is either 0 or bijective.
    \item If these representations are not isomorphic, $f=0$.
    \item If $V$ is a finite dimensional $\BC$--vector space, the only covariant linear maps $f:V\arrow V$ are the $f=\lambda\cdot Id$ for some $\lambda\in\BC$.
    \item If $V$ is an {\bf odd dimensional} $\BR$--vector space, the only covariant linear maps $f:V\arrow V$ are the $f=\lambda\cdot Id$ for some $\lambda\in\BR$.
\end{enumerate}
\end{lemma}
\begin{proof}
\begin{enumerate}
  \item Kernel and image must be mapped by $G$ into themselves, so they must be $\{0\}$ or the full space.
     If $ker(f)=V$ or $im(f)=\{0\}$, we have $f=0$. Otherwise, we must have $ker(f)=\{0\}$ and $im(f)=W$, i.e.
     $f$ is bijective.
  \item Follows from 1: If $f$ is bijective, there is an inverse $f^{-1}$ and these define an isomorphism.
  \item Over $\BC$, $f$ must have some eigenvalue $\lambda$ with eigenvector $\vv v\neq 0$. Then $f-\lambda Id$
    is again a covariant linear map; since $(f-\lambda\cdot Id)\vv v = 0$, it is not bijective, 
    so it must be 0, i.e. $f=\lambda \cdot Id$.
  \item Over $\BR$, the (multi-) set of eigenvalues of $f$ (with multiplicities) must be invariant under complex
  conjugation, since its cardinality $dim(V)$ is odd, there must be at least one real eigenvalue $\lambda$. Then we can apply the
  same argument as in 3.
\end{enumerate}
\end{proof}
\noindent
The last part is no longer true for even dimensions:\\
The two-dimensional representation $\rho: SO(2)\arrow GL(2), g\mapsto g$ is
irreducible, but any rotation is also $G$--covariant.
\\
However, since all irreducible representations of $SO(3)$ are
odd--dimensional, we can apply 4. as a replacement of 3.
\smallspace
\begin{corollary} \label{corollary:HomMatrix}
Let $\BK$ be $\BR$ or $\BC$, and let $\rho:G\arrow GL(V)$, $\rho':G\arrow GL(V')$ be representations that can be decomposed into irreducibles:
\begin{equation} \label{eq:isotypic}
   V \simeq V_1^{\oplus m_1} \oplus V_2^{\oplus m_2} \oplus ... \oplus V_n^{\oplus m_n}
\end{equation}
and
\begin{equation} \label{eq:isotypic2}
   V' \simeq V_1^{\oplus m'_1} \oplus V_2^{\oplus m'_2} \oplus ... \oplus V_n^{\oplus m'_n}
\end{equation}
with $V_1,..,V_n$ different (i.e. non-isomorphic) irreducible representations and 
\[ 
   V_i^{\oplus m_i} := V_i \oplus ... \oplus V_i \qquad \hbox{sum of}\ m_i\  \hbox{copies of}\ V_i.
\]
(We allow exponents to be 0 to be able to use the same $V_i$, with $V_i^{\oplus 0} := \{0\}$.)
If $\BK=\BR$, also assume that the $\dim(V_i)$ are odd.
Then the vector space of covariant maps $V\arrow V'$ can be described by a vector space isomorphism
\begin{equation} \label{eq:isotypic3}
    Hom[(\rho,V), (\rho',V')] \simeq \bigoplus_{\genfrac{}{}{0pt}{}{i=1}{m_i>0, m'_i>0}}^n M_{m_i,m'_i}(\BK)
\end{equation}
\end{corollary}
\smallspace
Among the many applications, we mention:
\begin{lemma}\label{lemma:unique_inv_scprod}
Let $V$ be an odd dimensional
vector space over $\BR$ and $\rho:G\arrow GL(V)$ an irreducible representation.
Then the invariant scalar product on $V$ is unique up to a factor.
\end{lemma}
\begin{proof}
Assume the invariant scalar products for $j=1,2$ are given by matrices $A_j$ as
$ \langle \vv v, \vv w \rangle_j = \vv v^T A_j \vv w$.
\\
Then being invariant under $G$ means for the matrices that
$  \rho(g)^T A_j \rho(g) = A_j$.
Then we get for the matrix $A_1^{-1} \cdot A_2$ the invariance
\begin{eqnarray*}
    A_1^{-1} \cdot A_2 
    &=& \rho(g)^{-1} A_1^{-1} \rho(g)^{-T} \cdot \rho(g)^T A_2 \rho(g) \\
    &=& \rho(g)^{-1} \left( A_1^{-1} A_2 \right) \rho(g) 
\end{eqnarray*}
which means $A_1^{-1}A_2$ is a covariant map.
So by part 4 of Schur's lemma above $A_1^{-1}\cdot A_2$ is a scalar multiple of the identity.
\end{proof}

\subsection{Isotypic decomposition}
\label{subsec:IsotypicDecomposition}
\begin{definition}[Isotypic component]\ \\
Let $\BK$ be either $\BR$ or $\BC$, and let $W$ be a finite dimensional $\BK$--vector
space, and $\rho:G\arrow GL(W)$ a representation of a group $G$.\\
For each irreducible representation $\lambda:G\arrow GL(V_\lambda)$ write $W_\lambda$ 
for the subspace of $W$ generated by the $\im(f)$ for covariant maps $f:V_\lambda\arrow W$,
it is called the isotypic component of $\lambda$ in $W$.
\end{definition} 
\noindent
Assume now $G$ is compact, then there exists a decomposition of $W$ into irreducible
components
\begin{equation}\label{eq:some_decomposition2}
  (W, \rho) \simeq (W_1,\rho_1) \oplus ... \oplus (W_m,\rho_m)
\end{equation}
Then all $(W_i, \rho_i) \simeq (V_\lambda,\lambda)$ are contained in $W_\lambda$ since
we have the covariant inclusion maps
\[
  f:V_\lambda\simeq W_i\hookrightarrow W 
  \qquad \hbox{with} \quad  \im(f)=W_i.
\]
So the sum of the $(W_i, \rho_i) \simeq (V_\lambda)$ are contained in $W_\lambda$.
\smallspace
Using Corollary \ref{corollary:HomMatrix} we can see that $W_\lambda$ cannot be larger:
\begin{corollary}\label{corollary:Isotypic}
Let $G$ be a compact group, $\BK=\BR$ or\ \,$\BC$, and $\rho:G\arrow GL(W)$ a finite dimensional representation over $\BK$. If $\BK=\BR$, also assume all irreducible representations are odd
dimensional.
Then for any decomposition \eqref{eq:some_decomposition2} the isotypic components
$W_\lambda$ are the sum of the  $(W_i, \rho_i) \simeq (V,\lambda)$.
\end{corollary}
\begin{proof}
For any covariant map $f:V\arrow W$ we have maps
\[
\begin{tikzcd}
  \vv v\arrow[mapsto]{d}
     & V \arrow[swap]{d}{f} \arrow{r}{f_i}
     & W_i
     & \vv w_i\\
  f(\vv v) 
     & W \arrow{r}{\simeq}
     & W_1 \oplus ... \oplus W_m \arrow[swap]{u}{\pi_i}
     & \vv w_1+...+\vv w_m\ \  \arrow[mapsto]{u}
\end{tikzcd}
\]
and can decompose the map $f$ as
\[
    f(\vv v) = f_1(\vv v) + f_2(\vv v) + ... + f_m(\vv v), \qquad
    f_i(\vv v) \in W_i \subseteq W
\]
If $\rho_i$ is not isomorphic to $\lambda$, then by
Schur's lemma $f_j=0$, so $V_\lambda$ can also not be
larger than the sum of all $W_i$ with $\rho_j\simeq \lambda$.\\
So $V_\lambda$ is exactly the sum of all $W_i$
with $\rho_i\simeq \lambda$.
\end{proof}
\noindent
This means that we get for representations of a compact group an isotypic decomposition:
\begin{corollary}
Let $G,W,\rho$ as before, and let $\{\lambda_1,...,\lambda_k\}$ be the 
set of (isomorphism classes of) irreducible representations of $G$ for which
$W_{\lambda_i}\neq 0$.\\
Then 
\begin{equation}\label{eq:isotypic_decomposition}
    W = W_{\lambda_1} \oplus ... \oplus W_{\lambda_k}
\end{equation}
\end{corollary}
Note that this decomposition is unique (up to the order in which we write the 
direct sum, of course), whereas \eqref{eq:some_decomposition2} was not.
\subsection{Dual, tensor product, and Hom}
\label{subsec:DualTensorHom}
\noindent
Given a representation $(\rho, V)$ over $\BK = \BR\ \hbox{or}\ \BC$, the dual space $V^*$
is the vector space of linear maps $L:V\arrow \BK$, on it the dual representation $(\rho^*, V^*)$
is defined by
\[
    \rho^*(g) L := \Big( v \mapsto L(g^{-1} x) \Big).
\]
For $v\in V, L\in V^*$ we also write $\langle L, v\rangle$ for $L(v)$.
\smallspace
The tensor product of representations $(\rho^V, V), (\rho^W, W)$ is a representation on 
the vector space $V\otimes W$. This vector space is spanned by expressions $v\otimes w$,
if $v$ and $w$ go over bases of $V,W$, the $v\otimes w$ go through a basis of $V\otimes W$.
The group operation on $V\otimes W$ has the property
\[
   \rho(g)(v\otimes w) = \rho^V(g)(v) \otimes \rho^W(g)(w).
\]
For slightly more details and examples see e.g. \cite{Unke2024e3x}.
\\
For two representations $(\rho^V, V), (\rho^W, W)$ we can form the vector space 
$Lin(V,W)$ of linear maps $V\arrow W$, this is a representation under
\[
    \rho^{VW}(g) f :=  \Big( v \mapsto \rho^W(g) f(\rho^V(g^{-1}) v) \Big)
\]
With these definitions of representations, the matrix multiplication
\[
    Lin(U,V) \times Lin(V,W) \arrow Lin(U,W)
\]
for representations $(\rho^U, U), (\rho^V, V), (\rho^W, W)$ 
and linear maps $f:U\arrow V, h:V \arrow W$ is a covariant map:
\begin{eqnarray*}
  \lefteqn{\rho^{UW}(h\circ f)(u)} \\
  &=& \rho^W(g)\  hf \Big(\rho^U(g^{-1}) u\Big)\\
  &=& \rho^W(g)\ h\Big(\rho^V(g^{-1})\ \rho^V(g)\ \ f\big(\rho^U(g^{-1}) u\big)\Big)\\
  &=& \Big(\rho^{VW}(g)(h)\Big)\ \Big(\rho^{UV}(g)(f)\Big)\ (u)
\end{eqnarray*}
\smallspace
If $V$ or $W$ are finite dimensional, the notions of $Lin$, $\otimes$ and dual representations
are related by the isomorphism of representations
\[
     V^* \otimes W \simeq Lin(V,W)
\]
which maps $v^* \otimes w$ to the linear map $v \mapsto \langle v^*, v \rangle w$.\\
(We are usually only interested in finite dimensional representations, but this
makes also sense for $W$ the infinite dimensional space of all PPSDs and $V$ 
a finite dimensional representation. If both $V$ and $W$ are infinite dimensional, this
algebraic isomorphism is no longer true: For $V=W$ the identity in $Lin(V,V)$ 
cannot be given by a finite linear combination of maps of rank 1.)
\smallspace
The homomorphisms in the category of $G$--representations are the covariant maps, 
they are just the $G$--invariant elements of the representation $Lin(V,W)$:
\[
    Hom_G(V,W) = Lin(V,W)^G = (V^* \otimes W)^G
\]

\subsection{Function spaces}
\label{subsec:FunctionSpaceReps}
\noindent
Let $X$ be a set on which $G$ acts, then $G$ also acts on the set $Map(X, \BR)$ of functions $f:X\arrow \BR$
by
\[
    \rho(g) f := \Big( x \mapsto f(g^{-1} x) \Big).
\]
The set of $G$--fixed points $Map(X,\BR)^G$ is just the set of invariant functions $X\arrow \BR$.
\smallspace
More generally, for a representation $(\rho^V, V)$ of $G$, the group $G$ also acts on the set $Map(X, V)$ by
\[
    \rho(g) f := \Big( x \mapsto \rho^V(g) \big( f(g^{-1} x) \big) \Big).
\]
The fix point set $Map(X,V)^G$ is the set of all $f$ for which for all $x\in X, g\in G$
\[
  \rho^V(g) \big( f(g^{-1} x) \big) = f(x)
\]
or equivalently $f(g x) = \rho^V(g) f(x)$ for all $x\in X, g\in G$; so this is exactly the set
of covariant maps $X\arrow V$.
\smallspace
We can map the tensor product $V \otimes Map(X,\BR)$ bijectively to $Map(X, V)$ by requiring
\[
    \iota(v \otimes f): x \mapsto f(x)\cdot v,
\]
applying this to $v_i \otimes f_j$ for bases $v_i$ of $V$ and $f_j$ of $Map(X, \BR)$ we can see that this
gives an isomorphism 
\[
    \iota: V \otimes Map(X, \BR) \xrightarrow{\sim} Map(X, V).
\]
of $G$--representations.
Similarly, we have
\[
  Lin(V, Map(X,\BR)) \simeq Map(X, V^*),
\]
an isomorphism $\iota$ is given by setting the image of $L: V \arrow Map(X, \BR)$ to
\[
   \iota(L): X\arrow V^*,  x\mapsto \big(v\mapsto L(v)(x)) \big)
\]
and this is compatible with the operation of $G$.

\subsection{Complex representations and scalar products}
\label{subsec:ComplexDual}
The aim of this subsection is to clarify a potentially confusing identification of
a representation $\rho$ with its dual $\rho^*$.
Our irreducible representations $\calH^{(l)}$ are defined over the real numbers
(are given by real valued matrices with respect to the basis of the real spherical harmonics),
but we can also consider the $2l+1$--dimensional vector space on which these 
$(2l+1)\times(2l+1)$--matrices operate as complex vector spaces on which we then have the
additional structure of complex conjugation.\\
We fix a scalar product on the real vector space that is invariant under $\rho(G)$.
In general, if we have a real vector space $V$ with scalar product $\langle\,,\,\rangle$,
we have two different notions of scalar product on its complexification $V \otimes \BC$:
The algebraic scalar product, which in an orthonormal basis of $V$ (which identifies 
$V \otimes \BC$ with $\BC^d$) is given by
\[
    \langle \vv w, \vv v \rangle_{alg} := \vv w^T\cdot \vv v
\]
and the usual Hermite scalar product
\[
    \langle \vv w, \vv v \rangle := 
    \langle \overline{\vv w}, \vv v \rangle_{alg} = \vv w^H\cdot \vv v.
\]
If the real representation $\rho$ was orthogonal on $V$, it will be unitary on $V\otimes \BC$
with respect to this Hermite scalar product.\\ 
For a representation $(V,\rho)$ on vector spaces over any field $\BK$, 
the dual representation $(V^*, \rho^*)$ is defined to operate on linear forms 
$L:V\arrow \BK$ in $V^*$ by 
\[
    \rho^*(g) L  :=  \Big( \vv v \mapsto L\big(\rho(g^{-1}) \vv v\big) \Big)
\]
The algebraic scalar product defines a vector space isomorphism between $V$ and 
its dual $V^*$: We represent the linear form $L$ by a vector $\vv w\in V$ such that 
$L(\vv v) = \langle \vv w, \vv v \rangle_{alg}$. For the representation $\rho^*$
this means we have 
\[
    \langle \rho^*(g) \vv w, \vv v\rangle_{alg} 
    = \langle \vv w,  \rho(g^{-1}) \vv v \rangle_{alg}
    = \langle (\rho(g)^{-1})^T \vv w,   \vv v \rangle_{alg}
\]
for all $\vv v, \vv w$, and for a unitary representation this is equivalent to
$\rho^*(g) = (\rho(g)^H)^T = \overline{\rho(g)}$. So for a scalar product on $V$
that is invariant under $\rho(G)$, this representation is unitary and 
the dual representation to $(V,\rho)$ is just given by the complex conjugate $\bar \rho$.
Since our representations $\calH^{(l)}$ are actually already defined over $\BR$,
this representation is isomorphic (as an abstract representation) to its dual.
(If we choose an orthonormal basis of the real vector space, $\rho$ and $\rho^*$ are given 
by the same matrices, and for arbitrary complex bases the
isomorphism can be realized by the base change that maps the complex basis vectors to
their complex conjugate.) While we are free to choose any complex basis of $\calH^{(l)}$, we 
have to keep this basis for both representations $\rho$ and $\rho^*$, so we have to
distinguish them when they both occur in the same computation.
\smallspace
From two of our representations $(\irred{l_1}, \rho_1)$, $(\irred{l_2}, \rho_2)$
we get a representation on 
$(2l_2+1)\times(2l_1+1)$--matrices by identification with the linear maps 
$\irred{l_1}\arrow \irred{l_2}$, on which $G=O(3)$ operates by 
\begin{eqnarray*}
    \lefteqn{\left(f:\irred{l_1}\arrow\irred{l_2}\right)\quad \mapsto} \\
    \rho_{Lin}(g) (f)&:&\irred{l_1}\arrow\irred{l_2}, 
          \vv v \mapsto  \rho_2(g) f(\rho_1(g^{-1}) v)
\end{eqnarray*}
When we choose bases of $\irred{l_1}$ and $\irred{l_2}$, the resulting operation on 
$(2l_2+1)\times(2l_1+1)$--matrices $A$ is given by
\[
    A \mapsto \rho_2(g) \cdot A \cdot \rho_1(g)^{-1} = \rho_2(g) \cdot A \cdot \rho_1(g)^H.
\]
We get an isomorphism of representations 
$Lin(\irred{l_1}, \irred{l_2}) \simeq \irred{l_2}\otimes {\irred{l_1}}^*$ by mapping
\[
    \vv v_2 \otimes \vv v_1^* \mapsto \left(
        \irred{l_1}\arrow\irred{l_2}, 
        \vv r \mapsto  \vv v_2 \cdot \langle \vv v_1^*, \vv r\rangle_{alg}
    \right)
\]

\section{Polynomial point set descriptors (PPSDs)}
\label{sec:PPSDs}
\noindent
In this section we only consider PPSDs in the original, narrow meaning, i.e. 
we only consider case i (polynomial functions), not case ii (general radial basis functions).

\subsection{PPSDs and their Normal Form}
\label{subsec:PPSDsNormalForm}
\noindent
Just in case the informal description of PPSDs leaves room for ambiguity, we give here a version of
increased formality:
\begin{definition}[PPSD]\ \\
Fix a finite set $\calC$ of colors. Consider the smallest set of expressions which contain
\begin{itemize}
    \item the constants from $\BR$,
    \item for any index symbol $\nu$ variables $x_\nu, y_\nu, z_\nu$ (referring
          to $x,y,z$--coordinates of points indexed with $\nu$),
    \item for any two expressions also their sum and product,
    \item for any expression $E$ and index symbol $\nu$ and color $\gamma \in \calC$
          also $\sum_{\nu\in S_\gamma} E$.
\end{itemize}
The scalar polynomial point set descriptors (PPSDs) for finite point sets with colors in $\calC$ are 
such expressions in which each occurring index symbol $\nu$ is bound unambiguously to some
preceding summation sign.\\
For a finite dimensional vector space $V$ the polynomial point set descriptors with values in $V$
are functions from finite point sets with colors in $\calC$ to $V$, such that when followed by 
any linear map $V\to \BR$, we get a scalar PPSD.
\end{definition}
\noindent
Here is an example for a (scalar) polynomial point set
descriptor (PPSD) for finite point sets $S_1, S_2\subset\BR^3$ of two colors:
\[
    \calD(S_1, S_2) := \sum_{(x_\kappa, y_\kappa, z_\kappa)\in S_1} \left(
     \left(\sum_{(x_\nu, y_\nu, z_\nu)\in S_2}x_\nu x_\kappa\right) \right.
     \cdot 
     \left.
     \left(\sum_{(x_\nu, y_\nu, z_\nu)\in S_2} (y_\nu - 2 \cdot z_\kappa)^2\right)
     \right)
\]
Most of the time, we just say PPSD for scalar PPSD. But sometimes (e.g. when embedding a
space of configurations into a $\BR^d$) we also use vector valued PPSDs. When the values are
in some $\BR^d$, these are just given by $d$ scalar PPSDs.\\
In the above example, by expanding all products and moving all summations to the left of the resulting 
monomials, we can write any such expressions as a linear combination of sums of monomials.
For example, the last such term in the expansion for the above descriptor would be
\[
   \calD(S_1, S_2) = ...
   + 4 \cdot \sum_{(x_\kappa, y_\kappa, z_\kappa)\in S_1}
           \sum_{(x_{\nu_1}, y_{\nu_1}, z_{\nu_1})\in S_2}
           \sum_{(x_{\nu_2}, y_{\nu_2}, z_{\nu_2})\in S_2}
           x_{\nu_1} x_\kappa \cdot z_\kappa^2
\]
(We used $\nu_1$ and $\nu_2$ to differentiate the two sums that both used index $\nu$
above.) In this example we have an ``empty'' sum: We have to sum over indices $\nu_2$, but
do not use any of the coordinates $x_{\nu_2}, y_{\nu_2}, z_{\nu_2}$ in this term
(or equivalently, their exponents are all 0). This means we can rewrite it as
\begin{equation}
   \calD(S_1, S_2) = ...
   + 4 \cdot |S_2| \cdot \sum_{(x_\kappa, y_\kappa, z_\kappa)\in S_1}
           \sum_{(x_{\nu_1}, y_{\nu_1}, z_{\nu_1})\in S_2}
           x_{\nu_1} x_\kappa \cdot z_\kappa^2
  \nonumber
\end{equation}
We will do this for all empty sums, and then we collect those sums over monomials that sum
over the same sets of points, i.e. over $m_1$ points in $S_1$, ..., $m_c$ points in $S_c$,
so the monomials depend on $m_1 + ... + m_n$ vectors.\\
This way we get for any PPSD a representation as sum of expressions
\begin{eqnarray}
  \lefteqn{
   \calE_{m_1,...,m_c; e_1,...,e_m}(S_1,...,S_c)} \nonumber \\
   &=&|S_1|^{e_1}\cdot ... \cdot |S_c|^{e_c} \cdot  \\
   & &
    \sum_{(\vv r_{1,1},...,\vv r_{1,m_1})\in S_1^{m_1}}...
    \sum_{(\vv r_{c,1},...,\vv r_{c,m_c})\in S_c^{m_c}}  
    P_{m_1,...,m_c; e_1,...,e_c}\big(\vv r_{1,1},...,\vv r_{c,m_c}\big)
    \nonumber \label{eq:NormalForm}
\end{eqnarray}
where $P_{m_1,...,m_c;e_1,...,e_c}\big(\vv r_{1,1},...,\vv r_{c,m_c}\big)$
is a polynomial which depends on the vector variables $\vv r_{1,1},...,\vv r_{c,m_c}$
and in each monomial each of these vectors appears, or equivalently: This polynomial
vanishes whenever we substitute one of its input vectors with zero.\\
(This sum may include terms for $m_1=...=m_c=0$, in which case we just have
\[
   \calE_{0,...,0; e_1,...,e_c}(S_1,...,S_c) = 
   |S_1|^{e_1}\cdot ... \cdot |S_c|^{e_c} \cdot P_{0,...,0; e_1,...,e_c}
\]
with constants $P_{0,...,0; e_1,...,e_c}$.)
\smallspace
Since the order of the $m_i$ vector variables $\vv r_{i,1},...,\vv r_{i,m_i}$
is arbitrary, we can replace the polynomials by their symmetrization, i.e. we can 
write this expression always with polynomials that are invariant under 
the product of symmetric groups
$\SymGr{m_1}\times\SymGr{m_2}\times...\times\SymGr{m_c}$ in the sense that
\begin{eqnarray*}
  \lefteqn{
   P_{m_1,...,m_c; e_1,...,e_c}\big(
            \vv r_{1,1},...,\vv r_{1,m_1},\ 
            ...,\ 
            \vv r_{c,1}, ...\vv r_{c,m_c}\big)} \\
   &=&  P_{m_1,...,m_c; e_1,...,e_c}\big(
            \vv r_{1,\sigma_1(1)},...,\vv r_{1,\sigma_1(m_1)},\ 
            ...,\ \vv r_{c,\sigma_c(1)},...\vv r_{c,\sigma_c(m_c)}\big)
\end{eqnarray*}
for any permutations $\sigma_1\in \SymGr{m_1}, ...\sigma_c\in\SymGr{m_c}$.
(If the polynomial had the property that it vanished whenever we substitute one of the vector
inputs with 0, the symmetrization will also have this property.)\\
\smallspace
We will call the resulting expression the ``Normal Form''.
(The definite article is appropriate: In the next section we will prove its uniqueness,
i.e. the contribution for each $(m_1,...,m_c; e_1,...,e_c)$ in a Normal Form can be reconstructed
already from the values of a PPSD. This will be used in the proof of Theorem \ref{theorem:InvariantTheory}.)
\smallspace
If we separate out all monomials in the polynomials $P_{m_1,...,m_c; e_1,...,e_c}$, we can express 
each PPSD as a sum of expressions 
\[
   \sum_{(x_{\nu_1}, y_{\nu_1}, z_{\nu_1})\in S_{\gamma_1}}\ ... \ 
   \sum_{(x_{\nu_k}, y_{\nu_k}, z_{\nu_k})\in S_{\gamma_k}}
         \qquad x_{\nu_1}^{a_1} \cdot ... \cdot x_{\nu_k}^{a_k} \cdot
           y_{\nu_1}^{b_1} \cdot ... \cdot y_{\nu_k}^{b_k} \cdot 
           z_{\nu_1}^{c_1} \cdot ... \cdot z_{\nu_k}^{c_k}
\]
with $\gamma_1,...,\gamma_k$ some sequence of colors, which can also be written as
\[
  \left( \sum_{(x_{\nu_1}, y_{\nu_1}, z_{\nu_1})\in S_{\gamma_1}}\ 
  x_{\nu_1}^{a_1} \cdot y_{\nu_1}^{b_1} \cdot z_{\nu_1}^{c_1} 
  \right)\ \cdot
  \ ... \ \cdot 
  \left( \sum_{(x_{\nu_k}, y_{\nu_k}, z_{\nu_k})\in S_{\gamma_k}}\ 
     x_{\nu_k}^{a_k} \cdot y_{\nu_k}^{b_k} \cdot z_{\nu_k}^{c_k} 
  \right)
\]
which shows that we can evaluate each PPSD as a polynomial in fundamental features
(i.e. polynomial features with only one summation).

\subsection{PPSDs on multisets, uniqueness of Normal Form}
\label{subsec:PPSDs_on_multisets}
\noindent
We can identify sets of points $S\subseteq \BR^d$ with functions $\mu:\BR^d \arrow \{0,1\}$
given by
\[
   \mu(x) := \begin{cases}
     1 & \hbox{if}\ x\in S\\
     0 & \hbox{else}
   \end{cases}
\]
By a ``multiset'' we then mean functions $\mu:\BR^d \arrow \{0,1,2,...\}$ -- we can interpret
$\mu(s)$ as the multiplicity with which $s$ occurs in $S$. We are here only interested in finite multisets,
i.e. there are only finitely many $x$ with $\mu(x)>0$.
\\
There is a natural extension of a PPSD $\calD(S_1,...,S_c)$ to multisets $S_1,...,S_c$:
If a point $\vv r$ in one of the sets has multiplicity $\mu(\vv r)>1$, replace it by 
$ m := \mu(\vv r)$ points at locations 
$\vv r + 1 \eps \cdot \vv d$, $\vv r + 2 \eps \cdot \vv d$,..., $\vv r + m \eps \cdot \vv d$
for some vector $\vv d$ and $\eps>0$. When we let $\eps\arrow 0$, this will let 
$\calD(S_1,...,S_c)$ converge to what we consider the value for multisets.\\
We can write down explicitly the value to which this converges:
If the PPSD $\calD$ is given in Normal Form, we have to interpret the $|S_i|$
in the sense of multisets, i.e. sum up all multiplicities, and we have to 
multiply the polynomials
$P_{m_1,...,m_c;e_1,...,e_c}\big(\vv r_{1,1},...,\vv r_{c,m_c}\big)$ with the additional 
factor $\mu(\vv r_{1,1}) \cdot ... \cdot \mu(\vv r_{c,m_c})$.
\\
For a set or multiset $S$ given by multiplicities $\mu$ and an integer $a\geq 0$ we write
$a \# S$ for the multiset with multiplicities $a\cdot \mu$.
\smallspace
Now we can consider for given PPSD $\calD$ and given finite sets $S_1,...,S_c\subset \BR^d$ 
the values
\[    f(a_1,...,a_c; b_1,...,b_c) := \calD(S'_1,...,S'_c) \]
with
\[
    S'_i := \big((a_i-1)\cdot b_i \cdot |S_i|\big) \#\{0\}\, \cup \big(b_i\# S_i\big)
\]
These new multisets have $|S'_i| = a_i\cdot b_i \cdot |S_i|$ elements and the original
elements in $|S_i|$ are repeated $b_i$ times. 
For each $(m_1,...,m_c; e_1,...,e_c)$ occurring in the Normal Form, this gives the 
extra factor $(a_i b_i)^{e_i}$ from the change in $|S_i|$ and the extra factor 
$b_i^{m_i}$ from the larger number of terms in the summation (the we only have to consider
the nonzero elements of $S'_i$, since the polynomials are constructed such that they are
0 if one of the entries is 0).
So as a function of $(a_1,...,a_c; b_1,...,b_c)$ this is a polynomial, which is determined
by $\calD$ and the sets $S_1,...,S_c$, and since a polynomial determines its coefficients,
this means that also the decomposition of $\calD(S_1,...,S_c)$ into individual contributions
of each order $(m_1,...,m_c)$ is determined by $\calD$.
\smallspace
This proves 
\begin{lemma}
Any PPSD has a \emph{unique} Normal Form.
\label{lemma:NFDecomposition_unique}
\end{lemma}

\subsection{Comparison of PPSDs with functions of configurations of fixed size}
\label{subsec:DistinctPoints}
\noindent
In the following we will go back to only evaluating the PPSDs at ordinary sets
(not multisets).
\smallspace
So far we have only looked at configuration features (PPSDs) that are defined for
configurations of any size. One may ask how our notion of PPSDs differs from just giving
different polynomials for point sets for different fixed numbers of points. 
We show the following:
\begin{proposition}\label{prop:NFPolynomial_unique}\ 
\begin{enumerate}
    \item For any given numbers of points $|S_1|,...,|S_c|$ of each color,
      any polynomial function on configurations of $|S_1|,...,|S_c|$ points of color $1,...,c$
      can be obtained by restricting a PPSD to these configurations. (We assume here that the 
      given polynomial function is indeed a function on sets, i.e. is invariant with respect to
      permuting arguments corresponding to the same color.)
    \item For any finite set of data $|S_1|,...,|S_c|$ as in 1, there is a PPSD that gives these functions
      on the corresponding configurations.
    \item Any PPSD is given by the function on point sets of a sufficiently large fixed size 
      (this sufficiently large size depends on the PPSD).
\end{enumerate}
\end{proposition}
\noindent
To rephrase this proposition: We can give separate functions on some types of configurations, and
always can complete it to a general PPSD on all configurations. However, this can only be done 
for a finite set of configurations -- from some size on (which we can make arbitrarily large),
the function on larger configurations must follow a general pattern, which determines 
the values on all large configurations by the values on the configurations at some particular size.
\smallspace
When we try to prove the first statement, our first attempt may be to just write down the given polynomial
in $k:=|S_1|+...+|S_c|$ variables and add some summation over $k$--tuples of points. However, this 
does not lead directly to the first statement: In the sum will appear terms where the same point is 
used several times in different argument positions of the polynomial, and some do not appear at all.
So before we prove this proposition, we first define a new variant of summation.
\smallspace
We do this first in the monochrome ($|\calC|=1$) case.\\
Given a PPSD and its components $\calE_{m;0}(S)$, 
define $\calE_{m;0}^*(S)$ for $|S|\leq m$ to be the sum over $k$--tuples in which each element 
of $S$ appears at least once:
\[
  \calE_{m;0}^*(S) := 
       \sum_{\substack{(\vv r_1,...,\vv r_m)\in S^m\\
             \{\vv r_1,...,\vv r_m \} = S }}  P_m(\vv r_1,...,\vv r_m)  
\]
According to this definition, $\calE_{m;0}^*(S)$ will be 0 for $|S|>m$, but this does not matter to us,
we will only need the cases with $|S| \leq m$.\\
The definition in the case of several colors is virtually the same --- in the definition of
$\calE_{m_1,...,m_c;0,...,0}^*(S_1,...,S_c)$ we require that all $c$ tuples of 
points in $S_\colori$ satisfy the condition that all elements of $S_i$ occur at least once as argument.
\begin{lemma}\label{lemma:SumOfDistinctPoints}
We can express all $\calE^*(S_1,...,S_c)$ as linear combinations of 
$\calE(T_1,...,T_c)$ for $T_i\subsetneq S_i$,and vice versa.
\end{lemma}
\begin{proof}
In the monochrome case, we see from the definitions that
$\calE_{m;0}^*(\{\vv r\}) = \calE_{m;0}(\{\vv r\})$ and for $|S|>1$:
\[
  \calE_{m;0}^*(S) = 
       \calE_{m;0}(S)\  - \ \sum_{T \subsetneq S} \ \calE_{m;0}^*(T)
       \qquad
       \hbox{for}\ |S|\leq m\\
\]
For several colors, in the recursion formula for
$\calE^*$ we subtract
all $\calE^*(T_1,...,T_c)$ in which \emph{any} of the $T_\colori$ is smaller.\\
For the reverse direction, we read the above formula just the other way round, i.e.
for the monochrome case:
\[
  \calE_{m;0}(S) = 
       \calE_{m;0}^*(S)\  + \ \sum_{T \subsetneq S} \ \calE_{m;0}^*(T)
       \qquad
       \hbox{for}\ |S|\leq m\\
\]
\end{proof}
\noindent
Now we can start with the proof of the proposition.
\begin{proof}[Proposition \ref{prop:NFPolynomial_unique}, 1.]\ \\
    In the monochrome case the polynomial
    \begin{eqnarray*}
       P_m(\vv r_1,...,\vv r_m) 
         &=& \frac{1}{m!} \sum_{\substack{(\vv r_1,...,\vv r_m)\in S^m\\
                              |\{\vv r_1,...,\vv r_m \}| = m }}  P_m(\vv r_1,...,\vv r_m)
    \end{eqnarray*}
    has to match $\frac{1}{m!} \calE_{m;0}^*(\{\vv r_1, ..., \vv r_m \})$.
    According to Lemma \ref{lemma:SumOfDistinctPoints} we can translate this into giving a
    polynomial for $\calE_{m;0}(\{\vv r_1, ..., \vv r_m \})$, which gives a PPSD.
    Similarly, for the case of several colors, in the last formula we have to divide by 
    $m_1!\cdot ... \cdot m_c!$.
\end{proof}
\begin{proof}[Proposition \ref{prop:NFPolynomial_unique}, 2.]
  For any finite list of different $c$--tupels of integers $(k_{i,1},...,k_{i,c})$
  we can write down polynomials that are 0 for all of the 
  (finitely many) given $c$--tupels except for the one with index $i$, e.g.
  \[
     \prod_{j\neq i} \Big( (k_1 - k_{j,1})^2 + ... + (k_c - k_{j,c})^2 \Big)
  \]
  Let $Q_i(k_1,...,k_c)$ be this polynomial, divided by its value at $(k_{i,1},...,k_{i,c})$,
  so at the given $c$--tuples it has only values 0 with one exception at tuple $i$, where it is 1.\\
  We get the numbers $|S_i|$ as the value of the PPSD $\sum_{\vv r\in S_i} 1$. So
  the expressions $Q_i(|S_1|,...,|S_c|)$ are also PPSDs.\\
  Now we can construct for all $c$--tuples $(|S_1|,...,|S_c|)$ the PPSDs from part 1,
  multiply them with with the corresponding $Q_i$, and sum up all these products.
  This gives a PPSD with the desired functions on each given $c$--tuple $(|S_1|,...,|S_c|)$.
\end{proof}
\begin{proof}[Proposition \ref{prop:NFPolynomial_unique}, 3.]
    The third part of Proposition \ref{prop:NFPolynomial_unique} is made precise in the following way:
    For any $n_1\geq m_1, ..., n_c\geq m_c$ the symmetric polynomial 
    $P_{m_1,...,m_c; e_1,...,e_c}$ can be reconstructed using 
    only the values of $\calE_{m_1,...,m_c; e_1,...,e_c}(S_1,...,S_c)$ for 
    $|S_1|=n_1,...,|S_c|=n_c$.
    To prove this, first note that the $e_1,...,e_c$ already
    give all information about the factor $|S_1|^{e_1}\cdot...\cdot |S_c|^{e_c}$, so 
    without limitation of generality we can restrict the proof to the case $e_1=...=e_c=0$.
    Furthermore, the way our polynomials $P_{m_1,...,m_c; 0,...,0}$ were constructed, the
    value of $\calE_{m_1,...,m_c;0,...,0}(S_1,...,S_c)$ is not changed if we add additional
    points $P=(0,0,...,0)$ at the origin to the multisets $S_1,...,S_c$. 
    This means the knowledge of all values of $\calE_{m_1,...,m_c;0,...,0}(S_1,...,S_c)$ for
    $|S_1|=n_1,...,|S_c|=n_c$ includes
    also the knowledge of these values for $|S_1|\leq m_1,...,|S_c|\leq m_c$, and we will
    use these values in the following.
    \smallspace
    We now apply Lemma \ref{lemma:SumOfDistinctPoints}, and see that in the monochrome case
    all values
    \begin{eqnarray*}
       P_m(\vv r_1,...,\vv r_m) 
         &=& \frac{1}{m!} \sum_{\substack{(\vv r_1,...,\vv r_m)\in S^m\\
                              |\{\vv r_1,...,\vv r_m \}| = m }}  P_m(\vv r_1,...,\vv r_m) \\
         &=& \frac{1}{m!} \calE_{m;0}^*(\{\vv r_1, ..., \vv r_m \})
    \end{eqnarray*}
    for $m$ distinct points $\vv r_1, ..., \vv r_m$ are determined if we know the values of
    $\calE_{m,0}(S)$ for $|S|\leq m$.\\
    Similarly, for the case of several colors, in the last formula we have to divide by 
    $m_1!\cdot ... \cdot m_c!$.
    With these changes the above proof then also shows part 3 of 
    Proposition~\ref{prop:NFPolynomial_unique} in the general case.
\end{proof}
\section{Proof of Theorem \ref{theorem:TopologicalCompleteness}}
\label{sec:ProofTopologicalCompleteness}

\subsection{Fundamental features describe sets of points}
\label{sec:MomentsUnique}
\noindent
Here we prove the first part of theorem \ref{theorem:TopologicalCompleteness}, i.e. that the
fundamental features uniquely describe point sets -- this is not yet using the group $G$.
(We can see this as giving coordinates to the infinite dimensional variety
of finite point sets --- polynomial point set descriptors are then
just the polynomial functions in these coordinates.)
\smallspace
We have already seen in \ref{subsec:VariantsDistinguishingCapabilities} that it is enough
to prove this in the case 1, i.e. $X=\BS^2$ (and polynomial functions). Furthermore, for this
first part of Theorem \ref{theorem:TopologicalCompleteness} the subsets of points of different
color and their fundamental features are completely independent, so we can treat each color
separately, and just assume we only have one color. 
\smallspace
We will use the following easy lemma:
\begin{lemma} \label{lemma:InjectiveLinFun}
Let $S$ be a finite set of points in $\BR^d$. Then there is a linear function $L:\BR^d\arrow \BR$
such that $L$ is injective on $S$.
\end{lemma}
\begin{proof}
We prove this by induction on $d$. For $d=1$ there is nothing to prove, so suppose now that 
$d>1$ and we know the statement for $d-1$.\\
If we have finitely many different points in $S \subset \BR^d$, there are only finitely many directions
for projections to a $d-1$ -- dimensional space under which two points are mapped 
to the same point, avoiding them gives a projection $f:\BR^d \arrow \BR^{d-1}$ that is injective
on $S$. Using the induction hypothesis then gives a linear map $L:\BR^{d-1}\arrow \BR$
that is injective on $f(S)$, so together $s \mapsto L(f(s))$ is a linear function $\BR^d\arrow \BR$
which is injective on $S$.
\end{proof}
\noindent
Now we can prove the first part of Theorem \ref{theorem:TopologicalCompleteness} in the polynomial
cases 1, 2i, 3i; and as seen in \ref{subsec:VariantsDistinguishingCapabilities} this implies it also
for the other cases 2ii, 3ii:
\begin{lemma} \label{lemma:FundFeatDetConfig}
Let $S\subset \BR^3$ be a set of at most $n$ elements.
Then the fundamental features $\sum_{\vv r\in S} P(\vv r)$ with polynomials $P:\BR^3\to\BR$
of degree $\leq 2n$ uniquely characterize the set $S$.\\
Equivalently, the first $2n+1$ moments
\[
    \sum_{\vv r\in S} \vv r^{\otimes k}
    \qquad \hbox{for}\ k=0,1,...,2n
\]
determine the set $S$.
\end{lemma}
\begin{proof}
The two versions are equivalent because the entry with index $(e_1,...,e_k)$ with $1 \leq e_i \leq 3$
in the moment tensor of order $k$ is
just the fundamental feature $\sum_{(x,y,z)\in S}\, x^a\,y^b \,z^c$ with $a,b,c$ the number of indices
$1,2,3$ in $(e_1,...,e_k)$.\\
Assume we have two different sets $S \neq S'$ of at most $n$ elements each. 
According to the last lemma, we can find a linear map $L:\BR^d\arrow\BR$ that is injective on
$S\cup S'$.\\
Assume now that we have a point $\vv s_0$ that is only in one of the two sets
$S$ and $S'$.
Given the $|S\cup S'| \leq 2n$ values of $L(\vv s)$ for $\vv s\in S\cup S'$, there is  
a polynomial $p(t)$ of degree $<2n$ in $L$ that is $1$ on $L(\vv s_0)$ and 0 on $L(\vv s)$ for
all other points $\vv s \neq \vv s_0$ of $S$. 
Thus we get a polynomial $P(\vv r) := p(L(\vv r))$ of degree 
$\leq 2n$ with 
\begin{equation*}
    \sum_{\vv r\in S}  P(\vv r) \neq \sum_{\vv r'\in S'} P(\vv r')
\end{equation*}
\end{proof}

\subsection{Completeness of invariant polynomial descriptors}
\label{sec:CompletenessInvPol}
\noindent
We will now use this and the following proposition to prove the second part of theorem \ref{theorem:TopologicalCompleteness}:
\begin{proposition} \label{prop:G_invariant_functions}
Let $V$ be an Euclidean vector space with $G$--representation $\rho: G\arrow O(V)$, 
and $\vv v, \vv w$ be two points of $V$ that are not
related as $\vv w = \rho(g) \vv v$ for any $g\in G$. Then there is a $G$--invariant polynomial
$P:V \arrow \BR$ such that $P(\vv v) \neq P(\vv w)$.
\end{proposition}
\begin{proof}
We choose an orthonormal basis of $V$ to identify it with some $\BR^d$.
Since $G$ is compact, the orbits $\rho(G) \vv v$ and $\rho(G) \vv w$ are compact, so they are also closed.
By assumption, they are disjoint. Since $V$ is a metric space, it is normal and Urysohn's lemma 
implies that there is a continuous function that is 0 on the one compact set and 1 on the other. So let 
$f: V\arrow \BR$ be such a continuous function with $f(\rho(G) \vv v)=\{0\}$ and $f(\rho(G) \vv w)=\{1\}$.
\smallspace
Since the orbits $\rho(G) \vv v$ and $\rho(G) \vv w$ are compact, they are also bounded, so let $R>0$ be such
that both are contained in the closed hypercube $H := [-R,R]^d$. 
By the Stone--Weierstrass approximation theorem, there is then for any $\epsilon > 0$
also a polynomial function $p:V\arrow \BR$ such that $|f(\vv x) - p(\vv x) < \epsilon|$ for all
$\vv x \in H$. Let $p$ be such a polynomial function for $\eps := 0.1$. Then 
\[
   p(\rho(G) \vv v) \subseteq [-0.1, 0.1] \qquad \hbox{and} \qquad p(\rho(G) \vv w) \subseteq [0.9, 1.1] 
\]
Averaging over $G$ gives now a function
\begin{equation}
   P(\vv x) := \int_{g\in G} p(\rho(g) \vv x) dg
   \label{eq:AveragePolynomial}
\end{equation}
for the Haar measure $dg$ on $G$ with total mass $\int_G 1 = 1$.
Since for each $g\in G$ the function $\vv x \mapsto p(\rho(g) \vv x)$ is a polynomial,
the integral \eqref{eq:AveragePolynomial} can actually be expressed as a polynomial 
in $\vv x$ with 
coefficients that are integrals over $G$, so the function $P:V\arrow \BR$ is actually a polynomial.
By construction, $P$ is $G$--invariant, and still satisfies the inequalities 
\[
   P(\rho(G)\vv v) \subseteq [-0.1, 0.1] \qquad \hbox{and} \qquad p(\rho(G) \vv w) \subseteq [0.9, 1.1]
\]
so in particular we have $P(\vv v) \neq P(\vv w)$.
\end{proof}
\noindent
We will now apply this proposition not to our original $3$--dimensional vector space $\BR^3$, but 
to a vector space which is given by fundamental features. 
To prove the second part of Theorem \ref{theorem:TopologicalCompleteness}, assume we have two
configurations $\{(\vv r_i, \gamma_i)\}$ and $\{(\vv r'_i, \gamma'_i)\}$ of points in $X\subseteq \BR^3$ 
with colors in $\calC$, that are not equivalent under $G$, i.e. for every $g\in G$ 
\[
 \{(\vv r_i, \gamma_i)\} \neq \{(g \vv r'_i, \gamma'_i)\}.
\]
We can first try to distinguish them by counting the number of points of each color, using
the fundamental features $\sum_{\vv r\in S_\gamma} 1 = |S_\gamma|$ for $\gamma\in\calC$.
If we do not see a difference, we already know that the numbers
$|S_\gamma|$ are the same for both configurations.\\
Now let $\calX$ be the set of all configurations consisting of $k_1$ points of color 1, ...
$k_c$ points of color $c$, and let $p,q\in \calX$ be the two configurations that are not
equivalent under $G$. We can see 
$\calX = (\BR^{3k_1}/\Sigma_{k_1}) \times ... \times  (\BR^{3k_c}/\Sigma_{k_c})$
(with $\Sigma_k$ as the symmetric group permuting $k$ points) as a topological space,
and each feature is a continuous map $\calX\to\BR$.
\smallspace
From the previous section we know that for every $g\in G$ there must be a fundamental feature 
$f_g: \calX\to\BR$ with $f_g(p) \neq f_g(gq)$. Since the features are continuous functions, there must
also be a open neighborhood $U_g$ of $g$ in $G$ for which $f_g(p) \neq f_g(g' q)$ for all $g'\in U_g$.
Since $G$ is compact, finitely many of these open sets $U_g$ must 
already cover $G$, i.e. there are finitely many fundamental features that, taken together, can distinguish
$p$ from $g'q$ for all $g'\in G$. 
\smallspace
Let $n$ be the maximal degree of a polynomial occurring in some of these finitely many fundamental features,
and in case ii let also $J\subseteq I$ be the finite subset of all $j$ such that the radial basis function
$g_j$ occurs in some of these finitely many fundamental features. Let $\calP^{(n)}$ be the 
finite dimensional vector space of polynomials in $x,y,z$ of degree $\leq n$. In case i each $(\gamma,P)\in \calC\times\calP^{(n)}$ determines a fundamental feature $f_{(\gamma,P)}$ as
\[
    \sum_{\vv r\in S_\gamma} P(\vv r)
\]
and in case ii each $(\gamma, P, j)\in \calC\times\calP^{(n)} \times J$ determines
a fundamental feature $f_{(\gamma, P, j)}$ as
\[
  \sum_{\vv r\in S_\gamma} P\Big(\frac{\vv r}{|\vv r|}\Big) g_j(|\vv r|). 
\]
The linear combinations of these fundamental features form a finite dimensional vector space $\calY$,
it can be described in case i as
\[
    \calY := \BR^\calC \otimes \calP^{(n)},
\]
and in case ii as
\[
    \calY := \BR^\calC \otimes \calP^{(n)} \otimes \BR^J.
\]
(If some $P_m(x,y,z)$ are a basis of $\calP^{(n)}$, then the fundamental features
described by $\gamma \otimes P_m \otimes j$ with $\gamma\in \calC$ and $j\in J$ would be a basis
of $\calY$ in case ii.)\\
We denote the evaluation of a fundamental feature $y\in \calY$ on a configuration $x\in \calX$ 
as $\langle x, y \rangle$, it is a linear function in $y\in\calY$. As a function on $\calX$
this is a vector valued PPSD with values in the finite dimensional vector space $\calY^*$:
When followed by a basis vector $\gamma \otimes P_m \otimes j\in \calY$ (as a map $\calY^*\to\BR$),
we get the fundamental feature $f_{\gamma, P_m, j}$ on $\calX$.\\
Furthermore, we have an action of $G$ on $\calY$ which acts trivially on $\calC$ and $J$, but maps
$P\mapsto gP$ with $gP(\vv r) := P(g^{-1} \vv r)$. This action on $\calY$ satisfies
\[
   \langle gx, gy \rangle = \langle x, y \rangle.
\]
We can use this to map a $x\in \calX$ to the linear function on $\calY$
\[
   \iota: \calX \to \calY^*,  x \mapsto \big(y\mapsto \langle x, y\rangle\big).
\]
On $\calY^*$ the group $G$ operates as $gL := \big(y\mapsto L(g^{-1}y)\big)$, this means
the map $\iota: \calX \to \calY^*$ is covariant:
\[
   \iota(gx) = (y \mapsto \langle gx, y \rangle)
             = (y \mapsto \langle x, g^{-1}y \rangle)
             = g\ \iota(x)
\]
Since $G$ is compact, we can choose a $G$--invariant scalar product on $\calY^*$, with 
respect to this scalar product we have an orthogonal representation $\rho: G\to O(\calY^*)$.
Now we can apply Proposition \ref{prop:G_invariant_functions} to find a $G$--invariant polynomial
in the fundamental features, i.e. a $G$--invariant PPSD, which distinguishes the two
given configurations.
\subsection{Describing configurations by finitely many features}
\label{subsec:ConfFinitelyManyFeatures}
\noindent
For the last part of Theorem \ref{theorem:TopologicalCompleteness} we will need a stronger
version of this embedding of $\calX$ to $\calY^*$ which not only maps the orbits of $p$ and $q$
to disjoint sets, but is injective on compact subsets of $\calX$.
This is no problem in the case i: Lemma \ref{lemma:FundFeatDetConfig} shows that in that case
$\iota:\calX\to\calY^*$ is already injective. So we only consider the case ii here.
\smallspace
For some fixed $R>0$ we denote by $B_R$ the closed ball of radius $R$ around the
origin in $\BR^3$.
We will now give a strengthened version of the argument in \ref{subsec:VariantsDistinguishingCapabilities}.
\begin{lemma}\label{lemma:InjMapToVS}\phantom{.}\\
{[}Assuming the case ii, with Condition II$_w(2k)${]}\\
Let $k=k_1+...+k_c$, $R>0$, and $J$ be the finite subset of $I$ from Condition $II'_w(2k)$. Let
\[
   \calX :=\big(B_R^{k_1}/\Sigma_{k_1}\big) \times ... \times  \big(B_R^{k_c}/\Sigma_{k_c}\big).
\]
and with $\calP^{(2k)}$ the polynomials of degree $\leq 2k$ set
\[
    \calY := \BR^\calC \otimes \calP^{(2k)} \otimes \BR^J.
\]
Then the map
\[
   \iota: \calX \to \calY^*,  x \mapsto \big(y\mapsto \langle x, y\rangle\big).
\]
is injective, covariant, and given by fundamental features.
\end{lemma}
\begin{proof}
We have already seen this is covariant and given by fundamental features. For the 
injectivity, assume we are given two configurations of $k$ points. 
Then at most $2k$ radii occur in them, so let's say all radii appear in the list
$ 0 \leq t_1 < t_2 < ... < t_{2k} $. Then we can find linear combinations $f_a$ for $a=1,2,..,2k$
of the constant 1 and radial basis functions $g_j$ with $j\in J$ that satisfy
$f_a(t_b) = \delta_{a,b}$.\\
Then the fundamental features $\sum_{r\in S_\gamma} f_a(|\vv r|)$ count the number of 
points with radius $t_a$, so we can assume they are the same in both configurations.
According to Lemma \ref{lemma:FundFeatDetConfig}, the fundamental features on
$\BS^2$ given by polynomials of degree $\leq 2k$ are enough to determine the sets 
\[
  \left\{\frac{\vv r}{t_a}\ \Big|\ \vv r \in S_\gamma, |\vv r| = t_a\right\}
\]
for any $t_a>0$. Therefore, using
\[
   \sum_{\vv r\in S_\gamma, |\vv r|=t_a} P\left(\frac{\vv r}{|\vv r|}\right)
   = \sum_{\vv r\in S_\gamma}  f_a(|\vv r|) \cdot P\left(\frac{\vv r}{|\vv r|}\right),
\]
for $t_a>0$ we can determine also the set of original points of radius $t_a>0$ by these fundamental features, i.e. the map $\iota$ cannot map two different configurations to the same image in $\calY^*$.
\end{proof}

\subsection{Completeness of {\em covariant} PPSDs}
\label{sec:CovariantComplete}
\noindent
For $G$--invariant functions $f$ on point configurations we can
evaluate $f$ on equivalence classes of $G$--equivalent configurations,
and we formulated the completeness in part 2 of 
Theorem \ref{theorem:TopologicalCompleteness} as ``there are enough
invariant polynomial functions to distinguish any two non-equivalent
configurations''.\\
We cannot formulate such a completeness criterion directly 
for covariant functions with values in $V$ for a non-trivial
representation $(\rho, V)$: We cannot evaluate such functions
on equivalence classes of $G$--equivalent configurations, since
they get different values for different equivalent configurations.
\smallspace
However, in the above proof we embedded a space of configurations $\calX$ by an equivariant
map into a vector space $\calY^*$ given by fundamental features, and used the 
Stone--Weierstrass approximation theorem to prove the separation of orbits by approximating
a continuous function that was 1 on one orbit and 0 on the other orbit. This can be
used more generally as an approximation property that can be formulated also for
arbitrary representations $(\rho, V)$.
\smallspace
The approximation property in Theorem \ref{theorem:TopologicalCompleteness} is of the form
``uniform approximation on compact subsets of configurations'', so we need to specify what
the topological space of all configurations is:\\
Given the numbers of points $k_1 = |S_1|,...,k_c=|S_c|$ of each color, we have the topological
space of all configurations with these $(k_1,...,k_c)$ given by 
\[
   \calX_{(k_1,...,k_c)} = (\BR^{3k_1}/\Sigma_{k_1}) \times ... \times  (\BR^{3k_c}/\Sigma_{k_c})
\]
(Strictly speaking, this does not incorporate that the $k_i$ points of color $i$ should be 
different. However, as we have seen before, we can extend our PPSDs in a canonical way 
from taking sets as input to taking multisets as input, so this does not matter.)
We consider the space of all configurations to be the disjoint union of these spaces, so
the $\calX_{(k_1,...,k_c)}$ are the connected components of the space of all configurations.
This also means that a compact subset can only contain points from finitely many of these
$\calX_{(k_1,...,k_c)}$. According to part 2 of Proposition \ref{prop:NFPolynomial_unique}
we can give arbitrary polynomial functions on each of finitely many connected components,
and then find a PPSD that gives these functions as restrictions to their component.
(The context of Proposition \ref{prop:NFPolynomial_unique} was the polynomial function class
[case i], but in this specific part we only used that the $|S_i|$ are PPSDs, so this is equally 
valid in case ii.)
So if we are interested in approximation on a compact subset of configurations, we can
assume without loss of generality that this subset is contained in one $\calX_{(k_1,...,k_c)}$.
Furthermore, since compact subsets of Euclidean vector spaces are bounded, there must be a maximal
$|\vv r|$ of all points in these configurations, so with $B_R$ the closed ball of radius $R$ around
the origin in $\BR^3$, we can assume that the compact subset of configurations lies in
\[
   \calX :=\big(B_R^{k_1}/\Sigma_{k_1}\big) \times ... \times  \big(B_R^{k_c}/\Sigma_{k_c}\big)
\]
for some $R>0$. Then the statement is that any continuous $G$--covariant function $f:\calX\arrow V$
can be uniformly approximated by $G$--covariant polynomial functions $P_i: \calX \arrow V$.
(To be precise, these continuous / polynomial functions are continuous / polynomial functions
on $B_R^{k_1+...+k_c}$ that are invariant under $\Sigma_{k_1} \times ... \times \Sigma_{k_c}$.)
This can then be proven for arbitrary representations $(\rho_V, V)$
in a similar way as before :\\
\begin{itemize}
    \item Use the results of the previous section \ref{subsec:ConfFinitelyManyFeatures} 
    (which assumed Condition II$_w(2k)$ in case ii) to map $\calX$
    injectively by covariant fundamental features into a finite dimensional vector space $\calY^*$.
    \item Using Stone--Weierstrass:\\
    Let $X':=\iota(X)$ be the image in the vector space $\calY^*$, this
    is a compact subset of $\calY^*$. Then any continuous function on $X$ can be seen as a continuous function on $X'$, and polynomial function on $X'\subseteq \calY^*$ are PPSDs. 
    So we can use the Stone--Weierstrass theorem to approximate uniformly any continuous functions on
    $X'$ by polynomials.
    \item Averaging over the compact group $G$:\\
    The group $G$ acts on $\calY^*$ by some representation $\rho$.\\
    As above, we can apply an averaging operator to project functions $X\arrow V$ 
    or their equivalent $X'\arrow V$ to equivariant
    functions for the representations $(\calY^*, \rho)$ and $(V, \rho_V)$:\\
    We map a function $f:X'\arrow V$ to the average
    \begin{equation}
       F(x) := \int_{h\in G} \rho_V^{-1}(h) f(\rho(h) x)  dh
       \label{eq:AverageCovariantFunction}
    \end{equation}
    Then 
    \begin{eqnarray*}
    F(\rho(g) x) &=&  \int_{h\in G} \rho_V^{-1}(h) f(\rho(hg) x)  dh\\
      &=& \int_{h\in G} \rho_V(g)\rho_V^{-1}(hg) f(\rho(hg) x)  dh\\
      &=& \rho_V(g)\int_{h\in G} \rho_V^{-1}(hg) f(\rho(hg) x)  dh\\
      &=& \rho_V(g)F(x)
    \end{eqnarray*}
    \item Averaging does not increase the supremum norm:\\
    For a compact group $G$ we have a scalar product / norm on $V$ that is invariant under
    $\rho_V(G)$. Since all norms give the same topology on $V$, we can use this norm to prove
    density in the supremum norm. Then the averaging operator \eqref{eq:AverageCovariantFunction} is an isometry for the supremum norm, and applying this
    to a sequence of polynomials from the Stone--Weierstrass theorem that converge to a covariant continuous function, this then gives a sequence of covariant polynomials that
    converge to the continuous function.
\end{itemize}
This concludes the proof of part 3 of Theorem \ref{theorem:TopologicalCompleteness}.

\section{Proof of Theorem \ref{theorem:FiniteFeatures}}
\label{sec:ProofFiniteFeatures}
\subsection{Context of Theorem \ref{theorem:FiniteFeatures}}
\label{sec:ContextFiniteFeatures}
\noindent
Theorem \ref{theorem:FiniteFeatures} bounds the number of features necessary to 
give a complete characterization of configurations of $k$ points up rotations in $G$.\\
There are different ways to define the number of features necessary for a ``complete'' 
description of objects that can be seen as points on a manifold (/orbifold/stratified manifold) $M$:
\begin{enumerate}
    \item The dimension of $M$.\\
      This gives the number of features necessary to distinguish configurations locally, i.e. in a 
      neighborhood of a particular chosen object. Often the dimension of quotients
      $M = N/G$ will be $dim(N)-dim(G)$ (e.g. if $G$ is a Lie group that acts smoothly, freely
      and properly on a manifold $N$). In fact, in our case we have $dim(M)=3k-3$ for $k>2$.
    \item The embedding dimension of $M$.\\
      This is the number of features necessary to characterize points in $M$ globally; it is 
      usually larger than $dim(M)$: E.g. around any given point on the sphere
      $M:=\BS^2\subset\BR^3$, two of the three coordinates are sufficient,
      but each particular pair of coordinates will not be enough to identify points globally, and
      the Borsuk--Ulam theorem implies that even when allowing two arbitrary continuous features there 
      will always be two different (in fact, even antipodal) points on the sphere on which the
      two functions give the same pair of numbers.\\
    \item For algebraic varieties $M$, the minimal number of generators of $\BR[M]$.\\
      $\BR[M]$ is the algebra of polynomial functions on $M$, and a set of such functions
      is said to be generators of the algebra $\BR[M]$ if every function in $\BR[M]$ can be
      expressed as a polynomial in the generators.\\
      This is a quite strong condition on the basis set, and this number is usually again larger
      than the second number. 
      As a simple example, let $M$ be the curve in the plane $\BR^2$ given by $y^2=x^3$. Then 
      the $y$--coordinate alone already determines the point $(x,y)$, since there is a unique 
      cube root of $y^2$ in $\BR$, but the $x$--coordinate is not a polynomial in $y$, so just
      $\{y\}$ is not a set of generators. The
      algebra $\BR[x,y]/(y^2-x^3)$ of polynomials on $M$ can also be given as
      $\BR[t^2, t^3] \subset \BR[t]$ with $t:= y/x$, and we can see that no single polynomial 
      $P(t)\in \BR[t^2, t^3]$ generates the algebra $\BR[t^2, t^3]$: Up to an additive constant
      (which does not change the subalgebra generated by $P$) it has to have order 2 or higher;
      if it has order 2, $t^3$ cannot be written as a polynomial in $P$, and if it has order $>2$,
      $t^2$ cannot be written as a polynomial in $P$, so we need at least 2 generators for the 
      algebra of polynomials, but only one feature to distinguish points.
\end{enumerate}
We are here interested in the second number (embedding dimension), and will prove that it is 
$\leq 6k-5$.
This proof will also use the fact that the dimension is $3k-3$, and in the case of polynomial
functions, we will also show and use that the third number is finite.
\smallspace
To formulate Theorem \ref{theorem:FiniteFeatures} precisely, we need to distinguish the
different use cases (see Appendix \ref{sec:Variants}): We have $G=O(3)$ or $G=SO(3)$,
and the points are in
\begin{enumerate}
    \item Sphere: $X=\BS^2$
    \item Spherical shell: $X=\{\vv r \in \BR^3\,\Big|\, r_0 \leq |\vv r| \leq r_1\}$
    \item Full space: $X=\BR^3$
\end{enumerate}
and the functions are
\begin{enumerate}[label={\roman*.}]
    \item polynomials, or more generally
    \item linear combinations of products of polynomials on $\BS^2$ and a set of radial basis functions
          $g_i:\BR_{\geq 0}\arrow \BR$ for $i$ in some index set $I$.
\end{enumerate}
Theorem \ref{theorem:FiniteFeatures} is valid for cases 1, 2i, 2ii, 3i, but not 3ii.
For the case 2ii we also require that the $g_i$ are analytic.\\
The statement of Theorem \ref{theorem:FiniteFeatures} is:\\
We consider $G$--equivalence classes of colored point sets with at most $k_1,...,k_c$ points of colors $1,...,c$.
For $k:=k_1+...+k_c$, we can find $2k\cdot\dim(X)-5$ invariants (of the given function class) that already distinguish
all equivalence classes.
\smallspace
We can consider $k$--point configurations in $X$ as points in $X^k/\Sigma$ where $\Sigma\subseteq \Sigma_k$ is the 
subgroup of the symmetric group $\Sigma_k$ that only permutes points with the same color. Then we can identify 
$X^k/(\Sigma \times G)$ with the $G$--equivalence classes of configurations, if it was a manifold, its dimension 
would be $d:=\dim(X^k)-\dim(G)=k\cdot\dim(X)-3$, and Theorem \ref{theorem:FiniteFeatures} says we can characterize
points in $X/(\Sigma \times G)$ by $2d+1$ features.
\smallspace
Theorem \ref{theorem:FiniteFeatures} is reminiscent of the embedding theorem of Whitney (see e.g. \cite{guillemin2010differential}, Chapter 1.8):
Any smooth real $d$-dimensional manifold can be smoothly embedded in $\BR^{2d}$ for $d>0$.\\
(The slightly weaker statement that it can be embedded in $\BR^{2d+1}$ is easier to prove, we
will follow the general idea of such a proof.)\\
There are three main differences: 
\begin{itemize}
    \item The quotient $X^k/(\Sigma\times G)$ is not a manifold (but an orbifold, see below).
    \item Instead of arbitrary smooth features we only allow polynomial features (in case i, or other
          specific analytic features in case ii).
    \item We will prove a slightly stronger version which bounds the factor by which distances
          can shrink in this embedding.
\end{itemize}
\smallspace
We will prove Theorem \ref{theorem:FiniteFeatures} in the more detailed form:
\begin{enumerate}[label=\alph*.]
    \item There are finitely many $G$--invariants that distinguish all $G$--equivalence classes of colored
          point sets with at most $k_1,...,k_c$ points of colors $1,...,c$.
    \item Given a set of $N$ $G$--invariants that distinguish all $G$--equivalence classes of colored
          point sets with at most $k_1,...,k_c$ points in $X$ of colors $1,...,c$, and $k:=k_1+...+k_c$,
          we can find $2k\cdot\dim(X)-5$ linear combinations that already distinguish all equivalence classes.
    \item These $2k\cdot\dim(X)-5$ linear combinations can be found by random orthogonal projections:
          After $N-2k\cdot\dim(X)+5$ projections to 1-codimensional subspaces we arrive at $2k\cdot\dim(X)-5$
          linear combinations of the features, with probability 1 this will be successful in the sense
          that they also already distinguish all equivalence classes of configurations.
    \item For each such successful sequence of random projections, the distance of 
          the projected $2k\cdot\dim(X)-5$-dimensional feature vectors is lower bound by the 
          distance in the original $N$--dimensional space up to a factor, i.e. there is a constant
          $\eps>0$ such that for all configurations $\{r_i,\gamma_i\}$, $\{r'_i,\gamma'_i\}$ the 
          distances in the feature spaces satisfy
          \[
               d_{proj}\left(\{r_i,\gamma_i\}, \{r'_i,\gamma'_i\}\right) \geq \eps \cdot d_{orig}\left(\{r_i,\gamma_i\}, \{r'_i,\gamma'_i\}\right)
          \]
    \item We required in general that the radial basis functions allow approximating polynomials,
          but for this theorem for a fixed $k$ it is enough that for any $2k$ different radii
          $t_1,...,t_{2k}$ there are $2k$ radial basis functions $g_1,...,g_{2k}$ such that
          the ${2k}\times {2k}$--matrix $g_i(t_j)$ is nonsingular.
\end{enumerate}
\smallspace
With these additional specification of the projection procedure and the resulting inequality for the distances, this theorem reminds of Johnson--Lindenstrauss type lemmas.
However, in the original Johnson--Lindenstrauss lemma the point set is finite, and the embedding 
dimension increases with the number of points. For versions in which the point set is a manifold
(e.g. \cite{Clarkson2008},\cite{Baraniuk2009}), instead of the number of points, some geometric
properties of the manifold (volume, curvature) influence the embedding dimension.\\
In contrast, here the dimension only depends on the dimensionality of the input (stratified) manifold.
(Also, the Johnson--Lindenstrauss type lemmas try to achieve an almost isometric map, they succeed
with some high probability, but not with probability 1. Here we are satisfied with some lower
bound on the distances in the target space, and achieve success with probability 1.)
\smallspace
We now turn to the proof in the following subsections.\\
To simplify a bit, we will first fix the numbers $|S_1|,...,|S_c|$ of points of colors $1,2,...,c$.
In section \ref{subsec:LessThanKPointsFiniteFeatures} we will then
explain why we can also encode these numbers and characterize all configurations with \emph{at most}
$k_i$ points of color $i$ without using more features.
\smallspace
While PPSDs are defined for configurations of any size, in this theorem we are now looking at
functions on configurations of a fixed number of points.
For each fixed combination $|S_1|,...,|S_c|$ of points of colors $1,2,...,c$ Proposition \ref{prop:NFPolynomial_unique}
says that PPSDs on these point sets are the same functions as the polynomials in points 
$\vv x_1,...,\vv x_k\in \BR^3$ that are invariant under permutations of points of of the same color,
so in the proof we will rather use these concrete polynomials in $3k$ variables.

\subsection{Finiteness}
\label{sec:FinitenessOnly}
\noindent
We will first look at part a, the finiteness.
The previous topological completeness theorem \ref{theorem:TopologicalCompleteness}
showed that \emph{all} invariant PPSDs are enough to distinguish any two configurations
that are not equivalent under $G$, but this
is an infinite set, and we need a finite subset to start the projection procedure.
\smallspace
{\bf Proof in case i: Polynomials}\\
Obviously, a finite set of invariants that distinguishes all non-equivalent
configurations in $X=\BR^3$ also gives enough invariants for other cases
$X\subset \BR^3$, so we will only consider $X=\BR^3$.
In this case we can prove that the algebra of invariant polynomials
$\BR[x,y,z]^{\Sigma\times G}$ is finitely generated as an algebra over $\BR$:
The group is $\Sigma\times G$ is compact, hence reductive, so this is a special case of
Hilbert's finiteness theorem on invariants. (See e.g. \cite{mukai2003introduction}, chapter 
4.3. In this reference, the ground field is assumed to be algebraically closed, but this
does not make a difference: Any invariant $f\in \BC[x,y,z]$ can be written as 
$f=g + i\cdot h$ with $g,h\in \BR[x,y,z]$, and the operation of $\Sigma\times G$
on $f$ is just given by its operation on $g$ and $h$.)
\\
Since the values of these finitely many generators of the algebra already determine the
values of all invariant polynomials, these are enough to distinguish any two non--equivalent 
configurations.
\smallspace
{\bf Counterexample in case 3ii}\\
In the case ii we have more freedom to choose the radial basis functions, hence
the algebraic proof of i does not work any more. In fact, we have to make up for
the relaxed restrictions on the function class by adding new assumptions, as we 
will show in this counterexample:
\smallspace
Assume $X=\BR^3$ and as in the counterexample in section \ref{sec:FiniteConditions}
the basis functions are (the constant 1 and) $\sin\big(\frac{2a-1}{2^b}\cdot r\big)$, and
$\cos\big(\frac{2a-1}{2^b}\cdot r\big)-1$ for natural numbers $a,b=1,2,...$.
We have seen in section \ref{sec:FiniteConditions} that although these satisfy 
Conditions I and II, no finite subset of these functions will be able to uniquely 
characterize configurations of only 1 point on the $x$--axis),
i.e. Theorem 3 does not hold in the case 3ii.
\smallspace
We also showed in section \ref{sec:FiniteConditions} that Condition II (analyticity)
was needed to derive Condition II$_w(2k)$, smoothness would not suffice.
So in the following we assume Conditions I and II (or for this finiteness part only,
assuming Condition II$_w(2k)$ only also is enough).

\smallspace
{\bf Proof in case 2ii: Compact $X$, Condition II$_w(2k)$}\\
Let $X:= \{\vv r\in \BR^3\,\big|\, r_0 \leq |r| \leq r_1\}$, and we want
to prove this for configurations of $k:=k_1+...+k_c$ points.
Assume Conditions I and II$_w{(2k)}$ (or I and II, as we showed in 
Appendix \ref{sec:FiniteConditions}, Proposition \ref{prop:AnalyticImpliesIIw}
that this implies Condition II$_w(n)$ for all $n$).
We use the same argument as in section \ref{subsec:ConfFinitelyManyFeatures}, this 
time for equivalence classes of configurations under $G$.
Let $J$ be the finite subset of $I$ guaranteed by condition II$_w(2k)$ for $R=r_1$.
Since we are given two configurations of $k$ points, at most $2k$ radii occur in 
them, so let's say all radii appear in the list
$ t_1 < t_2 < ... < t_{2k} $. Then we can find linear combinations $f_a$ of 
radial basis functions $g_j$ with $j\in J$ that satisfies $f_a(t_b) = \delta_{a,b}$.
Therefore features using these $f_a$ and the constant function $1$ on the sphere 
are enough to distinguish them if they do not have the same number of points on the same
radii.\\
Furthermore, for any polynomial feature on $\BS^2$ we can combine this feature with the radial basis
function $f_a / t_a$ for $a\in\{1,2,...,2k\}$ and thus use the finiteness of 
necessary features for $X=\BS^2$ and colors $\calC \times \{1,2,...,2k\}$ to deduce
the required finiteness condition for our case.
\subsection{General outline for bound on number of features}
\label{sec:OutlineFiniteFeatures}
\noindent
The proof of Theorem \ref{theorem:FiniteFeatures} will depend on the the variant that we are interested
in, in particular whether $X$ is compact (i.e. the sphere $X=\BS^2$ or the spherical shell
$X=\{ r\in \BR^3\,\big|\, r_0\leq|\vv r|\leq r_1\}$) or not (i.e. $X=\BR^3$), and whether 
we restrict the function space to polynomials or not. In the main part we have restricted our 
attention to the function space of polynomials, in this case we can use algebraic geometry to 
prove this, see section \ref{subsec:AlgebraicFiniteFeatures} below.
If we allow more general radial basis functions, we will 
need to use a different theory (subanalytic sets) and will also need more assumptions, in 
particular that $X$ is compact.
\smallspace
We will give the proof for $X$ compact (i.e. sphere or spherical shell) in 
section \ref{subsec:AnalyticFiniteFeatures}, assuming the radial basis
functions are analytic in the interval $[r_0,r_1]$ where they are used.
For $X=\BR^3$ and polynomial features
we will give the proof in section \ref{subsec:AlgebraicFiniteFeatures}.
\smallspace
Let $\calX$ be the set of $G$--equivalence classes of colored point sets $\subseteq X$ 
with $|S_1|, ..., |S_c|$ points of colors $1,2,...,c$.
Let $\calF$ be a vector space of functions on $\calX$ that separate points.
\smallspace
We want to show: There are $2\cdot dim(X)\cdot k-5$ elements of $\calF$ that already separate points.
The outline of the proof is: 
\begin{enumerate}
    \item $\calX$ is an orbifold (see below) of $dim(\calX) = k\cdot dim(X)-3$.
    \item According to part a) there is a finite tuple of functions $f_1,...,f_N\in \calF$ that
      separate points of $\calX$.
      This defines an embedding $f:\calX\hookrightarrow\BR^N$.
    \item Show that the set $D\subseteq \BP^{N-1}$ of directions in $\BR^N$
       given by two different points $f(\vv x_1), f(\vv x_2)$ has a closure 
       $\overline D$ of dimension $\dim(\overline D) \leq 2 \cdot \dim(\calX)$.
    \item If $N-1 > 2\cdot dim(\calX)$, then there is a linear map
       $\BR^N\arrow \BR^{N-1}$ that is injective on $\calX$ (and in fact only 
       decreases distances by a factor of at most $C$ for some constant $C>0$).
\end{enumerate}
Then the theorem follows because we can start with the embedding from 1, use 2 to decrease $n$ as long as
$n>2\cdot dim(\calX)+1$ (replacing the functions $f_1,...,f_n$ by $n-1$ linear combinations of them), and
end up with $n=2\cdot dim(\calX)+1$, which is $6k-5$ for $X=\BR^3$ and $4k-5$ for $X=\BS^2$.
\\
The problem is now to find appropriate notions of functions / sets / dimensions
to make this outline precise. 

\subsection{Dimension of orbifolds and stratified manifolds}
\label{subsec:Orbifolds}
\noindent
We will not use much of the general theories of orbifolds and stratified manifolds, only 
enough to give a meaning to ``$dim(\calX) = k\cdot dim(X)-3$'' (that is compatible with
the definitions of dimension in real algebraic geometry and subanalytic geometry).
\smallspace
We can consider the set $M:=X^k$ as a manifold on which the compact group $K :=\Sigma \times G$ operates.
In general, the quotient $M/K$ of a manifold by a compact group gives an orbifold.
For example, the 1-dimensional manifold $\BR$ divided by the group $\{\pm 1\}$
gives a quotient space that can be identified with $\BR_{\geq 0}$. This is no longer 
a manifold, as the point $0$ has no neighborhood that would be homeomorphic to an open
interval.\\
These quotients (or orbifolds in general) can be considered as stratified manifolds, see e.g. chapter 4
of \cite{Pflaum2001}, the definition of stratified manifolds is also
explained in chapter I.1 in \cite{goresky2012stratified}.
In particular, stratified manifolds have a decomposition as a finite disjoint union of manifolds (usually
of different dimensions). In the above example, the point 0 would be a 0-dimensional stratum,
and $\BR_{>0}$ a 1-dimensional stratum. The dimension of a stratified manifold is the largest $d$
for which there is a non--empty component of dimension $d$.
\\
The stratification of a quotient $M/K$ is given by orbit types: For $x\in M$ denote by 
$K_x$ the stabilizer
\[
    K_x := \{g\in K\,|\, gx=x\}.
\]
For any two points $x,y=gx$ in the same orbit the stabilizers are conjugate: $K_y = g K_x g^{-1}$.
The conjugacy class of $K_x$ is also called the orbit type of $x$.
Now Theorem 4.3.7 in \cite{Pflaum2001} shows that the orbit types define a stratification of $M$,
and Corollary 4.3.11 gives that the quotients of the strata by $K$ define a stratification
of $M/K$.
\\
The orbit types have a natural partial order: For two subgroups $H_1, H_2$ of $K$ 
and their conjugacy classes $(H_1), (H_2)$ we say $(H_1) \leq (H_2)$ if $H_2$ is conjugate
to a subgroup of $H_1$. Now Theorem 4.3.2 in \cite{Pflaum2001} shows that there is a unique
largest orbit type and the points of this orbit type are open and dense in $M$.
\\
In our case $K=\Sigma\times G$, and an element $\kappa=(\sigma, g)\in K$ operates on a configuration
$x = (\vv r_1, ... , \vv r_k)\in X^k$ by permuting points (of same color) by $\sigma$, and applying
the rotation $g$ to all points. If the distances from the origin $|\vv r_i|$ are all different, 
$(\sigma, g) x = x$ can only happen for $\sigma=e$ and $g \vv r_i=\vv r_i$ for all $i=1,2,...,k$.
This in turn means that for $g\neq Id$, all $\vv r_i$ must lie on the line that is mapped to itself by 
$g$. So if we choose points $\vv r_i$ that have all different radii and do not lie on a line, 
$K_x = \{(e, Id)\}$. Since we can always find such points for $k>1$, the largest orbit type
is given by the neutral element in $K$, and then $K$ operates freely on the largest stratum.
So for $k>1$ the largest stratum is a manifold of dimension $\dim(X^k)-\dim(G) = k\cdot\dim(X) - 3$.
(And for $k=1$ the orbifold $M/K$ is $\BR_{\geq 0}$ of dimension 1.)
\smallspace
Our stratified manifolds can be given as semi--algebraic or sub--analytic subsets of some $\BR^N$,
and instead of talking about general stratified manifolds, in the following we will use these sets
and the theory of semi--algebraic / sub--analytic sets instead.
\smallspace
We already proved that there are features $f_1,...,f_N$ that give a map
\[
   X^k \arrow \calX \subseteq \BR^N
\]
which is invariant under $\Sigma \times G$, and which distinguishes any two configurations in $X^k$
that are not equivalent under $\Sigma \times G$, so the image $\calX$ of this map can be identified 
with the set $X^G/(\Sigma\times G)$ in a natural way.
\smallspace
We will now have to define a dimension for this image (which will be the same as above)
and its secant set, and the closure of the secant set,
and show that they behave as expected (which will be the case for algebraic and analytic functions, but 
would again be wrong for smooth functions in general).

\subsection{Algebraic case}
\label{subsec:AlgebraicFiniteFeatures}
\noindent
When we restrict our functions to be polynomials on $X$, this becomes a problem of algebra,
which can be solved by purely algebraic means. For example, we could apply Theorem 5.3 of \cite{KamkeKemper2012}.
However, we here use a more geometric formulation that will be easier to extend to the analytic case.
Since we are looking at sets in $\BR^N$ and real--valued functions, we have to use real algebraic geometry.
A standard reference is e.g. \cite{bochnak2013real}, which we will use for citing the propositions that we need.
\smallspace
An \emph{algebraic subset} of $\BR^N$ is the set of zeros of some set of polynomials in $\BR[X_1,...,X_N]$.
(Def. 2.1.1., p. 23). For example, the sphere $\BS^2\subseteq \BR^3$ is an algebraic 
subset.
\smallspace
A \emph{semi--algebraic subset} of $\BR^N$ is a set of points that can be defined by (finitely many)
polynomial equations $P(X_1,...,X_N)=0$ and inequalities $P(X_1,...,X_N)>0$ combined with Boolean
operations. (Def. 2.1.4., p.24). For example, the spherical shell $\{\vv r \in \BR^3\,\Big|\, r_0 \leq |\vv r| \leq r_1\}$
is given by the inequalities
\[
 r_0^2 \leq x^2+y^2+z^2 \quad\wedge\quad x^2+y^2+z^2 \leq r_1^2,
\]
so it is a semi--algebraic (but no algebraic) subset of $\BR^3$.
\smallspace
We can extend that definition to maps: A map $A\arrow B$ is said to be semi--algebraic if its graph 
$\subseteq A\times B$ is semi--algebraic. For example, rational functions are semi--algebraic:
Let $A\subseteq \BR^d$ be a semi--algebraic subset on which the polynomial $g(x_1,...,x_d)$ is not zero.
Then the map $A\arrow \BR$ given by the rational function
\[
   q(x_1,...,x_d):= \frac{f(x_1,...,x_d)}{g(x_1,...,x_d)}
\]
is semi--algebraic, since its graph is
\[
   \{\,(x_1,...,x_d, y)\ |\ y \cdot g(x_1,...,x_d) = f(x_1,...,x_d)\,\}.
\]
\smallspace
A fundamental theorem of real algebraic geometry is that any projection of a semi--algebraic set is 
again semi--algebraic (Theorem 2.2.1), this can be formulated as a model 
theoretic statement: This theory has ``quantifier elimination''  --- 
Given any formula $\Phi(X_1,...,X_n)$ of first--order predicate logic using only polynomial equalities
and / or inequalities, the set $\{\vv x\in \BR^N\,\big|\, \Phi(\vv x)\}$ is also semi--algebraic
(Proposition 2.2.4).\\
In particular, since the distance squared is given by a polynomial expression and the closure of a set
can be described as ``the points for which for each distance $d$ there is a point of distance 
$\leq d$ in the set'', we get that the closure of a semi--algebraic set is again a semi--algebraic 
set (Proposition 2.2.2). Similarly the image of a semi--algebraic set under a
semi--algebraic map is again a semi--algebraic set (Proposition 2.2.7), and that gives us in particular
that our model $\calX\subseteq \BR^N$ of $X^k/(\Sigma\times G)$ is a semi--algebraic set.\\

\smallspace
The dimension of a semi--algebraic set can be defined by using Theorem 2.3.6 of \cite{bochnak2013real}: All semi--algebraic sets are disjoint unions of semi--algebraic sets (semi--algebraically) homeomorphic to the $d$--dimensional open hypercube $]0,1[^d$ for some $d$ (for $d$=0 this is defined as 
a point). The dimension of the set is then the highest $d$ that appears in such decomposition (see chapter 2.8
in \cite{bochnak2013real}). This dimension has the expected properties:
\begin{itemize}
    \item An open semi--algebraic subset of $\BR^n$ has dimension $n$. (Prop. 2.8.4)
    \item For a semi--algebraic set $A$ and a semi--algebraic map $f$, $dim(A)\geq\dim(f(A))$. (Prop 2.8.8)
    \item For a semi--algebraic set $A$, the closure $\bar A$ has $\dim(\bar A)=\dim(A)$. (Prop.2.8.13)
\end{itemize}
From the first two properties we also can see that this notion of dimension coincides with the dimension
of a stratified manifold as described in section \ref{subsec:Orbifolds}.
\smallspace
Points in the (real) projective space $\BP^{N-1}$ are given by 1-dimensional subspaces (i.e. lines through
the origin) in $\BR^N$, they are denoted by $[x_1:...:x_N]$ for $(x_1, ... x_N) \neq \{(0,...,0)\}$
with the convention that
\[
    [x_1:...:x_N] = [\lambda \cdot x_1:...: \lambda \cdot x_N] \quad \hbox{for any}\ \lambda \neq 0
\]
We can map this set bijectively to a semi--algebraic subset of matrices in $\BR^{N\times N}$ by
\begin{equation}
   [x_1:...:x_N] \ \mapsto\  \left(\frac{x_i\cdot x_j}{\sum_m x_m^2}\right)_{i,j}
   \label{eq:ProjMatrix}
\end{equation}
This means we encode a line in $\BP^{N-1}$ as the orthogonal projection to that line in $\BR^{N\times N}$.
The image is given by the set of those matrices that satisfy $A^T=A, A^2=A, \Tr(A)=1$, so it is even an 
algebraic set (see chapter 3.4 in \cite{bochnak2013real}).
By looking at the usual decomposition of $\BP^{N-1}$ into affine sets $\BR^m$ for 
$m=0,1,...,N-1$ one can see that $\dim(\BP^{N-1}) = N-1$.
\smallspace
Now we define our secant set: Let $\Delta \subset \calX\times\calX$ be the diagonal
\[
   \Delta := \{(x,x)\,|\, x\in \calX\}
\]
and consider the map
\[
   sec: \quad \calX \times \calX \setminus \Delta \quad\longrightarrow \quad \BP^{N-1}
\]
defined by
\[
   \Big((x_1,...,x_N), (y_1,...,y_N)\Big) \ \mapsto\  [x_1-y_1\ :\ ...\ :\ x_N-y_N]
\]
Considering $\BP^{N-1}$ as set of matrices given by the image of \eqref{eq:ProjMatrix},
this becomes a semi--algebraic map (recall the remark above about rational functions being
semi--algebraic).
\smallspace
Using these properties of semi--algebraic sets and functions, and the definition of dimension 
of semi--algebraic sets, we now can make the statement 3 in the outline precise (for the algebraic
case): Let $D$ be the image of the semi--algebraic function $sec$ 
defined on the semi--algebraic set $\calX \times \calX \setminus \Delta$, and $\bar D$ its 
closure (which is again semi--algebraic), then its dimension must be
be
\[ 
   \dim(\bar D) = \dim(D) \leq \dim(\calX \times \calX \setminus \Delta) = 2 \cdot \dim(\calX).
\]
Since $\dim(\BP^{N-1}) = N-1$, the condition $N-1 > 2\cdot \dim(\calX)$ in statement 4 implies 
that there is a direction in $\BP^{N-1}\setminus \bar D$, and that in turn means that the orthogonal
projection in this direction to the orthogonal complement of this direction is injective.
Since $\bar D$ is closed, this furthermore means there must be a positive angle $\alpha$ between this
direction and the closest point in $\bar D$, which then translates to the fact that the orthogonal
projection to the orthogonal complement multiplies distances in $\calX$ by a factor 
$ > \sin(\alpha)>0$.

\subsection{Smooth case (counterexample)}
\noindent
In the above proof, we used for the semi--algebraic set $A := \calX \times \calX \setminus \Delta$
and the semi--algebraic (secant) map $f$ the general formula
\[
    \dim(\overline{f(A)}) \leq \dim(A).
\]
This formula would not be valid for smooth maps $f$:
Let $A$ be the open interval $]0,1[$, this is a 1-dimensional manifold.
Let $\vv x_i\in ]0,1[^d$ be any sequence in some $d$--dimensional hypercube.
Then we can construct a smooth function $f$ with $f(1/(i+1))=\vv x_i$ for $i=1,2,3,...$.
Since the set of rational numbers $\BQ$ is countable, also the set of points of the hypercube
in $\BQ^d$ is a countable set, so we can define a $f:]0,1[\arrow \BR^d$
that goes through all these points. Since they are dense in the hypercube, we have
\[
   \dim(\overline{f(A)}) = d > 1 = \dim(A).
\]
So the above arguments would not be valid for smooth functions.
However, we will see that they work for analytic functions.

\subsection{Analytic case}
\label{subsec:AnalyticFiniteFeatures}
\noindent
In the analytic case, we will follow mostly the same arguments as in the semi--algebraic case,
but use the theory of subanalytic geometry. While the properties of subanalytic sets
are much more difficult to prove, using them for our purpose works almost the same
as in the semi--algebraic case.
A standard reference for sub--analytic geometry is e.g. 
\cite{Shiota1997}, see also chapter I.1 in \cite{goresky2012stratified} and the notes \cite{Valette2023}.

\begin{definition}[semi--analytic subset]\phantom{.}\\
(\cite{goresky2012stratified}, p.43)
A semi--analytic subset $A$ of $\BR^N$ is a subset which can
be covered by open sets $U \subseteq \BR^N$ such that each $U\cap A$ is a union of connected
components of sets of the form $g^{-1}(0) - h^{-1}(0)$, where $g$ and $h$ belong to some
finite collection of real valued analytic functions in $U$.
\end{definition}
Examples:
\begin{itemize}
    \item $A=\BR^N$ is a semi--analytic subset:\\
      Take $g(\vv x):=0, h(\vv x):=1$.
    \item $\{\vv x\in \BR^N\,|\, g(\vv x)=0\}$ and $\{\vv x\in \BR^N\,|\, h(\vv x)\neq 0\}$ are semi--analytic subsets:\\
      Take $h(\vv x):=1$ or $g(\vv x):=0$ respectively.
    \item $\{\vv x\in \BR^N\,|\, h(\vv x)>0\}$ is a semi--analytic subset: It is a union of connected components
       of $\{\vv x\in \BR^N\,|\, h(\vv x)\neq 0\}$ 
    \item The union and intersection of semi--analytic subsets are semi--analytic:\\
       Union by definition, for intersection use
       \begin{eqnarray*}
          g_1=0\,\wedge\,g_2=0 \quad &\Leftrightarrow&\quad g_1^2 + g_2^2=0\\
          h_1\neq 0\,\wedge\,h_2\neq 0 \quad &\Leftrightarrow& \quad h_1 \cdot h_2\neq 0.
       \end{eqnarray*}
    \item The product of semi--analytic subsets are semi--analytic subsets of the product of their manifolds.\\
        (Use the same formulas as for intersection.)
\end{itemize}
It follows that semi--algebraic sets are also semi--analytic; in particular, 
the sets $X=\{\vv r \in \BR^3\,\big|\, r_0 \leq |\vv r| \leq r_1\} \subset \BR^3$ 
and $X^k \subset \BR^{3k}$ are semi--analytic.\\
We want to derive some properties about the image of $X^k$, but the image of semi--analytic sets under 
analytic maps is not guaranteed to be semi--analytic, therefore we need a more general definition:
\smallspace
\begin{definition}[subanalytic subset / map]\phantom{.}\\
(\cite{goresky2012stratified}, p.43)
A subanalytic subset $B$ of $\BR^N$ is a subset which can be covered by open
sets $V\subseteq \BR^N$ such that $V \cap B$ is a union of sets, each of which is a connected
component of $f(G)- f(H)$, where $G$ and $H$ belong to some finite family $\calG$ of
semi--analytic subsets of $\BR^{N'}$, and where $f: \BR^{N'}\to \BR^N$ is an
analytic mapping such that the restriction of $f$ to the closure of $\cup \calG$ is proper.
A subanalytic map between two subanalytic sets is one whose graph is subanalytic. 
\end{definition}
\noindent
In particular, if $f:\BR^{N'}\to \BR^N$ is a proper analytic map, 
and $G\subseteq \BR^{N'}$ is semi--analytic, then $f(G)$ is subanalytic in $\BR^N$.
\smallspace
Note that the word ``proper'' here is needed, even when going to subanalytic sets it is not in general
true that the image of subanalytic sets under analytic maps is again subanalytic. This is
another reason why we need the restriction to compact $X$ in the analytic case.
\smallspace
To mimic the arguments of the semi--algebraic case, we will also use the following 
properties: For subanalytic sets $A,B$, also $A\times B$, $A\cap B$, $A\setminus B$, and the closure
$\bar A$ are subanalytic. (\cite{Shiota1997}, property I.2.1.1, p.41)
Also subanalytic sets have a stratification (\cite{Shiota1997}, Lemma I.2.2, p. 44)
that allows to assign a dimension to subanalytic sets, and this dimension again has the property
$\dim(\bar A) = \dim(A)$ (property I.2.1.2, p.41), and $\dim(f(A))\leq \dim(A)$ for bounded
subanalytic sets $A$ and subanalytic maps $f$ (\cite{Valette2023}, chapter 2.3).
This gives for bounded subanalytic $A$ and subanalytic $f$ the formula
\[
   \dim(\overline{f(A)}) \leq \dim(A)
\]
and allows transferring the proof of the semi--algebraic case also to the 
subanalytic case.

\subsection{Variable number of points \texorpdfstring{$\leq k$}{up to k}}
\label{subsec:LessThanKPointsFiniteFeatures}
\noindent
Special PPSDs are 
\[
    |S_1| = \sum_{\vv r \in S_1} 1\ ,\quad ... \ ,\quad  |S_c| = \sum_{\vv r \in S_c} 1
\]
There are only finitely many value combinations that these functions
can take for configurations of $\leq k$ points.
(In fact, their number is
$\binom{k+c}{c}$, although the concrete number is not important for the following arguments).
\smallspace
To treat configurations of less than $k$ points together with configurations 
of $k$ points, let us adopt the convention that the points are enumerated by
first writing down the $|S_1|$ points of color 1, then the $|S_2|$ points of color 2,
..., then the $|S_c|$ points of color $S_c$, and then add $k-|S_1|-...-|S_c|$ points
at the origin $(0,0,0)$.
\smallspace
Since $X$ is compact, each feature $f$ has a bounded image. So we can encode the finite
amount of information in the $|S_1|,...,|S_c|$ together with $f$
in one feature
\[
   f + C\cdot |S_1| + C^2\cdot |S_2| + ... +  C^c \cdot |S_c|
\]
for $C$ large enough.

\section{Proof of Theorem \ref{theorem:AlgCompleteness}}
\label{sec:ProofAlgCompleteness}
\noindent
{\bf Part 1:} ``All scalar PPSDs are some linear combination of fields in this schema.''\\
First row: In case i, this is the well known statement that harmonic functions are 
linear combinations of spherical harmonics, and that polynomials of degree $n$ are
$|\vv r|^2\cdot$ polynomials of degree $n-2$ $\oplus$  harmonic polynomials of degree $n$,
see e.g. \cite{kosmann2009}, chapter 7, Prop. 2.7.\\ In case ii, this follows
from the definitions and the statement for $X=\BS^2$.\\
Following rows: By definition, the PPSDs of order $d$ are polynomials of order $d$ in the fundamental features, so as 
a vector space they are spanned by the products of PPSDs of order $d-1$ and PPSDs of order 1.
\\
If $f, g$ are scalar components of the vector values functions $F,G$, then the product $f\cdot g$
appears as scalar component in $F\otimes G$.
\smallspace
{\bf Part 2:} ``Any \emph{$SO(3)$--covariant} PPSD with values in $\irred{l}$ is a linear combination of vectors in the $l$--th column.''
\smallspace
In column $l$ we give the vectors of $\irred{l}$ as $2l+1$ entries in the standard basis $\vv e_{-l},...,\vv e_l$
of $\irred{l}$ such that the entries correspond to the standard real spherical harmonics. 
Each such vector of PPSDs corresponds to a covariant function from the configurations to $\calH$, or equivalently,
a covariant map from $\calH$ to the vector space $\calF$ of PPSDs, so all functions occurring in any field 
in the $l$--th column is in the isotypic component of $\irred{l}$ in $\calF$.\\
It remains to show that these functions only fit together in one way, i.e. that any given collection on $2l+1$
PPSDs, of which each individually can be written as linear combination of functions occurring in fields of the 
$l$--th column, can also as a whole vector be written as linear combination of vectors.
\smallspace
The vector space and $G$--representation of all functions in the $l$--th column and row $d$ is usually infinite dimensional,
but we can write it as a union of finite dimensional pieces: If we limit ourselves to fundamental features of $l\leq l_{max}$
and a finite subspace $\calF_{fin}\subseteq \BR + \calR$ of radial functions, then we have only finitely
many fundamental features and hence also only finitely many products of degree $\leq d$. Therefore in the following
we can use the theory of finite dimensional representations even when talking about infinite dimensional representations.
\smallspace
Consider the $G$--representation on the vector space $\calF_{l,d}$ which is the isotypic component of $\irred{l}$
of all PPSDs of degree $d$ (i.e. occurring in the $d$--th row). We can write it (if necessary, restrict to finite 
dimensional subspaces as explained above) as $\irred{l}\oplus...\oplus\irred{l}$,
see \ref{subsec:IsotypicDecomposition}.
Now any $G$--equivariant function on configurations gives a covariant map
$\irred{l}\to\irred{l}\oplus...\oplus\irred{l}$, so by the Lemma of Schur it must have the form
$\vv v \mapsto \alpha_1 \vv v \oplus ... \oplus \alpha_m \vv v$.
Given our standard base $\vv e_{-l},...,\vv e_l$ this means that the $m\cdot (2l+1)$--dim vector space 
$\irred{l}\oplus...\oplus\irred{l}$ can be written as the direct sum of
$2l+1$ pieces $\vv e_i \cdot \BR$ (only as vector space, not as representation), and hence any function
in one of the $2l+1$ pieces can only be written as linear combination of functions in the same piece.
Furthermore, since also the our given new function must be of the form $\vv v \mapsto \alpha_1 \vv v \oplus ... \oplus \alpha_m \vv v$, the linear combination must be the same for all components, i.e. any covariant PPSD with values in 
$\irred{l}$ must be a linear combination of the covariant PPSDs with values in $\irred{l}$ that occur in
our schema.
\smallspace
{\bf Part 3:} ``Any \emph{$O(3)$--covariant} PPSD with values in $\irred{l}$ is a linear combination of vectors in the $l$--th column of the appropriate parity.''
\smallspace
In our schema all functions are even or odd, so we only need to pick the ones with the right parity.

\section{Proof of Theorem \ref{theorem:InvariantTheory}}
\label{sec:ProofInvariantTheory}
\noindent
Lemma \ref{lemma:NFDecomposition_unique} (uniqueness of the normal form) also describes the
operation of $SO(3)$ on PPSDs $\calD$ in terms of its operation on polynomials:
Let $\calD$ be written in normal form as a sum of terms $\calE_{m_1,...,m_c;e_1,...,e_c}$
and which contain the polynomials $P_{m_1,...,m_c;e_1,...,e_c}$ as in \eqref{eq:NormalForm}.
Applying a $g\in SO(3)$ to $\calD$ can be done by applying them to the polynomials $P_{m_1,...,m_c;e_1,...,e_c}$, and this gives a representation of $\calD$ which is again
in normal form. Since the normal form is unique, this gives the unique way that $g$ operates on PPSDs given in normal form.
In particular, if a PPSD is invariant under $SO(3)$, the polynomials $P_{m_1,...,m_c;e_1,...,e_c}$ must be invariant under $SO(3)$.
\smallspace
We now can apply the First Fundamental Theorem of Invariant Theory for the group $SO(3)$, see
e.g. \cite{procesi2006lie}, section 11.2.1, p. 390: This polynomial 
$P_{m_1,...,m_c;e_1,...,e_c}$ in the vectors
$\vv r^{(1)},\vv r^{(2)},...,\vv r^{(m)}\in\BR^3$ (with $m=m_1+...+m_c$) 
can be written as a polynomial in expressions of the 
form $\langle \vv r^{(i)}, \vv r^{(j)}\rangle$ and $\det(\vv r^{(i)},\vv r^{(j)}, \vv r^{(k)})$.
To prove Theorem \ref{theorem:InvariantTheory}, it is then enough to show that all parts
corresponding to the monomials of this polynomial can be written as a contraction of a tensor product of moment tensors and (optionally) $\eps_{ijk}$.
\smallspace
We rewrite this in tensor notation, but we don't use the Einstein summation convention which
may obscure the change of summation order that will be used here.
For a vector $\vv v\in \BR^3$ we write $t_i(\vv v)$ for the 
$i$--th component of $\vv v$, and more general we write $t_i(\vv v)$, 
$t_{ij}(\vv v)$, $t_{ijk}(\vv v)$etc. for the tensors 
$v$, $v^{\otimes 2}$, $v^{\otimes 3}$ etc., i.e. $t_{ij}(\vv v) := t_i(\vv v) \cdot t_j(\vv v)$ etc.
For a variable $\vv r$ summing over the points in $S_{\gamma}$ of some color $\gamma$ 
we then get the moment tensors as
\[
     T_{i_1,...,i_k}(\gamma) := \sum_{\vv r \in S_\gamma}  t_{i_1,...,i_k}(\vv r)
\]
This gives now the following recipe to rewrite a sum over points of products of scalar products and determinants as a contraction of a products of moment tensors and (optionally) Levi--Civita symbols:
\begin{itemize}
    \item Rewrite scalar products $\langle \vv r^{(a)}, \vv r^{(b)}\rangle$ as
        $\sum_i t_i(r^{(a)}) t_i(r^{(b)})$ (using new indices for every new factor).
    \item Rewrite determinants $\det(\vv r^{(a)}, \vv r^{(b)}, \vv r^{(c)})$ as
        $\sum_{i,j,k} \eps_{ijk} t_i(r^{(a)}) t_j(r^{(b)}) t_k(r^{(c)})$ (using new indices for every new factor).
    \item Move the summations over coordinate indices $\{1,2,3\}$ to the left of the the summations
          over the points.
    \item Replace the product of $k$ vector components $t_{i_j}(\vv r)$ involving the same
          vector variable $\vv r$ by one tensor expression $t_{i_1,...,i_k}(\vv r)$.
    \item Replace the sum over one variable $\sum_{\vv r\in S_\gamma}$ and the corresponding
          expression $t_{i,...}(\vv r)$ by the moment tensor $T_{i,...}(\gamma)$
\end{itemize}
After these transformations, we have the rewritten the PPSD $\calD$ as a linear combination
of contractions of moment tensors and Levi--Civita symbols, as required by Theorem
\ref{theorem:InvariantTheory}.
\smallspace
We illustrate this recipe in a generic example:
\begin{eqnarray*}
  \lefteqn{\calD(S_1,S_2) }\\
    &:=& \sum_{\vv r^{(1)}\in S_1} \sum_{\vv r^{(2)}\in S_1} \sum_{\vv r^{(3)}\in S_2} \\
    & &  \ \ \ \langle \vv r^{(1)}, \vv r^{(1)}\rangle
               \langle \vv r^{(1)}, \vv r^{(2)}\rangle
               \det(\vv r^{(1)}, \vv r^{(2)}, \vv r^{(3)}) \nonumber\\
   &=&  \sum_{\vv r^{(1)}\in S_1} \sum_{\vv r^{(2)}\in S_1} \sum_{\vv r^{(3)}\in S_2}  \\
    & &  \ \ \left(\sum_i t_i(\vv r^{(1)}) t_i(\vv r^{(1)})\right)
         \left(\sum_j t_j(\vv r^{(1)}) t_j(\vv r^{(2)})\right)\\
    & &  \ \  \left(\sum_{k,l,m}  \eps_{klm} t_k(\vv r^{(1)}) t_l(\vv r^{(2)}) t_m(\vv r^{(3)})\right)\\
   &=&  \sum_{\vv r^{(1)}\in S_1} \sum_{\vv r^{(2)}\in S_1} \sum_{\vv r^{(3)}\in S_2}   \\
    & & \ \  \sum_{i,j,k,l,m} t_{iijk}(\vv r^{(1)}) t_{jl}(\vv r^{(2)}) t_m(\vv r^{(3)})
          \eps_{klm}\\
   &=&  \sum_{i,j,k,l,m} \sum_{\vv r^{(1)}\in S_1} \sum_{\vv r^{(2)}\in S_1} \sum_{\vv r^{(3)}\in S_2}\\
    & &\ \ t_{iijk}(\vv r^{(1)}) t_{jl}(\vv r^{(2)}) t_m(\vv r^{(3)})\eps_{klm}\\
   &=&  \sum_{i,j,k,l,m} \left(\sum_{\vv r^{(1)}\in S_1} t_{iijk}(\vv r^{(1)})\right)\\
    & & \ \ \left(\sum_{\vv r^{(2)}\in S_1} t_{jl}(\vv r^{(2)})\right)
        \left(\sum_{\vv r^{(3)}\in S_2} t_{m}(\vv r^{(3)})\right) \\
    &=&  \sum_{i,j,k,l,m}     T_{iijk}(\gamma_1) T_{jl}(\gamma_1) T_m(\gamma_2) \eps_{klm}    
\end{eqnarray*}

\section{Matrix moments examples}
\label{sec:MatrixMomentEx}

\subsection{\texorpdfstring{$3\times 3$}{3x3} matrices}

\noindent
The simplest example of matrix moments occurs for $a=b=1$, i.e. $3\times 3$ matrices:
From
\[
   \irred{1} \otimes \irred{1} \simeq \irred{0} \oplus \irred{1} \oplus \irred{2}
\]
we see that the 9--dimensional space of $3\times 3$--matrices decomposes into
irreducible subrepresentations of dimensions 1,3,5, these are given by 
\begin{itemize}
  \item Multiples of the identity matrix,
  \item antisymmetric matrices,
  \item symmetric matrices of trace 0.
\end{itemize}
\smallspace
The corresponding moment matrices $M_l:=M_{1,1,l}$ are
\[ 
    M_0 =  \left(
    \begin{array}{ccc}
     1 & 0 & 0 \\
     0 & 1 & 0 \\
     0 & 0 & 1  \\
    \end{array}
    \right)
\]
and
\[
     M_1 =  \left(
    \begin{array}{ccc}
     0 & -z & y \\
     z & 0 & -x \\
     -y & x & 0  \\
    \end{array}
    \right)
\]
and
\[
     M_2 = \left(
    \begin{array}{ccc}
     \frac{2 x^2-y^2-z^2}3 & xy & xz \\
     xy & \frac{2 y^2-x^2-z^2}3  & yz \\
     xz & yz & \frac{2 z^2-x^2-y^2}3  \\
    \end{array}
    \right)
\]
We observe that  $M_1^2 = M_2 - \frac23 r^2 Id$, which corresponds to the
fact that $M_1\cdot M_1$ encodes both the scalar product $r^2 = \vv r^T \cdot \vv r$
of 
\[
   \vv r := \left(\begin{array}{c} x\\ y\\ z\end{array}\right)
\]
with itself and the outer product $\vv r\cdot \vv r^T$, while the vector product
$\vv r\times \vv r=0$. Readers may want to convince themselves that the matrix product
$M_1 \cdot M'_1$ corresponding to two different vectors $\vv r, \vv r'$ indeed 
encodes scalar product, vector product, and outer product of $\vv r$ and $\vv r'$.
(To get the $\irred{2}$--component of the product, we have to subtract from
the outer product the appropriate multiple of the identity to get a traceless symmetric
matrix.)
\subsection{\texorpdfstring{$5\times 5$}{5x5} matrices}
The $5\times 5$ moment matrices $M_l:=M_{2,2,l}$ with
$0\leq l \leq 2$ are 
\[
 M_0 =  \left(
\begin{array}{ccccc}
 1 & 0 & 0 & 0 & 0 \\
 0 & 1 & 0 & 0 & 0 \\
 0 & 0 & 1 & 0 & 0 \\
 0 & 0 & 0 & 1 & 0 \\
 0 & 0 & 0 & 0 & 1 \\
\end{array}
\right)
\ ,\qquad
M_1 = 
\left(
\begin{array}{ccccc}
 0 & 2 x & z & -y & 0 \\
 -2 x & 0 & y & z & 0 \\
 -z & -y & 0 & x & -\sqrt{3} y \\
 y & -z & -x & 0 & \sqrt{3} z \\
 0 & 0 & \sqrt{3} y & -\sqrt{3} z & 0 \\
\end{array}
\right)
\]
\begin{equation}
{\small
   M_2 = \left(
\begin{array}{ccccc}
 -2 x^2+y^2+z^2 & 0 & 3 x y & 3 x z & -2 \sqrt{3} y z \\
 0 & -2 x^2+y^2+z^2 & -3 x z & 3 x y & \sqrt{3} \left(z^2-y^2\right) \\
 3 x y & -3 x z & x^2-2 y^2+z^2 & 3 y z & \sqrt{3} x z \\
 3 x z & 3 x y & 3 y z & x^2+y^2-2 z^2 & \sqrt{3} x y \\
 -2 \sqrt{3} y z & \sqrt{3} \left(z^2-y^2\right) & \sqrt{3} x z & \sqrt{3} x y & 2 x^2-y^2-z^2 \\
\end{array}
\right)
}
\end{equation}
\smallspace
The matrix $M_2$ can also be given as
\[
   M_2 = M_1^2 + 2 r^2 \cdot Id\qquad \hbox{with}\quad r^2 := x^2+y^2+z^2
\]
\smallspace
The $M_i$ are antisymmetric for odd $i$, 
we write them as $D_i - D_i^T$ with upper triangular matrices $D_i$.
\smallspace
For even $i$, the $M_i$ are symmetric, we write them as $D_i + diag(d_i) + D_i^T$
with upper triangular matrices $D_i$ and diagonal matrices with entries $d_i$.\\
Then $D_2, d_2$ are given as
\[ \left(
\begin{array}{ccccc}
 0 & 0 & 3 x y & 3 x z & -2 \sqrt{3} y z \\
 0 & 0 & -3 x z & 3 x y & \sqrt{3} \left(z^2-y^2\right) \\
 0 & 0 & 0 & 3 y z & \sqrt{3} x z \\
 0 & 0 & 0 & 0 & \sqrt{3} x y \\
 0 & 0 & 0 & 0 & 0 \\
\end{array}
\right), \ 
\left(
\begin{array}{c}
 r^2-3 x^2 \\
 r^2-3 x^2 \\
 r^2-3 y^2 \\
 r^2-3 z^2 \\
 3 x^2-r^2 \\
\end{array}
\right)
\]
The matrices $D_3, D_4$ are
{\small
\[
D_3 = 
\left(
\begin{array}{ccccc}
 0 & 3 r^2 x-5 x^3 & 10 z^3-6 r^2 z & 6 x^2 y-4 y^3+6 y z^2 & 5 \sqrt{3} \left(x z^2-x y^2\right) \\
 0 & 0 & -6 x^2 y-y^3+9 y z^2 & -6 x^2 z+9 y^2 z-z^3 & 10 \sqrt{3} x y z \\
 0 & 0 & 0 & 10 x^3-6 r^2 x & \sqrt{3} \left(r^2 y-5 x^2 y\right) \\
 0 & 0 & 0 & 0 & -\sqrt{3} \left(r^2 z-5 x^2 z\right) \\
 0 & 0 & 0 & 0 & 0 \\
\end{array}
\right)
\]
}
\smallspace
and 
\smallspace
{\scriptsize
\[ D_4 = 
\left(
\begin{array}{ccccc}
 0 & 70 y z \left(y^2-z^2\right) & -20 x y \left(r^2-7 z^2\right) & -20 x z \left(r^2-7 y^2\right) & -10 \sqrt{3} y z \left(r^2-7 x^2\right) \\
 0 & 0 & 10 x z \left(2 x^2+9 y^2-5 z^2\right) & -10 x y \left(2 x^2-5 y^2+9 z^2\right) & 5 \sqrt{3} \left(6 x^2 \left(y^2-z^2\right)-y^4+z^4\right) \\
 0 & 0 & 0 & -20 y z \left(r^2-7 x^2\right) & 10 \sqrt{3} x z \left(7 x^2-3 r^2\right) \\
 0 & 0 & 0 & 0 & 10 \sqrt{3} x y \left(7 x^2-3 r^2\right) \\
 0 & 0 & 0 & 0 & 0 \\
\end{array}
\right)
\]
}
\smallspace
and the diagonal entries of $M_4$ are
\smallspace
{\small
\[
d_4 = \left(
\begin{array}{c}
 4 x^4-12 x^2 y^2-12 x^2 z^2-16 y^4+108 y^2 z^2-16 z^4 \\
 4 x^4-12 x^2 y^2-12 x^2 z^2+19 y^4-102 y^2 z^2+19 z^4 \\
 -16 x^4-12 x^2 y^2+108 x^2 z^2+4 y^4-12 y^2 z^2-16 z^4 \\
 -16 x^4+108 x^2 y^2-12 x^2 z^2-16 y^4-12 y^2 z^2+4 z^4 \\
 24 x^4-72 x^2 y^2-72 x^2 z^2+9 y^4+18 y^2 z^2+9 z^4 \\
\end{array}
\right)
\]
}
\section{Proof of Theorem \refTheoremMatMultComplete}
\label{sec:ProofTheoremMatMult}

\subsection{The representation on matrices}
\noindent
The Clebsch--Gordan relation gives the isomorphism
\begin{equation} \label{eq:Clebsch_Gordan2}
  \calH^{(|a-b|)}\oplus\calH^{(|a-b|+1)}\oplus ... \oplus\calH^{(a+b)} \simeq \calH^{(a)} \otimes \calH^{(b)},
\end{equation}
and since $\calH^{(a)}$ is also isomorphic to its dual, this can be identified with linear maps
${\calH^{(a)}}^* \otimes {\calH^{(b)}} = Lin(\calH^{(a)}, \calH^{(b)})$ (see section \ref{subsec:DualTensorHom}).
If the elements of $\calH^{(a)}$ and $\calH^{(b)}$ are expressed numerically as $(2a+1)$--dimensional
vectors $\vv w$ and $(2b+1)$--dimensional vectors $\vv v$, this identifies $\vv v \otimes \vv w$ with the $(2b+1)\times(2a+1)$
matrix $\vv v \cdot \vv w^T$.
\smallspace
In fact, \ref{eq:Clebsch_Gordan2} is also valid for $a,b\in \frac12+\BN$ as representations of
the double cover of $SO(3)$, which can be identified with $SU(2)$. While a 
``rotation by $2\pi$'' in the double cover of $SO(3)$ operates as $-1$ on $\irred{a}$ and $\irred{b}$,
it then operates as identity on $\irred{a}\otimes\irred{b}$, so $\irred{a}\otimes\irred{b}$ is again
a representation of $SO(3)$, which we can interpret as $(2a+1)\times(2b+1)$--matrices, which are now
matrices with {\em even} side lengths. However, in practice this may be less attractive since it would
require computations with complex numbers, so we are not exploring this further in this paper.
\smallspace
Note that the ``matrix of matrices'' computations are just using reducible representations $V, W$
to form matrices $Lin(V,W)$: The reducible representations are written as direct sums
\[
   V = \bigoplus_a \irred{a}, \qquad W = \bigoplus_b \irred{b}
\]
and this gives the partition of the total side length of the big matrices into the side lengths $(2a+1)$
or $(2b+1)$ of the smaller matrices.
\smallspace
We can also generalize this from matrices to tensors of arbitrary order: Repeated application of the Clebsch--Gordan isomorphism gives a decomposition into irreducible representations of any tensor products (of irreducible representations,
or of arbitrary representations that come with a decomposition into irreducible ones).\\
We will see in the next section that any bilinear product between representations can be built up from Clebsch--Gordan
representations, so this applies in particular to matrix products and more generally to tensor contractions. So we
could also generalize the matrix based features of appendix \ref{sec:Algorithm} to tensor based features.

\subsection{Using Schur's lemma}
\noindent
As explained in Section \ref{subsec:SchursLemma}, Schur's lemma says that any morphism between irreducible
representations is 0 if the representations are not isomorphic, and the only morphisms
from an irreducible representation to itself are scalars for odd dimensional representation over $\BR$.
(For our $G$, all irreducible representations are odd dimensional.)
\smallspace
Let $\irred{a}\subseteq U$ and $\irred{b}\subseteq V$ be some irreducible components of $G$--representations
on $U$ and $V$, let $\circ: U\times V \to W$ be a bilinear covariant map, and $W\to \irred{c}$ be the 
orthogonal projection to an irreducible component $\irred{c}$ of $W$. Then we can restrict this bilinear map
to the subspaces $\irred{a}\subseteq U$ and $\irred{b}\subseteq V$ and get a covariant map
\begin{equation}
    \irred{a} \otimes \irred{b} \xhookrightarrow{\ \ } U \otimes V \xrightarrow{\ \circ\ } W \to \irred{c}.
    \label{eq:CGinProduct}
\end{equation}
Since 
\[
  \irred{a} \otimes \irred{b} \simeq \irred{|a-b|}\oplus\irred{|a-b|+1}\oplus ... \oplus\irred{a+b},
\]
formula \eqref{eq:isotypic3} says that $Hom_G(\irred{a}\otimes\irred{b}, \irred{c})$ is one dimensional if
$|a-b|\leq c \leq a+b$ and zero dimensional else. 
Since the Clebsch--Gordan map is nonzero, the map \eqref{eq:CGinProduct}
can only be non--zero if $|a-b|\leq c \leq a+b$ and then it must be a composition of
the Clebsch--Gordan product and the multiplication by a scalar $\lambda$:
\begin{equation}
  \irred{a} \otimes \irred{b} \xrightarrow{\hbox{\footnotesize Clebsch--Gordan}} \irred{c} 
  \xrightarrow{\ \cdot\lambda\ } \irred{c}.
  \label{eq:LambdaInProduct}
\end{equation}
Repeating this argument for all other $\irred{a}, \irred{b}, \irred{c}$ in a decomposition of $U,V,W$ into 
irreducible representations, we see that any bilinear covariant map $U\times V\to W$ must be a linear combination 
of Clebsch--Gordan maps applied to irreducible components of $U$, $V$, and $W$.\\
In particular, this applies to the matrix product (or, more generally, also any contraction of two tensors).\\
Since bilinear covariant maps are linear combinations of Clebsch--Gordan operations, we can reasonably hope that
we may also be able to go the other way round and recover the Clebsch--Gordan operations from enough 
bilinear covariant maps.
\smallspace
To show that the resulting features satisfy an
algebraic completeness theorem, we have to show that every sequence of Clebsch--Gordan operations
really appears with a non--zero coefficient in our procedure using a particular bilinear
covariant operation.\\
In the following section, we will prove that for products of two matrices, the $\lambda$ in
\eqref{eq:LambdaInProduct} is always nonzero, and then we will use that in the following
section to show that we can indeed write every sequence of Clebsch--Gordon operations
as a sequence of matrix multiplications (up to a non--zero factor).

\subsection{Using complex spherical harmonics}
\label{subsec:ComplexSphericalHarmonics}
\noindent
We introduced the operation of $G$ on the real vector spaces $\irred{l}$, which assigns to each
$g\in G$ a $(2l+1)\times(2l+1)$--matrix $\rho(g)$ with real entries. But we can as well interpret
the $\rho(g)$ as matrices of complex numbers, i.e. as endomorphisms of a complex vector space
$\irred{l}\otimes_\BR \BC$. The real spherical harmonic functions are a basis of $\irred{l}$
both as $\BR$--vector space and of $\irred{l}\otimes_\BR \BC$ as $\BC$--vector space.
Another basis of $\irred{l}\otimes_\BR \BC$ is given by the complex spherical harmonic 
functions. While for numerical computations the real vector space $\irred{l}$ is usually preferable,
the representation on the complex vector space $\irred{l}\otimes_\BR \BC$ is easier to use
for theoretical purposes.\\
In this section we will use the complex version, which we simply denote again by $\irred{l}$ in the
rest of this section.
\smallspace
We use the basis of the Lie algebra $\lalgso(3)$:
\begin{equation}
   l_x = \begin{pmatrix}
      0 &  0 &  0 \\
      0 &  0 & -1 \\
      0 &  1 &  0
   \end{pmatrix}
   \qquad
   l_y = \begin{pmatrix}
      0 &  0 & 1 \\
      0 &  0 & 0 \\
     -1 &  0 & 0
   \end{pmatrix}
   \qquad
   l_z = \begin{pmatrix}
      0 & -1 & 0 \\
      1 &  0 & 0 \\
      0 &  0 & 0
   \end{pmatrix}
   \label{eq:LieGenerators}
\end{equation}
and $L_x := i\cdot l_x, L_y := i\cdot l_y, L_z := i\cdot l_z$. 
The complex spherical harmonics are eigenvectors
of $L_z$ with eigenvalues $-l,...,l-1,l$. 
\smallspace
We will use the operation of $L_z$ on the complex version of $\irred{l}$, but once we know that matrix 
multiplication is related to the Clebsch--Gordan operation by a nonzero factor, we are free to go back to
computing in the real vector space spanned by the real spherical harmonics.
\smallspace
On $Lin(\irred{l_1}, \irred{l_2}) \times Lin(\irred{l_2}, \irred{l_3})$ we have 
the concatenation of linear maps (i.e. matrix product), 
in terms of $\irred{l_2} \otimes {\irred{l_1}}^*$ and $\irred{l_3} \otimes {\irred{l_2}}^*$
this product is given by 
\[
      (\vv v_3 \otimes \vv v_2^*) \cdot (\vv w_2 \otimes \vv w_1^*) =
    \langle \vv v_2^*, \vv w_2 \rangle_{alg} \cdot (\vv v_3 \otimes \vv w_1^*)
\]
We now choose a (complex) orthonormal basis $\vv b_{-l},...,\vv b_{l-1}, \vv b_l$ 
of $\irred{l}$ consisting of eigenvectors of $L_z$, e.g. the complex spherical harmonics.
This gives the dual basis $\vv b_{m}^*$ of ${\irred{l}}^*$ with 
\[
   \langle \vv b_{i}^* , \vv b_{j} \rangle_{alg} = \delta_{ij}
\]
and a basis $\vv b_{m_1} \otimes \vv b_{m_2}^*$ of $\irred{l_1}\otimes {\irred{l_2}}^*$
where $-l_1\leq m_1 \leq l1$ and $-l_2 \leq m_2 \leq l_2$.
\smallspace
The rotation $R_z(\alpha)$ around the $z$--axis with angle $\alpha$ can be given on 
$\irred{l}$ and ${\irred{l}}^*$ by its action on the basis
\[
   \rho(R_z(\alpha)) \vv b_m = e^{i\cdot m\cdot \alpha}\vv b_m, \qquad 
   \rho^*(R_z(\alpha)) \vv b_m = e^{-i\cdot m\cdot \alpha} \vv b_m
\]
and hence on $\irred{l_2}\otimes{\irred{l_1}}^*$ as
\[
   (\rho_2\otimes\rho_1^*)(R_z(\alpha)) (\vv b_{m_2} \otimes  \vv b_{m_1}^*)
   = e^{i\cdot (m_2 - m_1)\cdot \alpha} (\vv b_{m_2} \otimes  \vv b_{m_1}^*)
\]
For an eigenvector of $L_z$ we call the eigenvalue the \emph{weight} of the 
eigenvector. So the weight of $b_m$ is $m$, the weight of $b_m^*$ is $-m$, and the 
weight of $(\vv b_{m_2} \otimes  \vv b_{m_1}^*)$ is $m_2-m_1$.
\smallspace
The matrix multiplication is given on the basis vectors as
\begin{eqnarray}
    (\vv b_{m_3} \otimes  \vv b_{m'_2}^*) \cdot (\vv b_{m_2} \otimes  \vv b_{m_1}^*)
    &=& \langle \vv b_{m'_2}^*, \vv b_{m_2}\rangle_{alg}  (\vv b_{m_3} \otimes \vv b_{m_1}^*)
    \nonumber\\
    &=& \begin{cases}
       \vv b_{m_3} \otimes \vv b_{m_1}^* & \hbox{if}\ m'_2 = m_2 \\
       0 & \hbox{else}
    \end{cases}
    \label{eq:matrix_mult_tensor_prod}
\end{eqnarray}
In particular, the weights add up when we multiply these basis vectors of 
$Lin(\irred{l}, \irred{l'})$.\\
As matrices with respect to the bases $\vv b_{-l},...,\vv b_l$ of $\irred{l}$ and
$\vv b_{-l'},...,\vv b_{l'}$ of $\irred{l'}$this
$\vv b_m \otimes \vv b_{m'}^*$ is the matrix that has a 1 at row $m$ and column $m'$ and is 0
elsewhere, its weight is how far above the diagonal the 1 entry is.\\
In the following, we will use the abbreviation 
\[
    \vv b_{m, m'} := \vv b_m \otimes \vv b_{m'}^*.
\]
for these basis elements of $Lin(\irred{l},\irred{l'})$.
\smallspace
The irreducible representations $\irred{a}$ inside $\irred{l_2}\otimes{\irred{l_1}}^*$
have a basis of $L_z$ eigenvectors of weights $-a,...,a$ and in particular
have a highest weight vector $\vv v_a$ (unique up to scalars) which is characterized by
$L_+ \vv v_a=0$ for $L_+ = L_x + i L_y$. 
Since it is of weight $a$, it must be a linear combination of $\vv b_{m, m'}$
with $m-m' = a$. 
The action of an element $L\in \lalgso(3)$ on 
$\vv b_{m,m'}\in\irred{l_2}\otimes{\irred{l_1}}^*$ is given by
\[ 
    (\rho_2 \otimes \rho_1^*) (L) (\vv b_m \otimes \vv b_{m'}^*) 
    = (\rho_2(L) \vv b_m)\otimes \vv b_{m'}^* 
      + \vv b_m \otimes (\rho_1^*(L) \vv b_{m'}^*)
\]
so in particular
\begin{eqnarray*}
    (\rho_2 \otimes \rho_1^*) (L_+) (\vv b_{m,m'})
    &=& (\rho_2(L_+) \vv b_m) \otimes \vv b_{m'}^* 
      + \vv b_m \otimes (\rho_1^*(L_+) \vv b_{m'}^*) \\
    &=& c_1 \cdot \vv b_{m+1, m'} \, + \, c_2 \cdot \vv b_{m,m'+1}
\end{eqnarray*}
for some nonzero constants $c_1,c_2\in \BC$.\\
As a consequence, the highest weight vector $\vv v_a$ of the copy of $\irred{a}$ in 
$\irred{l_2}\otimes{\irred{l_1}}^*$ must be a linear combination of $\vv b_{p} \otimes \vv b_q$
with $p-q = a$ in which every term has a nonzero coefficient, since at every position
$p',q'$ with $p'-q'=a+1$ the contributions of $p'-1,q'$ and $p',q'-1$ have to cancel.
\smallspace
So write for any $a=0,1,2,...,2l$ the highest weight vector $\vv v_a$ as
\[
    \vv v_a = c_{a,l}\cdot\vv b_{l,l-a} + c_{a,l-1} \cdot \vv b_{l-1,l-a-1} + ... 
             + c_{a,a-l} \cdot \vv b_{a-l,-l}
\]
with all $c_{a,j}\in \BC$ nonzero.\\
Then for $0\leq a,b, a+b \leq 2l$
the matrix product of highest weight vectors of the copies of 
$\irred{a}$ and $\irred{b}$ in  $\irred{l_2}\otimes{\irred{l_1}}^*$ 
must have the form 
\[
   \vv v_a \cdot \vv v_b \quad = \quad d_{a+b,l}\cdot \vv b_{l,l-a-b} 
   \quad +\quad  d_{a+b,l-1}\cdot \vv b_{l-1,l-a-b-1}\quad
     +\quad  ...\quad  +\quad d_{a+b,a+b-l} \cdot \vv b_{a+b-l,-l}
\]
and from \eqref{eq:matrix_mult_tensor_prod} we get
\begin{eqnarray*}
    d_{a+b,l} &=& c_{a,l} \cdot c_{b,l} \\
    ...\\
    d_{a+b, a+b-l} &=& c_{a,a-l} \cdot c_{b,b-l}
\end{eqnarray*}
so also all resulting coefficients of the matrix product are nonzero, so in particular the
product cannot be 0.
\subsection{Proof of Theorem \refTheoremMatMultComplete}
\label{subsec:FinishProofMatMultComplete}
\noindent
The previous computation was the main argument needed to prove Theorem 
\refTheoremMatMultComplete: We now know that a matrix multiplication
computes all Clebsch--Gordan operations (up to non--zero scalar factors).
This would be enough to use matrix multiplication for Clebsch--Gordan operations
if we always used one matrix for one input, and would extract one output from 
the result. It remains to show that starting from the matrix moments
\begin{equation}
   M_{a,b,l}(\colori) := \iota_{a,b,l} \sum_{\vv r \in S_\colori} Y_l(\vv r)
   \label{eq:MatrixMoments2}
\end{equation}
and computing matrix products of several factors
(without extracting the irreducible components after each matrix multiplication)
gives enough covariants to span the vector space of all covariant functions 
with values in $\irred{l}$ for all $l$.
\smallspace
However, for this we just need to start with a vector and keep multiplying matrices from the left, 
the result is again a vector, so there is nothing to extract. To make sure this works, we write it down 
in detail:
\smallspace
We prove this by induction on the number $d$.
For $d=1$ there is nothing to prove: The $M_{0,l,l}(\gamma)$ are the fundamental features.
So assume we have $d>1$ and we already get all concatenation of Clebsch--Gordan operations for $d-1$.\\
Assuming the previous step of the Clebsch--Gordan operations resulted in a vector in $\irred{a}$, and
now we need to simulate the operation $\irred{l}\otimes\irred{a}\arrow\irred{b}$ for some $|a-l|\leq b \leq a+l$.
These inequalities are equivalent to $|a-b|\leq l \leq a+b$ (both are ways to express that $a,b,l$ can appear
as sides of a triangle), so we can use $M_{a,b,l}$ and multiply this matrix with the previous result vector in 
$\irred a$ to obtain the result vector in $\irred b$.
\smallspace
As in Theorem \ref{theorem:AlgCompleteness}, all results are either even or odd covariants, so to
get $O(3)$--covariants, we only need to include thow with the right parity in the linear combination.
\section{JAX implementation of matrix multiplication}
\label{sec:JaxMatrixMultiplication}
\noindent
We implemented the Clebsch--Gordan operation
\[
   (\irred{0} \oplus ... \oplus \irred{L}) \otimes
   (\irred{0} \oplus ... \oplus \irred{L})
   \arrow  \irred{0} \oplus ... \oplus \irred{L}
\]
in the obvious way as a general product in an $(2L+1)^2$--dimensional $\BR$--algebra with the
Clebsch--Gordan coefficients as multiplication table, and by changing
the multiplication table to the one corresponding table for the $(2L+1)^2$--dimensional algebra of
$(2L+1)\times (2L+1)$--matrices gives an alternative implementation of matrix multiplication.
It turned out that this more complicated way of formulating matrix multiplication is actually
significantly faster for small $L$ on TPUs and GPUs! Of course, once $L$ is large enough,
the benefit of the hardware support kicks in. 
\smallspace
An example computation is:
\smallspace
{\bf Computation A} (Multiplication in arbitrary algebra of dimension 9):
\begin{verbatim}
Inputs: jnp.arrays of size   
input1:  [20480,9,8]
input2:  [20480,9,8]
multTable: [9,9,9,8]
Compute: jnp.einsum('clf,cmf,lmnf->cnf',
                    input1, input2, multTable)
Output:  [20480,9,8]
\end{verbatim}
\smallspace
{\bf Computation B} (Multiplication for the special case of $3\times 3$--matrices):
\begin{verbatim}
Inputs: jnp.arrays of size   
input1:  [20480,3,3,8]
input2:  [20480,3,3,8]
Compute: jnp.einsum('clmf,cmnf->clnf',
                    input1, input2)
Output:  [20480,9,8]                    
\end{verbatim}
\smallspace
The first two inputs are the same (except treating $3\times 3$ as 9), computation B is a matrix multiplication of $3\times3$ matrices.
Computation A does something more general, which needs an extra input of coefficients (for a particular choice of 0's and 1's 
you get back matrix multiplication).
The time used for these operations measured on different platforms was:

\begin{table}[ht]
    \centering
    \begin{tabular}{c|c|c|c}
         &  \ CPU\  & \ TPU\  & \  P100\  \\
       \hline
       A\  &\  15\  &  0.4\  & \ 0.8 \\
       B\  &\   7\  &  4.6\  & \ 1.5 \\
    \end{tabular}
    \caption{Time in ms used for computations A / B}
    \label{tab:my_label}
\end{table}
As expected, the matrix multiplication is faster on CPU, as it does not need to multiply with the entries of the multiplication table.
However, on accelerators the more general / complicated computation was faster!
(Using {\tt jax.lax.batch\_matmul} for computation B gave the same result.) On TPUs, this effect persisted up to $L=4$, i.e. $9\times 9$--matrices, for $L\geq 5$ and above the special case matrix multiplication was faster.

\section{Details for experiments}
\label{sec:ExperimentDetails}
\subsection{Distinguishing configurations}
\label{subsec:DistinguishingConfigs}
We first want to demonstrate in concrete examples that the invariants of Theorem \refTheoremMatMultComplete\ indeed can distinguish configurations in challenging pairs that cannot be distinguished by invariants of low body orders, see Figure \refFigOverview, part D.\\
The examples are point sets on the unit sphere $\BS^2$. The 2--body invariants would only measure the distance from the center, so in 
this case the only such invariant is the number of points. In the cases of Figure \refFigOverview, part D, this would only distinguish
between the 3 pairs of 2, 4, 7 points, but not the two configurations of each pair. The corresponding matrix invariant would be 
$\sum_{\vv r\in S} M_{0,0,0} = \sum_{\vv r\in S} 1 = |S|$.
\smallspace
The first pair can be distinguished by a 3--body invariant, corresponding to a product of two fundamental features:
The fundamental feature is just $\vv r \mapsto r$, and the (only) invariant polynomial (up to scalar multiples) in the fundamental features is
\[
  \left|\sum_{\vv r} \vv r\right|^2 = 
    \left( \sum x \right)^2 + \left( \sum y \right)^2 + \left( \sum z \right)^2
\]
For the configuration 1a the $\sum_{\vv r} \vv r$ is zero, but for the configuration 1b the sum
$\sum_{\vv r} \vv r$ gives $(1,1,0)^T$, which has square norm 2.\\
To express this as matrix product, we write the $L=1$ fundamental feature as $3\times 1$ and $1\times 3$-- matrices,
i.e. use a=1, b=0 and a=0, b=1 to obtain
\[
   M_{1,0,1}\cdot M_{0,1,1} = (1,1,0)\cdot \begin{pmatrix} 1 \\ 1\\ 0 \end{pmatrix} = 2.
\]
\smallspace
For the second pair (which needs 4--body invariants), we need a product of 3 matrices, and we use
\[  
   M_{1,0,1}\cdot M_{1,1,2} \cdot M_{0,1,1}
\]
where $M_{1,1,2}$ writes the $L=2$ fundamental feature as a symmetric, traceless $3\times 3$--matrix.
\smallspace
For the third pair, we define an invariant as a product of 4 matrices which give a linear map
\[
   \BR \xrightarrow{M_{0,2,2}} \BR^5 \xrightarrow{M_{2,1,1}} \BR^3 \xrightarrow{M_{1,1,2}} \BR^3 \xrightarrow{M_{1,0,1}} \BR 
\]
which is then of course again interpreted as a scalar.

\subsection{Synthetic data experiment}
\label{subsec:InvariantDefinitions}
\noindent
For the synthetic experiment we created an invariant PPSD 
on point sets on $\BS^2$ colored with 5 colors. Using the moment tensors
\[
   T_{i_1 i_2...i_k}(\gamma) := \sum_{\vv r\in S_\gamma} \vv r^{\otimes k}
\]
it can be expressed as a tensor contraction
\[
   T_{abcdefghij}(\gamma_1)T_{akn}(\gamma_2)T_{bckl}(\gamma_3)
           T_{deflm}(\gamma_4)T_{ghijmn}(\gamma_5)
\]
(using Einstein summation convention, i.e. summation over all indices is implied, but we write all indices as lower indices). Since tensor contractions are $O(3)$--covariant, this gives an $O(3)$--invariant function on point configurations colored with 5 colors. Each of the 
five tensors can be computed as a fundamental feature, so the degree of this invariant is 5 (i.e. its body order 6). The first tensor
is of order 10, so it contains irreducible representations up to $L=10$.
\smallspace
We generated 8192 training configurations, and 4096 test configurations of 20 points uniformly randomly 
sampled on the $\BS^2$, with 4 points assigned to each of the 5 colors.
\smallspace
For the {\bf Clebsch--Gordan} nets we use maximal degree $L=10$ and 25 channels.
First we use the method {\tt e3x.so3.irreps.spherical\_harmonics} to 
compute the spherical harmonics up to $L=10$ and then combine them linearly (with 
learnable weights) to 25 combinations for each of 5 ``factors''. Then we use the 
the method {\tt e3x.so3.clebsch\_gordan} from E3x (see \cite{Unke2024e3x}),
to ``multiply'' these layers 1 -- 4.
This means in particular that the data in each layer are 25 vectors in 
\[
  \irred{0} \oplus ... \oplus \irred{10}
\]
and each Clebsch--Gordan operation takes $11^3\cdot 25 = 33275$ learnable parameters.\\
For the last ``multiplication'' we use a scalar product (a performance optimization,
since we would only use the $\irred{0}$--part of the Clebsch--Gordan product), this then only
uses $11\cdot 25 \cdot 25 = 6875$ learnable parameters to compute all scalar products, 
which are then added up to give the scalar result.
\smallspace
For the {\bf matrix multiplication} nets we used big matrices with total side length
\[
   4 \times (2\cdot {\bf 4} + 1) +  4 \times (2 \cdot {\bf 5} + 1 ) = 80,
\]
i.e. 4 copies of $L=4$ and 4 copies of $L=5$.
Since the $(2\cdot5+1)\times (2 \cdot 5 + 1)$--matrices correspond
to the representation
\[
  \irred{5}\otimes\irred{5} = \irred{0} \oplus ... \oplus \irred{10},
\]
this corresponds again to a maximal $L=10$ being used, matching the Clebsch--Gordan setting.
The number of matrices is 64, which is a bit larger than the 25 channels used in the Clebsch--Gordan
setting, but on the other hand this only contains 16 matrices of the full size containing a $\irred{10}$ component, a bit less than the 25 in the Clebsch--Gordan case. (We cannot give fully equivalent
settings if we want to use also non--square constituent matrices, but this seems to be a reasonable approximation which also gives comparable accuracies.)
\smallspace
As in the Clebsch--Gordan case, we start with computing the spherical harmonics and combine their
sums into vectors in $\irred{0},...,\irred{10}$; again we use 5 factors (i.e. matrices in this case).
After computing their matrix product, we take the traces of the $2\cdot (4 \times 4)=32$ square sub--matrices that occur in the
big matrix, and combine them in a linear combination (with learnable coefficients), this is the 
direct equivalent of combining the $\irred{0}$--components as in the Clebsch--Gordan case (although 
without the performance optimization).
\smallspace
We trained with stochastic gradient descent for 40 episodes on the training set (with learning rate and batch size tuned separately for the two cases), and evaluated on the test set.
Averaging over 10 runs gave the learning curves reported in the main part.
\smallspace
As is to be expected, more / larger matrices usually give lower error, but longer execution time.
Below we add as an example a network with matrices of size
\[
   4 \times (2\cdot {\bf 1} + 1) + 4 \times (2\cdot {\bf 2} + 1) 
   + 4 \times (2\cdot {\bf 4} + 1) +  4 \times (2 \cdot {\bf 6} + 1 )
   = 120
\]
In the resulting learning curves (averaged over 10 runs) 
\\ \noindent
\begin{adjustwidth}{0pt}{-10mm}
  \includegraphics[width=0.47\textwidth]{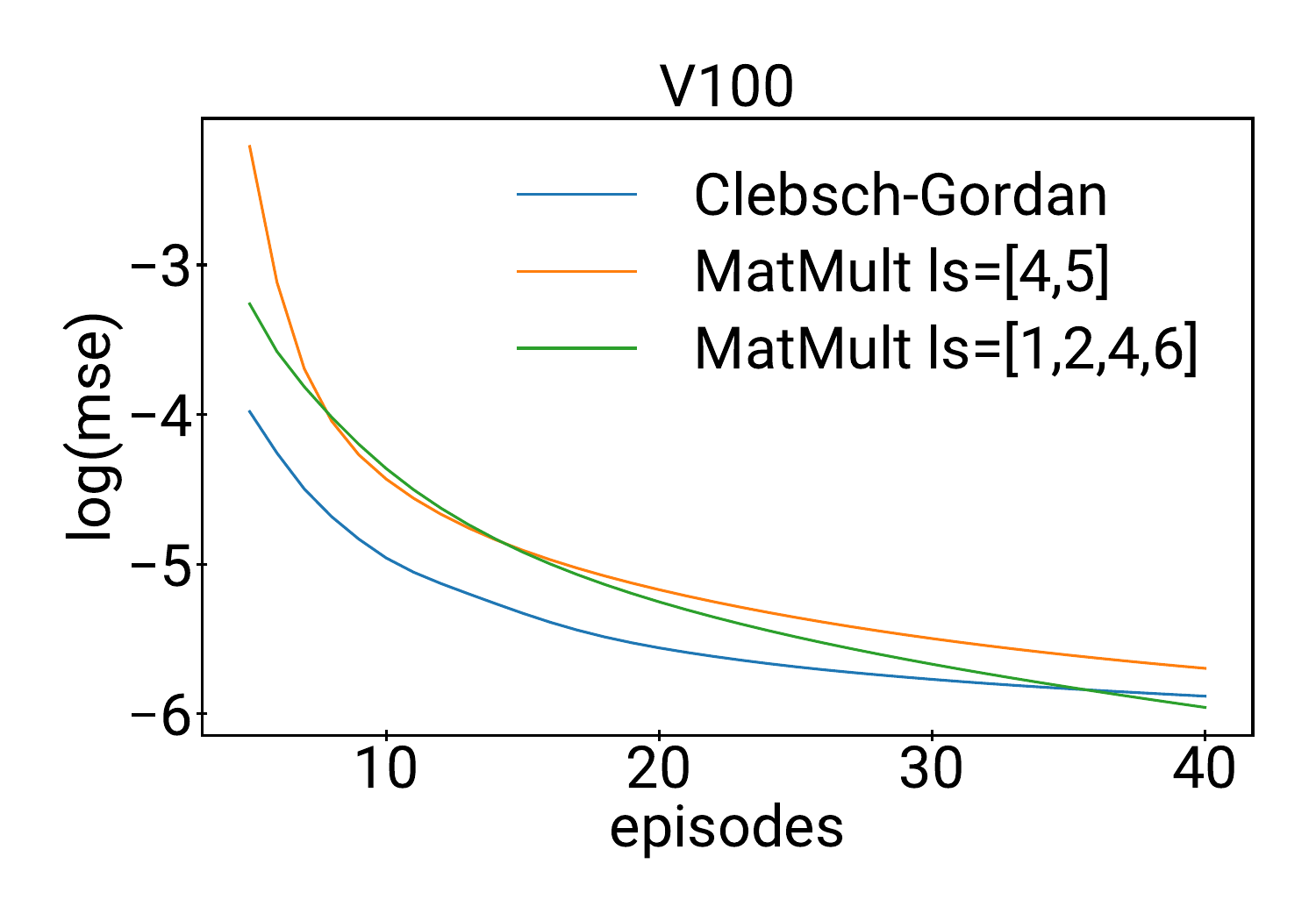}
  \raisebox{3mm}{\includegraphics[width=0.44\textwidth]{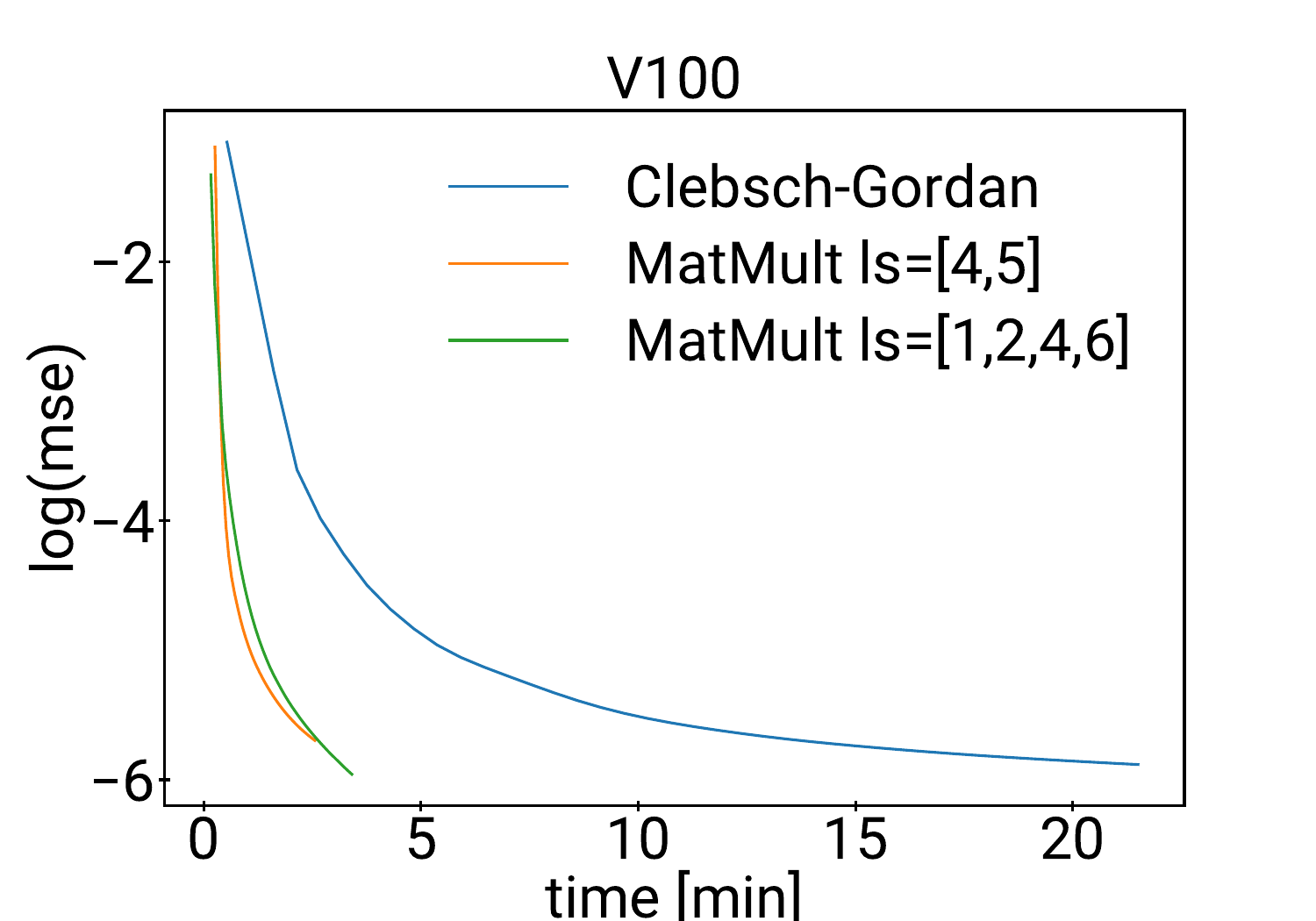}}
\end{adjustwidth}
we can see on the left hand side that the accuracy using these Matrix Multiplication nets for the same number of episodes is similar to the Clebsch--Gordan nets, but on the right hand side we can see that the Matrix Multiplication nets obtain this accuracy much faster.
%
\subsection{Experiment on atomistic simulation}
\noindent
In this experiment we want to show that we can use this method to obtain interesting results
on real world data. We use the data set MD-22 published with \cite{Chmiela2023}, and compare 
the results with the accuracies reported there and in \cite{frank2022so3krates}.
\smallspace
We use the architecture specified in Appendix \ref{sec:Algorithm} to learn ``energy
contributions'' for each atom in a molecule, these contributions are summed over all
atoms to give the energy of the configuration. To obtain the force acting on one atom, the
derivative of the energy in direction of the position of this atom is computed. This 
is done by automatic differentiation in JAX. The network is trained mainly on the forces:
We use a mean square error as the loss function on both forces and energies, but with
a scaling that puts almost all the weight on the forces.
\smallspace 
Since these energy contributions should depend on the element type of each atom, we 
ideally would learn different functions for each atom type. However, as a more 
parsimonious solution we just add learnable additive constants to all vectors and
sub--matrices which depend on the element. 
\smallspace
In contrast to the synthetic data experiment, the points are now in $\BR^3$, and 
we use as radial functions Chebyshev polynomials of the log radius and a cutoff
radius of 6 \AA (so we are effectively in our case 2ii). This is implemented in E3x
as {\tt e3x.nn.functions.chebyshev.exponential\_chebyshev}.
\smallspace
It seems to be beneficial to start with a smaller number of radial basis functions
and increase their number during the training, this increases the expressivity
of the radial basis functions only gradually and makes the end result a bit smoother.
We start with the first 12 Chebyshev polynomials as radial basis functions, and add a
new radial basis function every 4 epochs.\\
Similarly, we start with products of $b=4$ matrices and increase every 5 epochs by one, up to a maximal $b=12$.
\smallspace
The sub--matrices have side lengths 1,3,...,11, corresponding to $L=0,1,...,5$, 
so the maximal $L$ occurring in the sub--matrices is $L=10$. We use each of these
lengths 3 times, so the big matrix has a side length of
\[
   3 \cdot \Big( (2\cdot 0 + 1) + (2\cdot 1 + 1) + ... + (2 \cdot 5 + 1) \Big) = 108
\]
We use $n_{mat}=8$ of such matrix products, and we use $n_{vec}=12$ vectors on each
end of the product, but in slight deviation to the original algorithm as specified in 
Appendix \ref{sec:Algorithm} we do not reuse the vectors, so we ``only'' get
\[
  n_{mat} \cdot r \cdot n_{vec} = 8 \cdot (6 \cdot 3) \cdot 12 = 1728
\]
invariants from the full products of 12 matrices. 
We also do the same for each partial product of 1,2,...,11 matrices that we obtain while
computing the full products, so in total we get $12 \cdot 1728 = 20736$ invariants, 
of which we then learn the best linear combination.
\smallspace
For the two largest molecules we use slightly smaller networks, as specified
in Table \ref{tab:NetworkSizes}.
\begin{table}[ht]
    \centering
    \begin{tabular}{c||c|c|c}
         & \ small\  & buckyball catcher & nanotubes \\
         \hline
         $n_{mat}$      &        8 &        5 &   3 \\
         $n_{vec}$      &       12 &       10 &  10 \\
         $L_{max}$      &       10 &       10 &   8 \\ 
         \# repeats $L$ &        3 &        3 &   4 \\
         \hline
         \# invariants & 20,736 & 10,800 & 7,200 \\
    \end{tabular}
    \caption{Size of the network}
    \label{tab:NetworkSizes}
\end{table}
\smallspace
Maybe surprisingly, these simple networks using only a linear combination of the invariants, no message passing, no nonlinearities (apart from the matrix multiplication), give results that are similar to those of the more sophisticated 
methods, see Table \refTabMD\ in the main part.

\bibliographystyle{plain}
\bibliography{references}

\end{document}